%% file: main_final.tex
\documentclass[twoside,11pt]{article}

\usepackage{amsmath}
\usepackage[colorlinks=false,allbordercolors={1 1 1}]{hyperref}
\usepackage{amsthm}
\usepackage[nameinlink]{cleveref}
\hypersetup{ hidelinks }

\usepackage[preprint,nohyperref]{jmlr2e}

\input{preamble}

\usepackage{lastpage}
\jmlrheading{26}{2025}{1-\pageref{LastPage}}{8/24}{10/25}{24-1365}{Xing Liu and Fran\c{c}ois-Xavier Briol}

\ShortHeadings{On the Robustness of Kernel Goodness-of-Fit Tests}{Liu and Briol}
\firstpageno{1}

\begin{document}

\etocsettocdepth.toc{chapter}

\title{On the Robustness of Kernel Goodness-of-Fit Tests}

\author{\name Xing Liu 
        \email xingliu97@outlook.com \\
        \addr QuantCo
        \AND
        \name Fran\c{c}ois-Xavier Briol 
        \email f.briol@ucl.ac.uk \\
        \addr Department of Statistical Science\\
        University College London}

\editor{Bharath Sriperumbudur}

\maketitle

\begin{abstract}
    Goodness-of-fit testing is often criticized for its lack of practical relevance: since ``all models are wrong'', the null hypothesis that the data conform to our model is ultimately always rejected as the sample size grows. Despite this, probabilistic models are still used extensively, raising the more pertinent question of whether the model is \emph{good enough} for the task at hand. This question can be formalized as a robust goodness-of-fit testing problem by asking whether the data were generated from a distribution that is a mild perturbation of the model. In this paper, we show that existing kernel goodness-of-fit tests are not robust under common notions of robustness including both qualitative and quantitative robustness. We further show that robustification techniques using tilted kernels, while effective in the  parameter estimation literature, are not sufficient to ensure both types of robustness in the testing setting. To address this, we propose the first robust kernel goodness-of-fit test, which resolves this open problem by using kernel Stein discrepancy (KSD) balls. This framework encompasses many well-known perturbation models, such as Huber's contamination and density-band models.
\end{abstract}


\section{Introduction}
\label{sec: introduction}

Goodness-of-fit (GOF) testing \citetext{\citealp{d2017goodness}; \citealp[Chapter 16]{lehmann2008testing}} tackles the question of how well a given probabilistic model describes some observed data. More formally, given a model $P$ and observations drawn independently from some distribution $Q$, GOF testing compares the null hypothesis $H_0: Q = P$ against the alternative hypothesis $H_1: Q \neq P$. This process is fundamental to validating a model before it is used for predictions, decision-making, or providing probabilistic guarantees on important quantities of interest. Indeed, GOF testing is closely linked to the concept of statistical significance, which now permeates almost all areas of science and engineering. 

In this paper, we focus exclusively on kernel-based GOF tests \citep{chwialkowski2016kernel,liu2016kernelized}. Unlike most classical alternatives such as likelihood-ratio tests \citep[Chapter 8.7]{hogg1977probability} and Kolmogorov-Smirnov tests \citep{kolmogorov1933sulla}, kernel GOF tests can be used for models whose density function is known only up to a multiplicative constant. This is a significant advantage since it enables their use for models which are beyond the reach of other tests, including many modern flexible density estimation models, energy-based models, large graphical models, and Bayesian posteriors. Kernel GOF tests achieve this through a tractable test statistic based on the \emph{score function} (i.e., the gradient of the log-density), which is widely available since it can be evaluated without knowledge of the normalization constant of the density. This property has made kernel GOF tests popular and has led to numerous extensions specializing the approach to problems involving time-to-event data \citep{fernandez2020kernelized}, discrete data \citep{yang2018goodness}, point processes \citep{yang2019stein}, manifold data \citep{xu2020stein}, graphs \citep{xu2021stein}, protein structures \citep{amin2023kernelized}, and text documents of variable lengths \citep{baum2023kernel}.

A significant drawback of GOF testing is the following conundrum: in almost all real-world applications, models are wrong, causing GOF tests to ultimately reject the null hypothesis as the sample size grows. However, there is often a difference between statistical significance from a GOF perspective, and the practical relevance of the model. For instance, \emph{data corruption} may arise in signal processing due to sensor failures \citep{rizzoni1991detection,sharma2010sensor}, in classification tasks due to mislabelled entities \citep{frenay2013classification}, and in radio systems due to impulsive noise \citep{blackard1993measurements}. In these cases, $H_0$ may be rejected even though $P$ is \emph{almost} correct up to these mild perturbations, and thus it could still be useful for downstream tasks such as prediction or probabilistic inference. 

This conundrum has led to the development of \emph{robust tests} \citep{huber1965robust,lecam1973convergence,dabak1994geometrically,fauss2021minimax}, which aim to control the Type-I error rate (the probability of falsely rejecting the null hypothesis) under a \emph{composite} null hypothesis $\Hc_0: Q \in \cP_0$. The set $\cP_0$ is a family of probability distributions, called the uncertainty set \citep{fauss2021minimax}, which is constructed to include $P$ and distributions similar to $P$ in some appropriate notion. By designing $\cP_0$ to contain potential contaminations, such as adversarial contaminations or outliers, the composite null hypothesis $\Hc_0$ can now hold in realistic scenarios, thus avoiding the aforementioned conundrum and aligning statistical and practical significance. 

As a concrete example, consider a scenario where a practitioner models a data set of sensor signals that are potentially contaminated by outliers. Suppose the fitted model accurately captures the true signal in the data set. A standard GOF test might still reject the standard null hypothesis $H_0: Q = P$ due to the presence of outliers. Instead, one can adopt a composite null hypothesis $\Hc_0: Q \in \cP_0$, where $\cP_0$ is constructed to include ``contaminated versions'' of $P$ within a certain tolerance. A GOF test designed to be calibrated against $\Hc_0$ can now reject the model with the correct test level.

Unfortunately, robust kernel-based GOF tests have only received limited attention in the existing literature. The closest work is that of \citet{sun2023kernel,schrab2024robust}, who proposed robust tests using uncertainty sets defined by the Maximum Mean Discrepancy (MMD; \citealt{muller1997integral,gretton2012kernel}). These tests require samples from \emph{both} $Q$ and $P$. However, generating samples from $P$ may be infeasible or computationally expensive, and approximating $P$ by finite samples can also reduce statistical efficiency. 

We propose a novel class of robust kernel GOF tests based on the \emph{kernel Stein discrepancy} (KSD). These tests are appropriate in the \emph{one-sample} setting, where we assume samples are available only from $Q$, and an unnormalized density is available for $P$. To guarantee robustness against mild perturbations, we use \emph{tilted} kernels proposed in \citet{barp2019minimum,matsubara2022robust}, in a related construction for robust parameter estimation. We then study \emph{qualitative} and \emph{quantitative} robustness, two notions of robustness which capture a test's insensitivity to perturbations around $P$, as formally introduced in \Cref{sec: robustness GOF test}. Our contributions are summarized in \Cref{tab: summary} and detailed below:
\begin{enumerate}
    \item In \Cref{sec: lack_robustness}, we study the robustness properties of existing KSD tests. We show that KSD tests with stationary kernels are \emph{not necessarily} robust under infinitesimal model deviation (which we call \emph{qualitative robustness}; see \Cref{def: qualitative robustness}), and \emph{never} robust against fixed classes of misspecified models (called \emph{quantitative robustness}; see \Cref{def: quantitative robustness}). We then show that qualitative robustness can be guaranteed by using appropriately tilted kernels, but this is \emph{not} sufficient to guarantee quantitative robustness.
    \item In \Cref{sec: robust KSD test}, we propose a kernel GOF test that controls the Type-I error rate against all distributions in a KSD-ball of radius $\theta > 0$ centered at $P$. This test is straightforward to implement since it is identical to existing KSD tests up to a shift of the test statistic by $\theta$. We then make recommendation on how to select $\theta$ in practice to ensure robustness to common perturbations such as Huber's contaminations (which includes outliers) or density-band contaminations. 
\end{enumerate}

\begin{table}[t]
    \begin{center}
        \resizebox{\textwidth}{!}{
        \begin{tabular}{lccc}
            \multicolumn{1}{c}{\bf Goodness-of-Fit Test} &\multicolumn{1}{c}{\bf Qualitative robustness}  &\multicolumn{1}{c}{\bf Quantitative Robustness}  &\multicolumn{1}{c}{\bf Result} \\  
            \\ \hline \\
            Existing KSD test $+$ stationary kernels    & {\color{red}\xmark} & {\color{red}\xmark}  & \Cref{thm: non robust stationary} \\
            Existing KSD test $+$ tilted kernels    & {\color{ao(english)}\cmark}   & {\color{red}\xmark} & \Cref{thm: robust tilted} \\
            Novel KSD tests proposed in this paper & {\color{ao(english)}\cmark} & {\color{ao(english)}\cmark} & \Cref{thm: bootstrap validity} \\
        \end{tabular}
        }
        \vspace{-5mm}
    \end{center}
    \caption{Summary of theoretical contributions to the robustness of kernel GOF tests provided in this paper.}
    \label{tab: summary}
    \vspace{-2mm}
\end{table}


\section{Background}
\label{sec: Background}
We now review the necessary background for this work, covering kernel Stein discrepancy and robust GOF testing. We begin by briefly summarizing our notation.

Let $\cP(\R^d)$ denote the set of probability measures on $\R^d$. Given an integer $r > 0$, $\cC^r$ is the set of functions $f: \R^d \to \R$ that are $r$-times continuously differentiable. The gradient of $f \in \cC^1$ is denoted $\nabla f(\bx) = (\partial_1 f(\bx), \ldots, \partial_d f(\bx))^\top$. The set $\cC_b^1$ contains functions in $\cC^1$ that are bounded and with bounded gradient, and the set $\cC_b^2$ contains $f \in \cC_b^1$ for which $\sup_{\bx \in \R^d} | \nabla^\top \nabla f(\bx) | < \infty$, where $\nabla^\top \nabla f(\bx) = \sum_{j=1}^d \partial_j \partial_j f(\bx)$. For a function $k: \R^d \times \R^d \to \R$ with two arguments, $\nabla_i k(\bx, \bx')$ will be used to denote its gradient with respect to the $i$-th argument respectively for $i = 1, 2$, and we define $\nabla_1^\top \nabla_2 k(\bx, \bx') = \sum_{j=1}^d \partial_{1 j} \partial_{2 j}  k(\bx, \bx')$, where $\partial_{ij} k(\bx, \bx') = [\nabla_i k(\bx, \bx')]_j$ for $i=1, 2$ and $j=1, \ldots, d$. The set $\cC^{(1, 1)}$ (resp., $\cC^{(1,1)}_b$) denotes all functions $k: \R^d \times \R^d \to \R$ with $(\bx, \bx') \mapsto \partial_{1j} \partial_{2j} k(\bx, \bx')$ continuous (resp., continuous and bounded) for all $j = 1, \ldots, d$.

\subsection{The Kernel Stein Discrepancy}
A key ingredient for kernel GOF tests is the kernel Stein discrepancy \citep[KSD,][]{chwialkowski2016kernel,liu2016kernelized,gorham2017measuring,oates2017control}.
Throughout, we assume $P,Q \in \cP(\R^d)$, and $P$ admits a Lebesgue density function $p \in \cC^1$ such that $p > 0$ on $\mathbb{R}^d$. Let $k \in \cC^{(1, 1)}$ be a scalar-valued reproducing kernel with associated reproducing kernel Hilbert space \citep[RKHS,][]{berlinet2004reproducing} denoted by $\cH_k$. The \emph{Langevin KSD} gives a notion of discrepancy between $Q$ and $P$ and is defined as 
\begin{align*}
    \ksd(Q, P)
    \;&\coloneqq\;
    \sup_{f \in \cB}
    \big| \E_{\bX \sim Q}[ (\cA_p f) (\bX) ] \big| = \sup_{f \in \cB}
    \left| \int_{\mathbb{R}^d} (\cA_p f) (\bx) Q(\diff\bx)  \right|
    \;,
\end{align*}
where $\cB \coloneqq \{ h=(h_1, \ldots, h_d): \; h_j \in \cH_k \textrm{ and } \sum_{j=1}^d \| h_j \|_{\cH_k}^2 \leq 1 \}$ is the unit ball in the $d$-times Cartesian product $\cH_k^d$ of the RKHS $\cH_k$. The operator $\cA_p$ is a linear operator called the \emph{Langevin Stein operator} and it maps (suitably regular) vector-valued functions $f$ to scalar-valued ones via $(\cA_p f)(\bx) \coloneqq \bs_p(\bx)^\top f(\bx) + \nabla^\top f(\bx)$, where $\bs_p(\bx) \coloneqq \nabla_\bx \log p(\bx)$ is known as the \emph{(Hyv\"{a}rinen) score function} of $P$ \citep{hyvarinen2005estimation}. Throughout this paper, we will shorten Langevin KSD to KSD for brevity, but note that other Stein operators can also be used to construct KSDs; see \cite{anastasiou2023stein} for a review. When $\E_{\bX \sim P}[\| \bs_p \|_2 ] < \infty$, the \emph{squared} KSD admits the following double-integral form \citetext{\citealp[Corollary 1]{barp2022targeted}; \citealp{gorham2017measuring}}
\begin{subequations}
    \begin{align}
        \ksd^2(Q, P)
        \;&=\;
        \E_{\bX, \bX' \sim Q}[ u_p (\bX, \bX') ] = \int_{\mathbb{R}^d} \int_{\mathbb{R}^d} u_p (\bx, \bx')Q(\diff\bx) Q(\diff\bx')
        \;,
        \nonumber
        \\
        u_p(\bx, \bx')
        \;&\coloneqq\;
        \bs_p(\bx)^\top \bs_p(\bx') k(\bx, \bx') + \bs_p(\bx)^\top \nabla_2 k(\bx, \bx')
        \\
        &\;\quad
        + \nabla_1 k(\bx, \bx')^\top \bs_p(\bx') + \nabla_1^\top \nabla_2 k(\bx, \bx') 
        \;,
    \end{align}
    \label{eq:Stein_kernel}%
\end{subequations}
whenever $D^2(Q, P) $ is well-defined. The function $u_p$ is also a reproducing kernel, known as the \emph{Stein (reproducing) kernel}. The main advantage of the KSD is that it is computable even if $p$ is unnormalized. Indeed, $u_p$ depends on $p$ only through $\bs_p$, which does not depend on the normalizing constant of $p$ (as the constant is cancelled due to the differentiation). The KSD can be straightforwardly estimated at $\cO(n^2)$ cost using a V-statistic estimator. It is formed by replacing $Q$ with its empirical version $Q_n$ based on independent observations $\X_n \coloneqq \{\bX_i \}_{i=1}^n$ drawn from $Q$:
\begin{align}
    D^2(\X_n) 
    \;&\coloneqq\;
    \ksd^2(Q_n, P) = 
    \frac{1}{n^2}\sum_{i=1}^n\sum_{j=1}^n u_p(\bX_i, \bX_j)
    \;,
    \label{eq: KSD V stat}
\end{align}
where we have overloaded the notation by writing $D^2: (\R^d)^n \to [0, \infty)$ as a test statistic. The V-statistic is non-negative, biased but consistent. 

An alternative estimator is a U-statistic, which differs from \eqref{eq: KSD V stat} by summing only over \emph{disjoint} index pairs $i \neq j$. Although a U-statistic is unbiased, it can take negative values. Therefore, we focus on V-statistics in this paper, and leave a discussion on how to extend our results to U-statistics to \Cref{app: extension to u-stats}. 

When the kernel $k$ is \emph{characteristic}, the KSD is \emph{$P$-separating}, meaning that $\ksd(Q, P) \geq 0$, with equality if and only if\ $Q = P$ for all $Q$ that finitely integrates $\|\bs_p\|_2$; see, e.g., \citet[Theorem 3]{barp2022targeted}. Most characteristic kernels used in kernel GOF tests are sufficiently smooth stationary kernels of the form $k(\bx,\bx') = h(\bx-\bx')$, where $h \in \mathcal{C}^2_b$ and $h(0)>0$. These include the squared-\emph{exponential kernel} $k(\bx, \bx') = \exp(-\| \bx - \bx' \|^2 / (2\lambda^2))$ and the \emph{inverse multi-quadric} (IMQ) kernel $k(\bx, \bx') = (1 + \| \bx - \bx' \|_2^2 /\lambda^2 )^{-b} $, where $b > 0$, and $\lambda > 0$ is a hyperparameter known as the \emph{bandwidth}. The IMQ kernel is often preferred in practice since the resulting KSD has the desirable property of \emph{$P$-convergence control} \citep{gorham2017measuring,barp2022targeted}. The bandwidth $\lambda$ also plays a crucial role in the performance of kernel-based tests \citep{reddi2015high,ramdas2015decreasing,huang2023highdimensional}, and a \emph{median-heuristic} \citep[Section 5]{fukumizu2009kernel} is often used to select its value in practice \citep{liu2016kernelized,chwialkowski2016kernel}.

\subsection{Kernel Goodness-of-Fit Testing}
\label{sec: KSD-based GOF Testing}

KSD is a natural choice of test statistic for GOF testing, since, by the $P$-separation property of KSD, testing $H_0: Q = P$ is equivalent to testing whether $D(Q, P) = 0$. As KSD is non-negative, A GOF test can therefore be constructed whereby $H_0$ is rejected for large values of the KSD estimate \citep{chwialkowski2016kernel,liu2016kernelized}. 
 KSD tests have been used for GOF testing with unnormalized models in a wide range of applications and data structures \citep{yang2018goodness,fernandez2020kernelized,xu2021stein,xu2021interpretable,amin2023kernelized}. 
Extensions have also been developed to address its limitations, such as high computational cost \citep{jitkrittum2017linear,huggins2018random}, difficulty in bandwidth selection \citep{schrab2022ksd} and lack of test power against certain alternatives \citep{liu2023using}.

A significant practical challenge with KSD tests is determining an appropriate decision threshold. A valid threshold can be obtained by considering the quantiles of the null distribution of $D^2(\X_n)$, but since this distribution is not usually tractable, bootstrapping is often employed to estimate it. One approach is to compute the empirical quantiles of bootstrap samples of the form
\begin{align}
    D^2_{\bW}(\X_n)
    \;=\;
    \frac{1}{n^2}\sum_{i=1}^n\sum_{j=1}^n (W_i - 1) (W_j - 1) u_p(\bX_i, \bX_j)
    \label{eq: bootstrap sample}
    \;,
\end{align} 
where $\bW \coloneqq (W_1, \ldots, W_n) \sim \textrm{Multinomial}(n; 1/n, \ldots, 1/n)$. This procedure, called \emph{Efron's bootstrap} or \emph{weighted bootstrap} \citep{arcones1992bootstrap,janssen1994weighted}, assigns a multinomial weight to each observation, which mimics recomputing KSD using data sampled with replacement from $\X_n$. Other bootstrap methods such as wild bootstraps \citep{leucht2013dependent,shao2010dependent} are also viable, particularly when the samples are potentially correlated \citep{chwialkowski2016kernel}, but we focus on the weighted bootstrap since we find that they perform similarly in our setting with independent samples; see \Cref{app: efron vs wild}. 

The test threshold is selected as the $(1-\alpha)$-quantile of the distribution of $D^2_{\bW}(\X_n)$ conditional on $\X_n$, i.e.,
\begin{align}
    q^2_{\infty, 1-\alpha}(\X_n)
    \;\coloneqq\;
    \inf\big\{ u \in \R: \; 1 - \alpha \leq \Pr\nolimits_{\bW}\big( D^2_{\bW}(\X_n) \leq u \;|\; \X_n \big) \big\}
    \;.
    \label{eq: boot quantile population}
\end{align}
In practice, this quantile is approximated with Monte-Carlo estimation by first drawing $B$ independent copies $\{ \bW^b \}_{b=1}^B$ of $\bW$, and then computing
\begin{align}
    q^2_{B, 1-\alpha}(\X_n)
    \;\coloneqq\;
    \inf\Bigg\{ u \in \R: \; 1 - \alpha \leq \frac{1}{B + 1} \bigg( \indicator\{D^2(\X_n) \leq u\} + \sum_{b=1}^{B} \indicator\{ D^2_{\bW^b}(\X_n) \leq u \} \bigg) \Bigg\}
    \label{eq: boot quantile}
    \;,
\end{align}
where $\indicator\{\cA\}$ denotes the indicator function for the event $\cA$. The KSD test then rejects $H_0$ if $D^2(\X_n) > q^2_{B, 1-\alpha}(\X_n)$. This test is asymptotically well-calibrated for  $H_0: Q=P$ \citep[Theorem 4.3]{liu2016kernelized}. When KSD is $P$-separating, this test is also \emph{consistent}, meaning that, whenever $Q \neq P$, the probability of rejection approaches one as the sample size $n \to \infty$ \citep{chwialkowski2016kernel,liu2016kernelized}.

\subsection{Robustness for Goodness-of-Fit Testing}
\label{sec: robustness GOF test}

Robustness of GOF tests refers to the lack of sensitivity of the test outcome to small model deviations \citep{rieder1982qualitative,lambert1982qualitative}. Model deviations can be formalized as a neighborhood $\cP_0$ around a nominal distribution $P$. The neighborhood $\cP_0$ is often called an \emph{uncertainty set}, and encodes the practitioner's uncertainty on $P$. One popular example for $\cP_0$ is \emph{Huber's contamination model} \citep{huber1964robust,huber1965robust}
\begin{align}
    \cP(P; \epsilon)
    \;\coloneqq\;
    \{ (1 - \epsilon') P + \epsilon' R: \; 0 \leq \epsilon' \leq \epsilon, \; R \in \cP(\R^d)\}
    \label{eq: Huber model}
    \;,
\end{align}
where $\epsilon \in [0, 1]$ is the maximal contamination ratio, and the probability measure $R$ acts as arbitrary contamination. This model is appropriate when practitioners believe that $P$ accurately describes all but a small proportion of the data. 

Huber's models have been studied in the context of robust GOF testing including \citet{huber1965robust,qin2017robust}, as well as in robust estimation \citep{hampel1974influence,huber2011robust}. Special interests lie in the case when $R$ is restricted to point masses, i.e., $R = \delta_\bz$ for some $\bz \in \R^d$, where robustness is often called \emph{bias-robustness} \citep{huber2011robust}. Beyond Huber's models, $\cP_0$ can also be chosen as \emph{density-band models} \citep{kassam1981robust}, which assume the density of the data-generating distribution is close to the model density up to a small error (see \Cref{sec: choosing uncertainty radius}). Moreover, $\cP_0$ can be set to a ball defined via a statistical divergence or metric, such as Hellinger distance \citep{lecam1973convergence}, Wasserstein metric \citep{gao2018robust}, and Maximum Mean Discrepancy \citep{sun2023kernel}.

Uncertainty sets $\cP_0$ provide a framework for assessing the robustness of a GOF test, by studying the rejection probability when $Q$ deviates from the nominal distribution $P$ but remains within $\cP_0$. In the robust testing literature, a common approach is to consider a sequence $\{\cP_0^n\}_{n=1}^\infty$ of uncertainty sets with \emph{shrinking} size as the sample size $n$ increases. This is formalized in the following notion of \emph{qualitative robustness}, inspired by \citet{rieder1982qualitative}.
\begin{definition}[Qualitative robustness to a sequence of neighborhood]
    \label{def: qualitative robustness}
    Let $\cP_0^1 \supseteq \cP_0^2 \supseteq \ldots$ be a sequence of subsets of $\cP(\R^d)$ that contain $P$ and such that $\cap_{n=1}^\infty \cP_0^n = \{P\}$. For any positive integer $n$, let $T_n: \big(\R^d\big)^n \to \R$ be a test statistic where a large value of $T_n$ suggests deviation from the null $H_0: Q = P$, and let $\gamma_n: \big(\R^d\big)^n \to \R$ be a function that computes the decision threshold. A sequence of hypothesis tests (indexed by $n$) that reject $H_0$ when $T_n > \gamma_n$ is \emph{qualitatively robust} to $\{\cP_0^n\}_{n=1}^\infty$ if, as $n \to \infty$,
    \begin{align}
        \sup\nolimits_{Q \in \cP_0^n } \big| \Pr\nolimits_{\X_n \sim Q}\big( T_n(\X_n) > \gamma_n(\X_n) \big) - \Pr\nolimits_{\X_n^\ast \sim P}\big( T_n(\X_n^\ast) > \gamma_n(\X_n^\ast) \big) \big|
        \;\to\; 0
        \;.
        \label{eq: qualitative robustness}
    \end{align}
\end{definition}
Intuitively, this notion of robustness asks how sensitive the test outcome is under infinitesimally small model deviation. The shrinking-size condition ensures that the rejection probability under any $Q \neq P$ does not trivially approach one due to the consistency of the test. By choosing the $\cP_0^n$ in \Cref{def: qualitative robustness} to be Prokhorov balls \citep{prokhorov1956convergence}, we would recover the qualitative robustness introduced in \citet[Definition 2.1]{rieder1982qualitative}, which parallels the conventional notion of qualitative robustness for estimators \citep[Remark 2]{rieder1982qualitative}. However, distributions in Prokhorov balls do not have a simple form, posing challenges to the analysis. Our definition extends it to a general sequence of neighborhood. In \Cref{sec: lack_robustness}, we will choose $\cP_0^n$ to be Huber's contamination models \eqref{eq: Huber model}, which both significantly simplifies the analysis and encompasses a wide range of relevant scenarios. We will show that, within Huber's neighborhood, the standard KSD test is \emph{not necessarily} qualitatively robust with stationary kernels, but is qualitatively robust with appropriately \emph{tilted} kernels.

Moreover, our definition of qualitative robustness is tied to a \emph{sequence} of shrinking neighborhood $\{ P_0^n \}_{n=1}^\infty$. Clearly, if a test is qualitatively robust to a sequence $\{ P_0^n \}_{n=1}^\infty$, then it is also qualitatively robust to any sequences that decay faster. Therefore, loosely speaking, the rate of decay of $\{ P_0^n \}_{n=1}^\infty$ characterizes the \emph{degree} of qualitative robustness. In \Cref{sec: lack_robustness}, we will explicitly derive the rate required for the standard KSD test to retain qualitative robustness.

However, qualitative robustness has its own limitations. It only concerns the insensitivity of a test to \emph{sufficiently small} model deviations, but offers no guarantees for deviations of a \emph{fixed} size. The latter scenario is more practically pertinent, as practitioners often need to account for a specific form of model misspecification and require the test to remain well-calibrated under that level of uncertainty. This can be formalized by relaxing the point null hypothesis $H_0$ to a composite hypothesis $\Hc_0: Q \in \cP_0$, and requiring calibration under $\Hc_0$. This motivates the following notion of \emph{quantitative robustness}.
\begin{definition}[Quantitative robustness to a single neighborhood]
\label{def: quantitative robustness}
    Given $\alpha \in (0, 1)$ and $\cP_0 \subset \cP(\R^d)$, a test is \emph{quantitatively robust to $\cP_0$} at level $\alpha$ if its rejection probability under \emph{any} $Q \in \cP_0$ does not exceeds $\alpha$.
\end{definition}
Quantitatively robust GOF tests have been developed for various types of uncertainty sets $\cP_0$, including Huber's contamination model \citep{huber1965robust} and neighborhoods defined by Kullback-Leibler divergence \citep{levy2008robust,yang2018robust} or $\alpha$-divergence \citep{gul2016robust}. These tests enjoy minimax optimality but require the normalizing constant of $P$ to be known, thus not applicable to unnormalized models. Two-sample tests that are quantitatively robust to Hellinger distance \citep{lecam1973convergence}, Wasserstein distance \citep{gao2018robust}, and Maximum Mean Discrepancy \citep{sun2023kernel,gao2021maximum} have also been proposed. However, to be used for GOF testing, they require approximation of $P$ by finite samples, which incurs extra approximation error and can be computationally demanding due to the non-trivial task of sampling from $P$. 

In \Cref{sec: robust KSD test}, we will choose $\cP_0$ to be a \emph{KSD ball} centered at $P$. We choose KSD balls over other types of neighborhood because \emph{(i)}  this choice naturally lends itself to a GOF test that is both easy to implement and guarantees robustness at little extra computational cost, and \emph{(ii)} it is easy to select the radius of such balls so that our proposed test is quantitatively robust to various types of contamination of interest, such as Huber's contaminations or density-band contaminations.

We conclude this section with a brief comparison between qualitative and quantitative robustness. Qualitative robustness concerns the limiting behavior of a test against some specific class of local alternatives. Thus, it is closely tied to the minimax separation boundary of a test \citep{ingster1987minimax,ingster1993asymptotically}. 
Moreover, it offers a notion of ``degree of robustness'' of a given test via the decay rate of the uncertainty neighbouhood sequence $\{\cP_0^n\}_{n=1}^\infty$. In contrast, quantitative robustness cannot be used for this purpose, because any consistent test is not quantitatively robust by definition. 
However, quantitative robustness has more \emph{practical} relevance, as it concerns a \emph{fixed} composite null set, which can be explicitly constructed to enforce robustness. This is why we have introduced two different notions of robustness.


\section{The (Lack of) Robustness of Existing Kernel Goodness-of-fit Tests}
\label{sec: lack_robustness}

We will now study the robustness of standard KSD GOF tests using stationary kernels in \Cref{sec: Non-robustness of KSD Test}, and then their tilted counterparts which are popular in the parameter estimation literature in \Cref{sec: robustness of KSD test with tilted kernel}. 

\subsection{Existing KSD Tests with Stationary Kernels are not Qualitatively Robust}
\label{sec: Non-robustness of KSD Test}

Our first result states that contamination of Huber's type can considerably affect the probability of the standard KSD test rejecting the null hypothesis $H_0: Q = P$. Our result holds with the bootstrap threshold defined in \eqref{eq: boot quantile population}, and all probabilities are taken over the randomness of both the sample and the bootstrap weights $\bW \sim \textrm{Multinomial}(n; 1/n, \ldots, 1/n)$. The proof is in \Cref{app: non-robust stationary}.
\begin{theorem}
    \label{thm: non robust stationary}
    Assume $\E_{\bX \sim P}[ \| \bs_p(\bX) \|_2^4 ] < \infty$, the function $\bx \mapsto \| \bs_p(\bx) \|_2$ is unbounded, $k(\bx, \bx') = h(\bx - \bx')$ with $h \in \cC_b^2$ and $h(0) > 0$, and assume the integrability conditions
    \begin{align*}
        \sup_{\bz \in \R^d} \|\bs_p(\bz)\|_2^4 \E_{\bX \sim P} \big[ h(\bX - \bz)^4 \big]
        \;<\; \infty 
         \quad \text{and} \quad 
        \sup_{\bz \in \R^d} \|\bs_p(\bz)\|_2^2 \E_{\bX \sim P} \big[ \|\nabla h(\bX - \bz)\|_2^2 \big]
        \;<\; \infty
        \;.
    \end{align*} 
    Then, for any test level $\alpha \in (0, 1)$ and any sequence $\{\epsilon_n\}_{n=1}^\infty$ with $\epsilon_n = o(n)^{-1}$, the following holds as $n \to \infty$,
    \begin{align*}
        &
        \sup_{Q \in \cP(P; \epsilon_n)}\big| \Pr\nolimits_{\X_n \sim Q, \bW}\big( D^2(\X_{n}) > q^2_{\infty, 1-\alpha}(\X_{n}) \big)
         - \Pr\nolimits_{\X_n^\ast \sim P, \bW}\big( D^2(\X_{n}^\ast) > q^2_{\infty, 1-\alpha}(\X_{n}^\ast) \big) \big|
         \\
         \;&\to\;
         1 - \alpha
         \;,
    \end{align*}
    where $\cP(P; \epsilon_n)$ is the Huber's contamination model defined in \eqref{eq: Huber model}.
\end{theorem}
\Cref{thm: non robust stationary} immediately implies that the standard KSD test is \emph{not} qualitatively robust to any sequence of Huber's models that satisfies the rate condition $\epsilon_n = o(n)^{-1}$. At first glance, the lack of qualitative robustness might seem trivial given that the KSD test is consistent, meaning that the rejection probability under any $Q \neq P$ approaches one as $n \to \infty$. However, one subtlety is that the set $\cP(P; \epsilon_n)$ can also shrink in size, since $\epsilon_n$ is allowed to decay. The condition $\epsilon_n = o(n)^{-1}$ imposes a lower bound on this decay rate. Intuitively, this rate condition requires $\epsilon_n n$, the expected number of contamination in the sample, to grow with $n$. In particular, this excludes the case where $\epsilon_n n$ is a constant. Moreover, since Huber's models $\cP(P; \epsilon_n)$ are contained within Prokhorov balls, \Cref{thm: non robust stationary} immediately implies non-qualitative robustness to Prokhorov neighborhoods. 

\begin{remark}[Unbounded Stein kernel]
\label{rem: unbounded stein kernels}
    \Cref{thm: non robust stationary} is a consequence of the fact that the Stein kernel evaluated at $\bx$, i.e., $\bx \mapsto u_p(\bx, \bx)$, is \emph{unbounded} when $P$ has an exploding score function. Exploding score functions are often associated with light tails. Examples of such distributions include those with a density of the form $p(\bx) \propto \exp(- \| \bx \|_2^r)$ with $r \geq 2$. These distributions have sub-Gaussian tails, and their score function has the form $\bs_p(\bx) = - r \bx \|\bx\|_2^{r-2}$, which is unbounded for $r \geq 2$. 
    On the other hand, this exploding-score condition excludes heavy-tailed models with bounded score functions, such as super-Laplacian or t-distributions \citep{gorham2017measuring,barp2022targeted}.  
\end{remark}
\begin{remark}[Moment conditions]
    The moment conditions in \Cref{thm: non robust stationary} are mild. For the example from \Cref{rem: unbounded stein kernels}, namely $p(\bx) \propto \exp(- \| \bx \|_2^r)$ with $r \geq 2$, direct computation shows that these conditions hold if $h(\bx)$ and $\nabla h(\bx)$ decay at least as fast as $\|| \bx \|_2^{-(r-1)}$ so as to cancel the growth of the score $\| \bs_p(\bx) \|_2 = r\|\bx\|_2^{r-1}$. Such kernels include squared-exponential kernels, IMQ kernels $h(u) = (1 + \| u \|_2^2)^{-s/2}$ with $s \geq r - 1$, and Mat\'{e}rn kernels with sufficient smoothness. \Cref{sec: contam gaussian} will provide numerical evidence using a Gaussian model.
\end{remark}

\begin{remark}[Connection to separation boundaries]
\label{rem: separation boundary}
    Let $\cQ_\delta \coloneqq \{Q \in \cP(\R^d): S(Q, P) \geq \delta \}$, where $S$ is some statistical divergence or metric. The \emph{separation boundary} \citep{ingster1987minimax,ingster1993asymptotically} of a test is the fastest decaying sequence $\{\delta_n\}_n$ such that the test power under any $Q \in \cQ_{\delta_n}$ still converges to 1 as $n \to \infty$. From this perspective, \Cref{thm: non robust stationary} implies that the separation boundary of the standard KSD test must decay at least with rate $\epsilon_n = o(n)^{-1}$, whenever $Q_{\delta_n}$ contains Huber-contamination models. This complements existing results on the separation boundary of KSD tests. The most relevant works are \citet{schrab2022ksd,hagrass2025minimax}, where they consider alternative distributions $Q$ that have a density with respect to either the Lebesgue measure or the target model $P$. In particular, their results cannot be used to derive \Cref{thm: non robust stationary}, since their density assumption on $Q$ is violated in our case. Indeed, we consider alternatives of the form $Q = (1-\epsilon_n)P + \epsilon_n \delta_\bz$, which involve a Dirac delta measure and thus has no density with respect to the Lebesgue measure nor any continuous measures, including $P$.
\end{remark}

\subsection{Existing KSD Tests with Tilted Kernels are Qualitatively Robust}
\label{sec: robustness of KSD test with tilted kernel}
Since unbounded Stein kernels are the main cause of the lack of robustness of existing KSD tests, a natural approach to enforce robustness is to choose a suitable $k$ so that the Stein kernel $u_p$ becomes bounded. We now show that this can be achieved using \emph{tilted} kernels \citep{barp2019minimum,matsubara2022robust}. Specifically, we will show in \Cref{thm: robust tilted} that, with a tilted kernel, a small proportion of contamination in the data has negligible impact on the outcome of the standard KSD test.

We first prove that tilted kernels give rise to bounded Stein kernels. The proof is deferred to \Cref{app: pf of bounded Stein kernel}.
\begin{lemma}[Bounded Stein kernel]
    \label{lem: bounded Stein kernel}
    Suppose $k(\bx, \bx') = w(\bx) h(\bx - \bx') w(\bx')$, where $h \in \cC_b^2$ is a stationary reproducing kernel, and the weighting function $w \in \cC_b^1$ satisfies the condition $\sup_{\bx \in \R^d} \| w(\bx) \bs_p(\bx) \|_2 < \infty$. Then 
    \begin{align*}
        \sup_{\bx, \bx' \in \R^d} | u_p(\bx, \bx') | \;\leq\; \sup_{\bx \in \R^d} u_p(\bx, \bx) \;=\; \tau_\infty \;<\; \infty.
    \end{align*}
\end{lemma}
Notably, the function $\bx \mapsto u_p(\bx, \bx)$ is non-negative since $u_p$ is a reproducing kernel. \citet[Theorem 9]{barp2022targeted} showed that KSD with such tilted kernels still satisfies $P$-separation, namely $\ksd(Q, P) = 0 \iff Q = P$, provided that $\| \bs_p \|_2$ grows at most root-exponentially and the spectral density of the translation-invariant kernel $k$---which exists by Bochner's Theorem \citep[Theorem 20]{berlinet2004reproducing}---is bounded away from zero on compact sets. 

Most stationary kernels used in KSD tests (such as the IMQ or squared-exponential kernels) satisfy the conditions in \Cref{lem: bounded Stein kernel}. Intuitively, the weighting function $w$ is used to counteract the growth of the score function $\bs_p$. When the score grows polynomially like $\| \bs_p(\bx) \|_2 =\cO(\|\bx\|_2^r)$ for some $r > 0$, it suffices to choose $w(\bx) = (1 + a^2 \| \bx \|_2^2)^{-b}$ for any $a > 0$ and $b \geq 1 / (2r)$. This weight function is common in the frequentist parameter estimation literature \citep{barp2019minimum} and the generalized Bayesian inference literature \citep{matsubara2022robust,altamirano2023robusta,altamirano2024,duran-martin2024outlierrobust}. 

The following result states that, with a tilted kernel, the outcome of the standard KSD test will not be significantly affected by small proportions of contamination in the data. The proof is provided in \Cref{app: robustness tilted}.
\begin{theorem}
\label{thm: robust tilted}
    Assume $\E_{\bX \sim P}[\| \bs_p(\bX) \|_2] < \infty$ and that $k$ is a tilted kernel satisfying the conditions in \Cref{lem: bounded Stein kernel}. Then, for any test level $\alpha \in (0, 1)$ and any sequence $\epsilon_n = o(n^{-1/2})$, the following holds as $n \to \infty$,
    \begin{align*}
        \sup_{Q \in \cP(P; \epsilon_n)} \big| \Pr\nolimits_{\X_n \sim Q, \bW}\big( D^2(\X_n) > q^2_{\infty, 1-\alpha}(\X_n) \big) - \Pr\nolimits_{\X_n^\ast \sim P, \bW}\big( D^2(\X_n^\ast) > q^2_{\infty, 1-\alpha}(\X_n^\ast) \big) \big|
        \;\to\;
        0
        \;,
    \end{align*}
    where $\cP(P; \epsilon_n)$ is the Huber's contamination model defined in \eqref{eq: Huber model}.
\end{theorem}
The LHS of the above convergence is the worst-case difference between the rejection probabilities of the standard KSD test with and without data corruptions. This result thus immediately implies that the standard KSD test is qualitative robust to any sequence of Huber's models so long as the contamination ratio decays sufficiently fast as $\epsilon_n = o(n^{-1/2})$. The condition on $\epsilon_n$ allows the expected number of contaminated data $n \epsilon_n$ to grow with $n$ but requires the growth rate to be slow. In particular, it is met when the expected number of contaminated data is bounded, i.e., $\epsilon_n = \cO(n^{-1})$. This condition is not an artifact of the proof: in \Cref{app: rate contam ratio}, we show empirically that this rate is tight, i.e., the tilted-KSD test is no longer qualitatively robust when $\epsilon_n = n^{-r}$ for any $r \leq 1/2$. 

\begin{remark}
    The intuition behind \Cref{thm: robust tilted} is that, when the conditions in \Cref{lem: bounded Stein kernel} are met so that the Stein kernel $u_p$ is bounded, the impact of any single outlier on the test statistic $\ksd^2(\X_n)$ can be bounded. In contrast, under the setting of \Cref{thm: non robust stationary}, where $k$ is translation-invariant and thus $u_p$ is unbounded, even a single outlier can drive $\ksd^2(\X_n)$ to infinity. This is the key motivation for using a tilted kernel to guarantee robustness: by choosing a suitable weighting function $w$, the Stein kernel can be made bounded, thereby mitigating the effect of outliers.
\end{remark}

\begin{remark}
    The boundedness condition on $\| w(\bx) \bs_p(\bx)\|_2$ required of \Cref{lem: bounded Stein kernel} and \Cref{thm: robust tilted} bears similarity with the ones imposed for KSD-based robust estimation methods \citep{barp2022targeted,matsubara2022robust}. For example, \citet[Proposition 7]{barp2019minimum} showed that their KSD-based estimator is \emph{globally bias-robust} assuming $\bx \mapsto \| \bs_p(\bx) \|_2 \int_{\R^d} \| k(\bx, \bx') \bs_p(\bx') \|_2 Q(\diff\bx')$ is bounded. This is slightly weaker than those we assume in \Cref{lem: bounded Stein kernel} and \Cref{thm: robust tilted}, but also harder to verify in practice. 
\end{remark}

\begin{remark}[Connection to separation boundaries, continued]
    \Cref{thm: robust tilted} implies that the separation boundary of the KSD tests cannot decay faster than $n^{-1/2}$ when considering alternatives that include Huber-contamination models (see also \Cref{rem: separation boundary}). This can be compared to existing results on the separation boundaries of tests based on \emph{Maximum Mean Discrepancy} \citep[MMD,][]{muller1997integral,gretton2012kernel}. MMD is a broader family of discrepancies that recovers KSD when the kernel is chosen to be the Stein kernel $u_p$ \citep[Theorem 1]{barp2022targeted}. The separation rates of MMD tests have been studied extensively  \citep{balasubramanian2021optimality,hagrass2024spectralGOF}, but these works assume the alternative $Q$ has a density with respect to either the Lebesgue measure or $P$, which does not hold in our contamination setting. An exception is \citet{hagrass2024spectral}, which does not make this assumption. However, their results are limited to the two-sample setting, where the target distribution $P$ is unknown and approximated by finite samples, again differing from our setup.
\end{remark}

Although tilted kernels enforce \emph{qualitative} robustness, they are \emph{not} enough to guarantee \emph{quantitative} robustness. Indeed, so long as $k$ is characteristic, the consistency of the standard KSD test implies that $\Pr_{\X_n \sim Q, \bW}(D^2(\X_n) > q^2_{\infty,1-\alpha}(\X_n)) \to 1$ whenever $Q \neq P$ \citetext{\citealt[Proposition 4.2]{liu2016kernelized}; \citealt[Theorem 3.3]{schrab2022ksd}}, so KSD tests with either stationary or tilted kernels are \emph{not} quantitatively robust to any neighborhoods \emph{strictly} larger than the singleton set $\{P\}$. We provide numerical evidence in \Cref{sec: contam gaussian}.


\section{Robust Kernel Goodness-of-fit Tests for KSD-Ball Uncertainty Sets}
\label{sec: robust KSD test}

\begin{algorithm}[t!]
    \caption{Robust-KSD (R-KSD) test for goodness-of-fit evaluation.}
    \label{alg:rksd}
    \begin{algorithmic}[1]
        \State {\bfseries Input:} Data $\X_n = \{\bx_i\}_{i=1}^n$; target distribution $P$; uncertainty radius $\theta$; bootstrap sample size $B$; test level $\alpha$.
        \State Compute test statistic $\Delta_\theta(\X_n)$ as defined in \eqref{eq: robust test statistic}.
        \For{$b = 1, \ldots, B$}
        \State Draw $\bW^b \sim \textrm{Multinomial}(n; 1/n, \ldots, 1/n)$ and compute bootstrapped sample $D_{\bW^b}^2(\X_n)$ by \eqref{eq: bootstrap sample}. 
        \EndFor
        \State Compute the (non-squared) bootstrapped quantile $q_{B, 1-\alpha}(\X_n)$ by \eqref{eq: boot quantile}.
        \State Reject $\Hc_0: Q \in \cB^\KSD(P; \theta)$ if $\Delta_\theta(\X_n) > q_{B, 1-\alpha}(\X_n)$.
    \end{algorithmic}
\end{algorithm}

We now propose a novel robust KSD test for the setting where the uncertainty set is a KSD-ball. Given $P \in \cP(\R^d)$ with density $p \in \cC^1$, we consider the composite hypotheses
\begin{align}
    \Hc_0: \; Q \in \cB^\KSD(P; \theta)
    \;,\qquad
    \Hc_1: \; Q \not\in \cB^\KSD(P; \theta)
    \;,
    \label{eq: KSD-ball hypotheses}
\end{align}
where $\theta \geq 0$ and $\cB^\KSD(P; \theta) \coloneqq \{Q: \; \ksd(Q, P) \leq \theta \}$. The advantage of using a KSD-ball as the uncertainty set is that it lends naturally to a simple, tractable test that is robust to the contamination models we reviewed in \Cref{sec: robustness GOF test}. We first describe our novel test, which we call \emph{robust-KSD} test, and show that it is quantitatively robust to KSD balls in the sense of \Cref{def: quantitative robustness}. We then discuss in \Cref{sec: choosing uncertainty radius} how to choose the uncertainty radius $\theta$ to incorporate common contamination models, such as Huber's contamination models and density-band models.

\subsection{A Robust KSD Test}

Given a prescribed test level $\alpha \in (0, 1)$, our robust KSD test uses the following test statistic
\begin{align}
    \Delta_\theta(\X_n) \;\coloneqq\; \max\big(0, \ksd(\X_n) - \theta \big)
    \label{eq: robust test statistic}
    \;,
\end{align}
where $\ksd(\X_n)$ is the square-root of the V-statistic \eqref{eq: KSD V stat}. The test rejects $\Hc_0$ for large values of $\Delta_\theta(\X_n)$. The decision threshold should be chosen to control the Type-I error rate for all possible $Q \in \cB^\KSD(P; \theta)$, i.e., we require $\gamma = \gamma(\X_n)$ so that $\sup_{Q \in \cB^\KSD(P; \theta)} \Pr\nolimits_{\X_n \sim Q, \bW}( \Delta_n(\X_n) > \gamma ) \leq \alpha$. One approach to construct $\gamma$ is to use deviation inequalities for bounded functions as done in \citet[Corollary 11]{gretton2012kernel}. However, such approach tends to be overly conservative \citep{gretton2012kernel}. Instead, we propose a bootstrap procedure to construct a decision threshold. We will show that choosing $\gamma = q_{\infty, 1-\alpha}(\X_n)$, the square-root of the (population) bootstrap quantile defined in \eqref{eq: boot quantile population}, gives the desired Type-I error control asymptotically. Our robust-KSD test therefore rejects $\Hc_0: Q \in \cB^\KSD(P; \theta)$ if $\Delta_\theta(\X_n) > q_{\infty, 1-\alpha}(\X_n)$. This is summarized in \Cref{alg:rksd}.

The validity and consistency of our robust KSD test is stated in the following result, proved in \Cref{sec: proof of thm: bootstrap validity}. In particular, it implies that our robust-KSD test is quantitatively robust to KSD balls \emph{in the infinite-sample limit}.

\begin{theorem}
    \label{thm: bootstrap validity}
    Suppose $\E_{\bX \sim P}[ \| \bs_p(\bX) \|_2 ] < \infty$ and $k(\bx, \bx') = w(\bx) h(\bx - \bx') w(\bx')$, with stationary reproducing kernel $h \in \cC_b^2$ and weighting function $w \in \cC_b^1$.  Define the set of distributions $\cP(\R^d; w) \coloneqq \{Q \in \cP(\R^d): \E_{\bX \sim Q}[ \| w(\bX)\bs_p(\bX) \|_2^4 ] < \infty \}$.
    \begin{enumerate}
        \item (Calibration) It holds that
        \begin{align*}
            \sup_{Q \in \cB^\KSD(P; \theta) \cap \cP(\R^d; w)} \limsup_{n \to \infty} \Pr\nolimits_{\X_n \sim Q, \bW}\big(\Delta_\theta(\X_n) > q_{\infty, 1-\alpha}(\X_n) \big) \;\leq\; \alpha \;.
        \end{align*}
        \item (Consistency) For any $Q \in \cP(\R^d; w)\backslash \cB^\KSD(P; \theta)$, it holds that
        \begin{align*}
        \lim_{n \to \infty} \Pr\nolimits_{\X_n \sim Q, \bW}\big( \Delta_\theta(\X_n) > q_{\infty, 1-\alpha}(\X_n) \big) \;=\; 1 
        \;. 
        \end{align*}
    \end{enumerate}
\end{theorem}
The condition $w \in \cC_b^1$ requires the weighting function $w$ and its gradient to be bounded. In particular, this holds for $w(x) = 1$, in which case $k(\bx, \bx') = h(\bx - \bx')$ reduces to an \emph{untilted} stationary kernel such as IMQ or squared-exponential kernels. \Cref{thm: bootstrap validity} also assumes the data-generating distribution $Q$ finitely integrates the fourth moment of the weighted score, i.e., $\| w(\bx) \bs_p(\bx) \|_2^4$. This automatically holds for any $Q \in \cP(\R^d)$ when $k$ is a tilted kernel satisfying the conditions in \Cref{lem: bounded Stein kernel}, in which case $\| w(\bx) \bs_p(\bx) \|_2$ becomes bounded.

In practice, the decision threshold $q_{\infty, 1-\alpha}$ is intractable, and we again use the Monte Carlo estimate $q_{B, 1-\alpha}$, defined as the squared root of \eqref{eq: boot quantile}, as an approximation. Compared with the standard KSD test targeting the point null $H_0: Q = P$, our robust-KSD test uses the same bootstrap procedure to compute the decision threshold, but has a slightly different test statistic $\Delta_\theta(\X_n)$ instead of $D^2(\X_n)$, given by \eqref{eq: KSD V stat}, to account for the composite null. Moreover, as $\theta \to 0$, the test statistic $\Delta_\theta$ approaches $D(\X_n)$, the test statistic of the standard test, and the KSD-ball $\cB^\KSD(P; \theta)$ falls to the singleton $\{P\}$, so the robust-KSD test reduces to the standard test. Our robust KSD test can hence be viewed as a generalization of the standard KSD test to the composite hypotheses in \eqref{eq: KSD-ball hypotheses}. In particular, the robust KSD test is also qualitatively robust whenever the conditions on $k$ and $s_p$ from \Cref{thm: robust tilted} hold; see \Cref{prop: qualitative robust tilted general} in \Cref{app: robustness tilted} for a formal statement and proof. 

Another advantage of our robust KSD test being a direct generalization of the standard test is that no extra computation is required to guarantee robustness. For a given $\theta > 0$, our robust test only requires a minor transformation of the test statistic of the standard KSD test. It hence has the same computational cost as the standard test, namely $\cO(n^2d)$. However, extra computation is potentially needed in determining $\theta$, as it might require optimizing the Stein kernel. This will be discussed in detail in \Cref{sec: choosing uncertainty radius}.

\begin{remark}[Pointwise and uniform controls]
    The Type-I error control shown in Theorem 4 is \emph{pointwise} over the null distributions $Q$. For robust tests, a more desirable control is often a \emph{uniform} one, whereby the supremum over $Q$ is taken \emph{before} the limit over $n$ \citep[Chapter 11]{lehmann2008testing}. For the proposed test, a uniform control can be shown if the bootstrapped quantile is replaced by the ground-truth quantile, say $q_{1-\alpha}^\ast(\X_n)$. That is (see also \Cref{rem: uniform bound} in the appendix),
    \begin{align*}
        \limsup_{n \to \infty} \sup_{Q \in \cB^\KSD(P; \theta) \cap \cP(\R^d; w)} \Pr\nolimits_{\X_n \sim Q, \bW}\big(\Delta_\theta(\X_n) > q_{1-\alpha}^\ast(\X_n) \big) \;\leq\; \alpha \;.
    \end{align*} 
    However, extending this to the bootstrapped quantile requires uniformly bounding the bootstrap approximation errors, which is non-trivial. We leave this as an open question for future research. In the literature, most kernel-based robust tests that offer uniform controls are two-sample tests using either permutation \citep{schrab2024robust} or deviation bounds \citep{sun2023kernel}. Yet, these approaches are either not applicable to our one-sample setting or too conservative. This is why our test uses bootstrapping. Another kernel robust test that uses bootstrapping is \citet{key2021composite}, but their Type-I error control is also pointwise. In \Cref{app: KSD dev}, we propose an alternative robust test that leverages deviation bounds in a manner similar to \citet{sun2023kernel} to achieve uniform controls, but at the expense of a lower test power.
\end{remark}

\subsection{Choosing the Uncertainty Radius}
\label{sec: choosing uncertainty radius}
A crucial design choice in our robust tests is the uncertainty radius $\theta$. This should be guided by the types of contamination that the practitioner is willing to tolerate. In this section, we discuss how to choose $\theta$ for Huber's contamination models and density-band models when using KSD-balls based on tilted kernels.

Firstly, suppose we are considering Huber's contamination model $\cP(P; \epsilon_0)$ for some $\epsilon_0 \in [0, 1]$. Then $\theta$ should be chosen such that the KSD-ball $\cB^\KSD(P; \theta)$ contains $\cP(P; \epsilon_0)$. The following result, proved in \Cref{app: KSD balls and contam models}, shows how to achieve this given an upper bound on the Stein kernel $\tau_\infty = \sup_{\bx \in \mathbb{R}^d} u_p(\bx, \bx)$.

\begin{proposition}
    \label{prop: KSD bound huber model}
    Suppose $k$ satisfies the conditions in \Cref{lem: bounded Stein kernel} and $\E_{\bX \sim P}[ \| \bs_p(\bX) \|_2 ] < \infty$. Let $\epsilon_0 \in [0, 1]$. Then $\cP(P; \epsilon_0) \subseteq \cB^\KSD(P; \theta)$ if $\theta = \epsilon_0 \tau_\infty^{1/2}$ and this bound is tight, i.e., 
    $\sup_{Q \in \cP(P; \epsilon_0)} D(Q, P) = \epsilon_0 \tau_\infty^{1/2}$.
\end{proposition}
This result suggests that, when at most a proportion $\epsilon_0$ of data is corrupted, setting $\theta = \epsilon_0 \tau_\infty^{1/2}$ will ensure quantitative robustness to $\cP(P; \epsilon_0)$. Note that using a similar proof, Huber's contamination models can also be related to the MMD balls; see \Cref{lem: MMD ball Huber model} in  \Cref{app: mmd tests}. 

One practical issue is that $\tau_\infty$, the supremum of the Stein kernel, does not always admit a closed-form expression. One way to compute it is by numerical optimization of $\bx \mapsto u_p(\bx, \bx)$, but this assumes the optimizer converges and requires extra computation, thus not suitable in high dimension or when evaluation of the score function is costly. We propose an alternative approach, where $\tau_\infty$ is approximated by the maximum of the Stein kernel evaluated at the observed data, i.e., $\max_{i=1, \ldots, n}u_p(\bX_i, \bX_i)$. This approach requires no extra computation and gives a reasonable estimate with moderate or large $n$. In \Cref{sec: Experiments}, we use this approach and show empirically that the resulting test is still well-calibrated despite this approximation. Further discussions on how this approximation affects the test performance can be found in \Cref{app: tau_infty}, where we demonstrate that this approach still controls the Type-I error rate, even with a small sample size.

\begin{remark}
    \Cref{prop: KSD bound huber model} holds for Huber's contamination models of the form $Q = (1-\epsilon) P + \epsilon R$ with $\epsilon \leq \epsilon_0$. In particular, it holds for \emph{any} contamination distribution $R$. This agnosticism can be useful in practice, because the exact form of contamination is often unknown. On the other hand, when prior knowledge about $R$ \emph{is} available, it is also possible to incorporate it into the proposed test. For example, if $R$ is known to have a support bounded by some $B > 0$, then the bound in \Cref{prop: KSD bound huber model} can be tightened by replacing the worst-case bound $\tau_\infty = \sup_{\bx \in \R^d} u_p(\bx, \bx)$ with the localized version $\tau_\infty = \sup_{\| \bx\|_2 \leq B} u_p(\bx, \bx)$. As discussed in \citet{fauss2021minimax}, such assumptions are realistic in, e.g., highly regulated experimental environments, where any outliers are known to lie within bounded regions.
\end{remark}

Another uncertainty model commonly studied in the literature is the \emph{density-band} model \citep{kassam1981robust,hafner1993construction}, defined as distributions whose density function lies within an error band of a nominal model with density $p$, i.e., $\{Q \in \cP(\R^d): \; Q \textrm{ has density } q \textrm{ with } | q(\bx) - p(\bx)| \leq \delta(\bx) \textrm{ for all } \bx \}$, for some function $\delta: \R^d \to [0, \infty)$. It can be shown that density-band models can be rewritten as Huber's contamination models \eqref{eq: Huber model} with additional constraints on the outlier distribution \citep{fauss2016old}. However, compared with Huber's models, density-band models have the advantage of being more interpretable and more natural for certain forms of model disparity such as heavy tails. The following result suggests how the uncertainty radius $\theta$ should be chosen for such models and is proved in \Cref{sec: pf of KSD bound density band}.
\begin{proposition}
    \label{prop: KSD bound density band}
     Suppose $k$ satisfies the conditions in \Cref{lem: bounded Stein kernel} and $\E_{\bX \sim P}[\|\bs_p(\bX)\|_2] < \infty$. Furthermore, assume $Q, P \in \cP(\R^d)$ admit positive densities $q, p$ on $\R^d$ and $p \in \cC^1$. If $| q(\bx) - p(\bx) | \leq \delta(\bx)$ for some function $\delta: \R^d \to [0, \infty)$ such that $\delta_0 \coloneqq \int_{\R^d} \delta(\bx) \diff\bx < \infty$, then $\ksd(Q, P) \leq \delta_0 \tau_\infty^{1/2}$.
\end{proposition}
The bound in \Cref{prop: KSD bound density band} is not tight, as its proof involves bounding an integrand that can take negative values by its absolute value. Nevertheless, as we will show in \Cref{sec: exp: misspecified tails}, this bound is not overly loose and can still be useful. In particular, it allows us to design a KSD test that is robust to tail misspecification while still maintaining non-trivial power.

\Cref{prop: KSD bound huber model} and \ref{prop: KSD bound density band} also reveal a limitation of our robust KSD test---it can be overly conservative to alternatives that are not the intended contamination but still lie within the chosen KSD-ball uncertainty set. For example, to ensure robustness to Huber's contamination models \eqref{eq: Huber model} with tolerance $\epsilon_0$, \Cref{prop: KSD bound huber model} suggests setting the uncertainty radius to $\theta = \epsilon_0 \tau_\infty^{1/2}$. However, the resulting KSD-ball $\cB^\KSD(P; \theta)$ contains not only Huber's models, but also other distributions such as density-band models. As a result, the robust KSD test will control Type-I errors for all these distributions, regardless of whether they are the intended contamination type. This limitation is not unique to our method; it is a generic drawback of all robust tests based on uncertainty sets \citep{fauss2021minimax}. 

\section{Numerical Experiments}
\label{sec: Experiments}
We will now evaluate the proposed GOF tests using both synthetic and real data. Unless otherwise mentioned, all standard KSD tests are based on an IMQ kernel $k(\bx, \bx') = h_\mathrm{IMQ}(\bx - \bx')$ where $h_\mathrm{IMQ}(\bu) = (1 + \| \bu \|_2^2 / \lambda^2)^{-1/2}$ with a bandwidth $\lambda^2 > 0$ selected via the median heuristic, i.e., $\lambda_\mathrm{med} \;=\; \mathrm{Median}\big\{ \| \bX_i - \bX_j \|_2: \; 1 \leq i < j \leq n \big\}$. All tilted-KSD and robust-KSD tests are based on a \emph{tilted} IMQ kernel with weight $w(\bx) = (1 + \| \bx - \ba \|_2^2 / c)^{-b}$, where $\ba \in \R^d$ and $c > 0$. Intuitively, $\ba$ and $c$ respectively centers and scales the input. We fix $\ba = 0$ and $c = 1$ in all experiments, as all data will always be centered and on a suitable scale. More generally, we could replace $\| \bx - \ba \|_2^2 / c$ by a weighted norm of the form $(\bx - \ba)^\top C (\bx - \ba)$, where $C \in \R^{d \times d}$ is a pre-conditioning matrix, chosen possibly as the empirical covariance matrix or robust estimates of it. Since our experiments will focus on sub-Gaussian models, we choose $b = 1/2$. This ensures the Stein kernel is bounded.

All tests have nominal level $\alpha = 0.05$. The probability of rejection is computed by averaging over 100 repetitions, and the $95\%$ confidence intervals are reported. Our robust-KSD test will be shortened as \emph{R-KSD}. Code for reproducing all experiments can be found at \href{https://github.com/XingLLiu/robust-kernel-test}{\texttt{github.com/XingLLiu/robust-kernel-test}}.

\subsection{Toy Gaussian Model}
\label{sec: contam gaussian}
We first consider a Gaussian model $ P = \cN(0, 1)$, in which case the score function of the model is $\bs_p(x) = -x$ and is hence unbounded. This toy example is simplistic and not representative of the unnormalized models our test is most suited for, but it will nonetheless be helpful to study our algorithmic choices, to verify \Cref{thm: non robust stationary} and 
\Cref{thm: robust tilted} numerically, and to compare against alternative robust tests. 
 
\subsubsection{Decay Rate of Weighting Function}
We first draw random samples $\X_n$ of size $n=500$ from $Q = (1-\epsilon) P + \epsilon \delta_z$, for different values of $\epsilon \in [0, 1]$ and $z \in \R$. This setting mimics the presence of outliers at $z$. 
Given a stationary kernel $h$, a natural question is ``how does the choice of weight $w$ affect the Stein kernel and the test power?''. On the left-hand side of \Cref{fig: stein kernels}, we plot the Stein kernel $x \mapsto u_p(x, x)$ for different values of $b$. As $b$ grows, the tails of the Stein kernel are progressively down-weighted. We then plot the rejection probability of the standard KSD tests using these Stein kernels when the outlier is $z = 10$ on the right-hand side of \Cref{fig: stein kernels}. As $b$ increases, the rejection probability decreases, suggesting that the test power is lower when the tails of the Stein kernel are overly down-weighted. This is not surprising since the Stein kernel decays faster for larger values of $b$, making the test insensitive to model deviations at the tails. Ideally, $w$ should decay just enough so that $u_p$ remains bounded, but not too fast as it would lose power. This highlights the trade-off between robustness and power in the choice of $w$. A similar trade-off has been observed in robust estimation, where an overly decaying $w$ can impede the estimation accuracy \citep[Appendix F.2]{barp2019minimum}. Henceforth, we choose $b = 1/2$ to balance this trade-off, as it is the smallest value that ensures $u_p$ is bounded for sub-Gaussian models.

\begin{figure}
    \begin{minipage}[t]{0.55\textwidth}
        \centering
        \includegraphics[width=1.\textwidth]{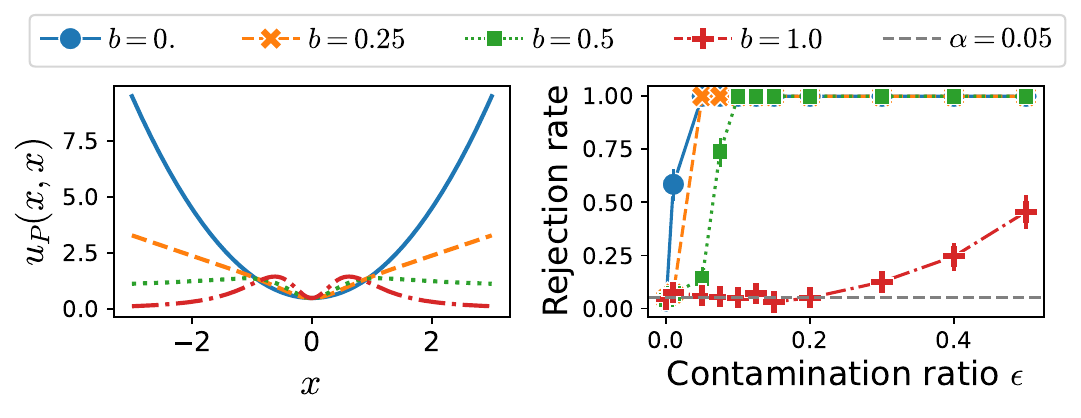}
        \caption{\emph{Left.} Stein kernel for $P=\mathcal{N}(0,1)$ and an IMQ kernel tilted by $w(x) = (1 + x^2)^{-b}$.  The larger $b$ is, the more the tails of the function $x \mapsto u_p(x,x)$ are down-weighted. The choice $b=0$ corresponds to no weighting, reducing to an IMQ kernel.  \emph{Right.} The rejection probability under contamination by $R=\delta_z$ with $z = 10$.}
        \label{fig: stein kernels}
    \end{minipage}
    \hfill
    \begin{minipage}[t]{0.435\textwidth}
        \centering
        \includegraphics[width=.85\textwidth]{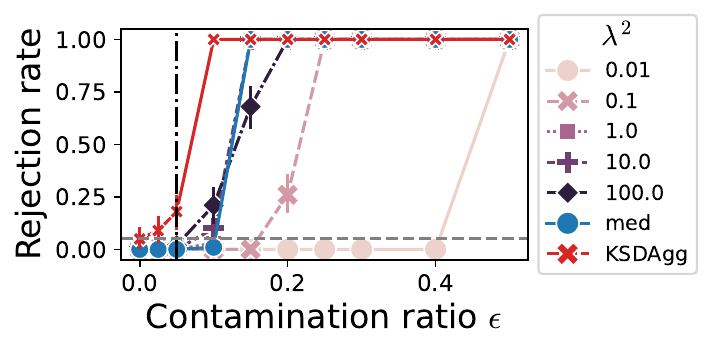}
        \caption{Rejection probability of robust-KSD with different bandwidths $\lambda$. ``med'' is the median heuristic. ``KSDAgg'' is the test of \citet{schrab2022ksd}. The dashed line is $\alpha = 0.05$. The vertical line is the maximal proportion of contamination $\epsilon_0 = 0.05$ controlled by robust-KSD.}
        \label{fig: bw robust}
    \end{minipage}
\end{figure}

\begin{figure}[t]
    \centering
    \includegraphics[width=1.\textwidth]{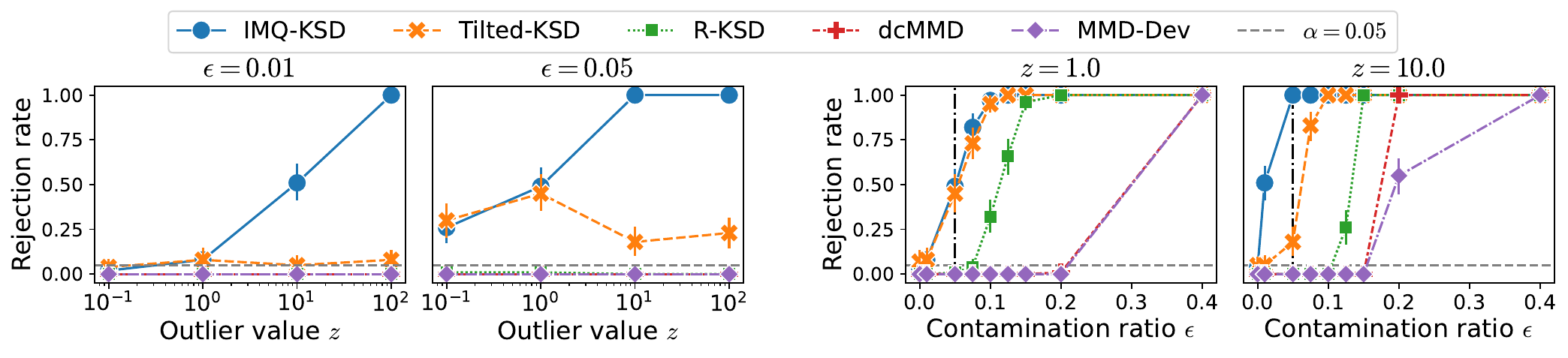}
    \caption{Rejection probability under an outlier-contaminated Gaussian model with different outlier values $z$ and contamination ratios $\epsilon$. The grey dotted horizontal line is the test level $\alpha = 0.05$, and the black dash-dot vertical line corresponds to $\epsilon_0 = 0.05$. The KSD tests with IMQ kernel lack both qualitative and quantitative robustness since they reject even for small $z$ or $\epsilon$. The tilted-KSD test is more robust in cases where $z$ or $\epsilon$ are larger, but ultimately still reject the null due to their lack of quantitative robustness.}
    \label{fig: gauss outliers}
\end{figure}

\subsubsection{Kernel Bandwidth}
We now use the same data set to investigate how the choice of the kernel bandwidth $\lambda$ affects the performance of the robust-KSD test. It is well-known that the bandwidth plays a crucial role in kernel-based tests and can considerably affect the ability of the test to detect model disparities \citep{gretton2012kernel,ramdas2015decreasing,reddi2015high,schrab2022ksd,schrab2023mmd}. In general, a smaller bandwidth is better at detecting local differences, while a larger bandwidth is more suited for global deviations \citep{schrab2023mmd}. A common strategy for bandwidth selection in kernel-based testing is the median heuristic. \citet{schrab2022ksd,schrab2023mmd} also proposed testing frameworks that aggregate the test result with multiple bandwidths to give the final decision, thereby avoiding bandwidth selection and achieving higher test power. A natural question is therefore: ``how does the bandwidth affects the robustness of the proposed test?''. 

We run the robust-KSD test with various choices of bandwidths $\lambda^2 \in \Lambda \cup \{ \lambda_\mathrm{med}^2 \}$, where $\Lambda = \{ 0.01, 0.1, 1, 10, 100\}$ and the uncertainty radius $\theta$ is set to control at most $\epsilon = 0.05$ contaminations in the sample. We also compare it with the KSDAgg test of \citet{schrab2022ksd} using the same tilted kernel and aggregating over the collection of bandwidths $\Lambda$. Other configurations of KSDAgg follow the recommendations in \citet[Section 4.2]{schrab2022ksd}. This experiment is repeated 200 times to reduce randomness, and the results are reported in \Cref{fig: bw robust}. For any choice of $\lambda$, the robust-KSD test is able to control the Type-I error when $\epsilon \leq \epsilon_0$, where $\epsilon$ is the proportion of contamination in the data, suggesting that robust-KSD remains well-calibrated for all bandwidths. If instead $\epsilon > \epsilon_0$, all robust-KSD tests eventually reject $\Hc_0$, with $\lambda$ chosen by median heuristic achieving the highest power. The KSDAgg test is not robust and more sensitive to contamination than all robust-KSD tests at all contamination levels $\epsilon$. This is not surprising given that KSDAgg is designed to \emph{maximize} the test power over multiple bandwidths, rendering this test more sensitive to a wide range of alternatives, including contaminated models. As a result, we henceforth do not include KSDAgg in our experiments. A similar study for the \emph{standard} KSD test is given in \Cref{app: bw robust}.

\subsubsection{Standard versus Robust KSD Tests}
Using the same data set, we then compare KSD tests with a stationary kernel to tilted-KSD tests and our proposed robust-KSD test in \Cref{fig: gauss outliers}. As shown in the left two plots, with a contamination ratio $\epsilon$ of $0.01$ or $0.05$, the standard IMQ-KSD test rejects the point null with high probability for large outlier values $z$. This aligns with the lack of qualitative robustness of tests based on stationary kernels proved in \Cref{thm: non robust stationary}. In contrast, the standard Tilted-KSD test rejects with lower probability for \emph{all} $z$ values, aligning with the qualitative robustness result with tilted kernels in \Cref{thm: robust tilted}. Notably, our theoretical results in both \Cref{sec: lack_robustness} and \Cref{sec: robust KSD test} apply only to \emph{fixed} kernels, thus excluding kernels based on observed data such as those selected via median heuristic. However, our numerical results suggest that the conclusion of our theory may hold more broadly.  

The Type-I error of the tilted-KSD test still exceeds the nominal level $\alpha=0.05$, suggesting that a tilted kernel alone is not enough to enforce \emph{quantitative} robustness. Notably, the rejection probability of Tilted-KSD is highest when $z = 1$ but declines for larger $z$. This is because $u_p(z, z)$ with the tilted kernel peaks at around $|z| \approx 1$ and then decreases with $z$ (see \Cref{fig: stein kernels}), making the test more sensitive when $|z| \approx 1$ but less sensitive to large outliers. In contrast, running our robust-KSD test with a pre-set contamination control $\epsilon_0 = 0.05$, we can see from the left two plots of \Cref{fig: gauss outliers} that it remains well-calibrated for all values of outlier $z$. The results also suggest that R-KSD is conservative since its rejection probability is close to 0; we attribute this to the fact that the robust tests are designed to control \emph{all} types of contamination, of which the injected outlier is only one specific type. To demonstrate the non-trivial power of the robust tests, we show in the right two plots the rejection probability against different contamination ratios $\epsilon$. As $\epsilon$ grows, R-KSD eventually is capable of rejecting the null hypothesis $\cH_0: Q \in \cB^\KSD(P; \theta)$ with test power approaching one. 

We also compared our tests with two existing robust kernel tests, namely the robust MMD tests of \citet{schrab2024robust} (denoted \emph{dcMMD}) and of \citet[Eq.~38]{sun2023kernel} (denoted \emph{MMD-Dev}). Both tests are provably quantitatively robust to Huber's contamination models, but they are two-sample tests that require extra samples from the model $P$. When the same number of samples from $Q$ and from $P$ are used and the cost of simulation is negligible, the cost of a single evaluation of MMD comparable to that of KSD. However, this cost could be prohibitively expensive when $P$ is a more complex model. Implementation details of these tests are deferred to \Cref{app: mmd tests}. As shown in \Cref{fig: gauss outliers}, both tests are more conservative than our proposed KSD tests. The conservativeness of dcMMD is because it is designed for a contamination model that is \emph{stronger} than Huber's model, while the conservativeness of MMD-Dev is because it uses a deviation inequality to construct the decision threshold, which is known to be conservative; see, e.g., \citet{gretton2012kernel}.

\begin{figure}[t]
    \centering
    \includegraphics[width=0.8\textwidth]{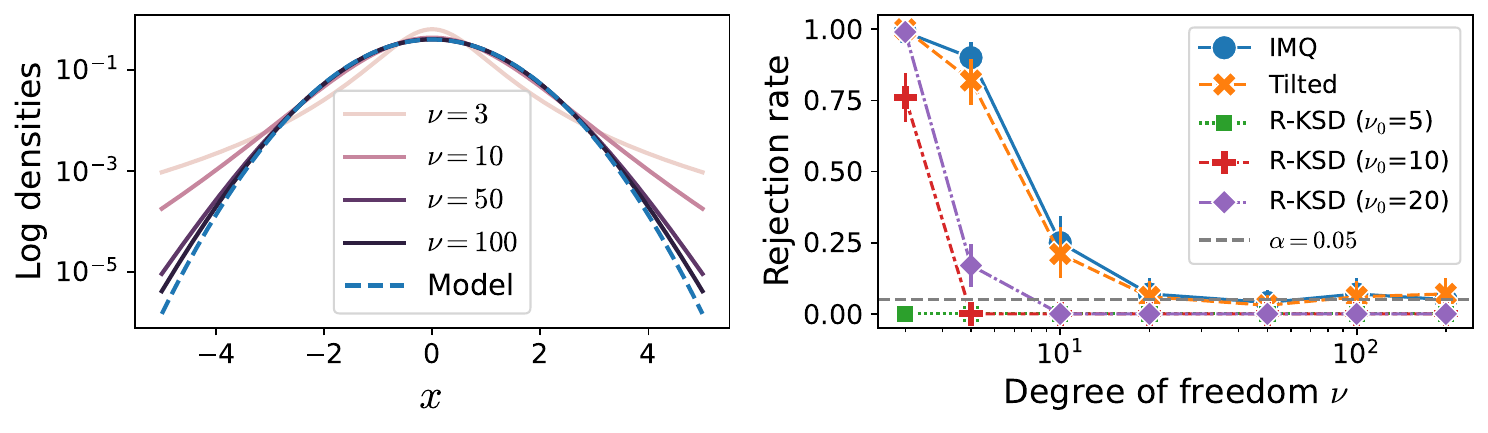}
    \vspace{-.8em}
    \caption{Heavy-tailed experiment. \emph{Left}. Log densities of a standard Gaussian model and scaled t-distributions with different degree-of-freedom (dof) and moment-matched to Gaussian. \emph{Right}. Rejection probability of the standard and robust tests with different $\theta$ set to control the cases $\nu \geq \nu_0$ for different values of $\nu_0$.}
    \label{fig: heavy tail}
\end{figure}

\subsubsection{Misspecified Tails}
\label{sec: exp: misspecified tails}
Finally, we now change the data-generating mechanism to study the impact of misspecified tails. We sample from $Q_\nu = t_\nu \sqrt{(\nu - 2)/\nu} $, where $t_\nu$ is the Student's t-distribution with degree-of-freedom (dof) $\nu > 2$, and the scaling factor $\sqrt{(\nu - 2)/\nu}$ ensures its second moment matches that of $P$. We set the uncertainty radius so that the R-KSD test is robust to $Q_\nu$ with $\nu \geq \nu_0$ for different values of $\nu_0$. This is achieved using \Cref{prop: KSD bound density band} to derive a bound on $\ksd(Q_\nu, P)$, which is discussed in detail in \Cref{prop: KSD ball fat tails} of \Cref{app: KSDball_fat_tails}. The results are reported in \Cref{fig: heavy tail}. For small values of $\nu$, the tail of the model is misspecified, leading the robust tests with $\nu_0 = 10$ or $20$ to reject with high probability because the KSD between $Q_\nu$ and $P$ is large. As $\nu $ grows, $Q_\nu$ converges weakly to the standard Gaussian, making the model well-specified in the limit $\nu \to \infty$. Consequently, neither the standard tests or the robust tests reject the null hypothesis for large $\nu$. The robust test with $\nu_0 = 5$ shows no power because, to be robust to tails that are so misspecified, the uncertainty radius $\theta$ computed using \Cref{prop: KSD bound density band} must be very large, thus resulting in a small test statistic (cf.\ \eqref{eq: robust test statistic}) and low test power. Overall, the robust tests are more conservative than the standard one. This is because the robust tests control \emph{all} possible contaminations within a KSD-ball, while misspecified tails are only one of the many possible such forms of contamination.

\subsection{Gaussian-Bernoulli Restricted Boltzmann Machines}

Our next experiment is an example of energy-based model called a Gaussian-Bernoulli Restricted Boltzmann Machine (RBM) model \citep{cho2013gaussian}. RBMs have been used in a wide range of scientific applications \citep[see][for a review]{fischer2014training,zhang2019review}, and they are a common benchmarks for assessing kernel GOF tests \citep{liu2016kernelized,jitkrittum2017linear,schrab2022ksd}. RBMs are latent-variable models with joint density $p(\bx, \bh) \propto \exp(\frac{1}{2} \bx^\top B \bh + \bb^\top \bx + \bc^\top \bh - \frac{1}{2} \| \bx \|_2^2 )$, where $\bx \in \R^d$ is an observable variable, $\bh \in \{\pm 1\}^{d'}$ is a binary hidden variable with latent dimension $d'$, and $B \in \R^{d \times d'}, \bb \in \R^d$ and $\bc \in \R^{d'}$ are model parameters. Computing the normalizing constant of the marginal $p(\bx)$ requires summing over $2^{d'}$ terms, so it becomes intractable when $d'$ is large. However, its score function has a closed-form expression $\bs_p(\bx) = b - \bx + B \mathrm{tanh}( B^\top \bx + \bc )$, where $\mathrm{tanh}$ is applied entry-wise. We set $d = 50$ and $d' = 10$, and randomly initialize $B, \bb, \bc$ by sampling each entry independently from a $\cN(0, 1)$. We then generate $n = 500$ samples from the model by block Gibbs sampling following \citet{cho2013gaussian,jitkrittum2017linear}, and randomly replace $\epsilon \in [0, 1]$ proportions of the data with outliers drawn from $R = \cN(0, 0.1^2I_d)$; see the left and middle of \Cref{fig: rbm}. As shown by the plots, the outliers are clearly abnormal since they have low density values.

\begin{figure}[t]
    \centering
    \includegraphics[width=.9\textwidth]{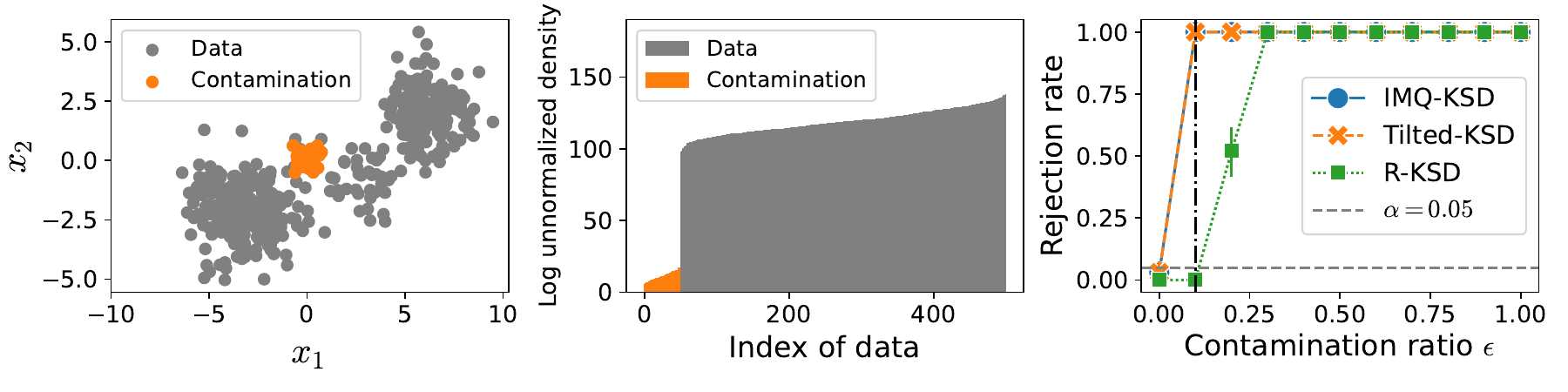}
    \caption{Gaussian-Bernoulli RBM experiment. \emph{Left.} Data generated from $P$ and injected contamination in the first two dimensions. \emph{Middle.} Log unnormalized densities of the data and contamination ordered from small to large; the injected contamination is indeed abnormal since they have much lower densities. \emph{Right.} Probability of rejection with a Gaussian-Bernoulli RBM against the contamination ratio $\epsilon$. The robust tests are calibrated to control the Type-I error with no more than $\epsilon_0 = 0.1$ proportion of contamination.}
    \label{fig: rbm}
\end{figure}

We plot the rejection probabilities of the tests on the right-hand side of \Cref{fig: rbm}. The standard IMQ-KSD and Tilted-KSD tests are not well-calibrated under the composite null $\Hc_0: Q \in \cB^\KSD(P; \theta)$, while robust-KSD with $\theta$ chosen by setting $\epsilon_0 = 0.1$ controls the rejection probability below the level $\alpha = 0.05$ when the contamination ratio $\epsilon$ does not exceed $\epsilon_0$. On the other hand, when $\epsilon$ exceeds the maximal allowed proportion $\epsilon_0$, the robust test rejects $\Hc_0$ with probability approaching $1$, thus showing its power.

When generating the synthetic data from $P$, we run a Gibbs sampler for $7000$ steps and discard the first $2000$ as burn-in and keep every $10$-th datum to reduce correlation. This takes $58.40$ seconds on average, compared with $0.84$ seconds required to perform the robust-KSD test. As a result, it will clearly not be reasonable to use the robust MMD tests from \citet{sun2023kernel,schrab2023mmd} here as simulating a large enough number of samples from $P$ would cost orders of magnitude more than the cost of the GOF test. Moreover, the samples generated through Gibbs sampling are no longer i.i.d., thus violating the assumptions in our theoretical results. Extension to the non-independent case could potentially be made following the approach in \citet{Cherief-Abdellatif2019}, which studied the robustness of estimators (instead of GOF tests) under correlated data.

\subsection{A Kernel Exponential Family Model for Density Estimation}

\begin{figure}[t]
    \centering
    \begin{subfigure}[b]{0.52\linewidth}
        \centering
        \includegraphics[width=1\textwidth]{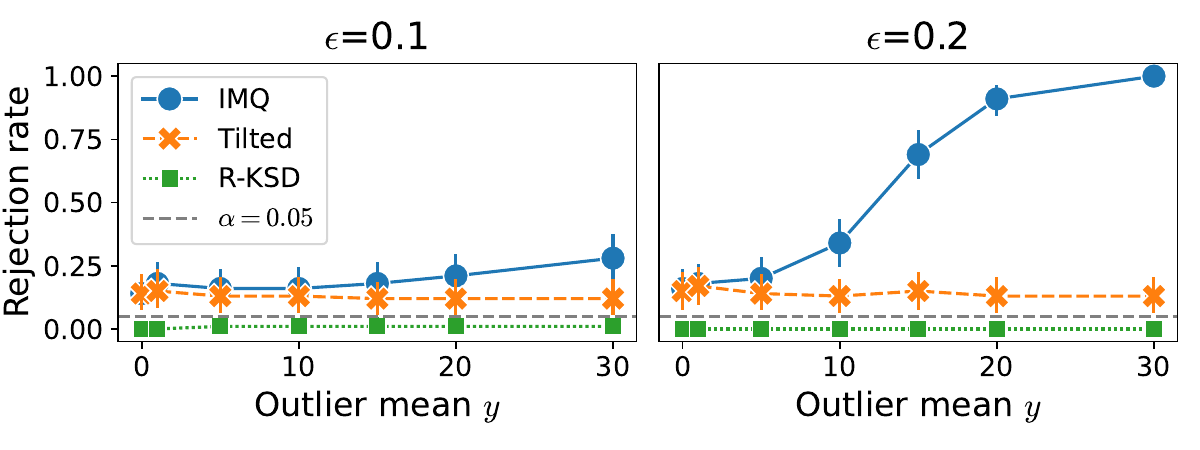}
    \end{subfigure}
    \hfill
    \begin{subfigure}[b]{0.4\linewidth}
        \centering
        \includegraphics[width=1\textwidth]{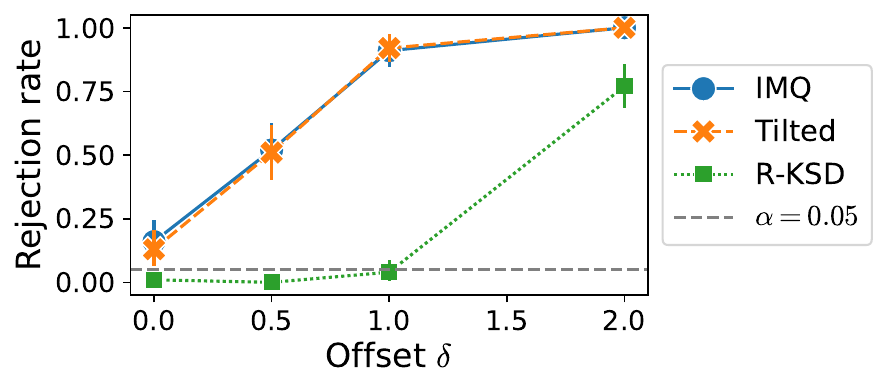}
    \end{subfigure}
    \begin{subfigure}[b]{0.52\linewidth}
        \centering
        \includegraphics[width=1\textwidth]{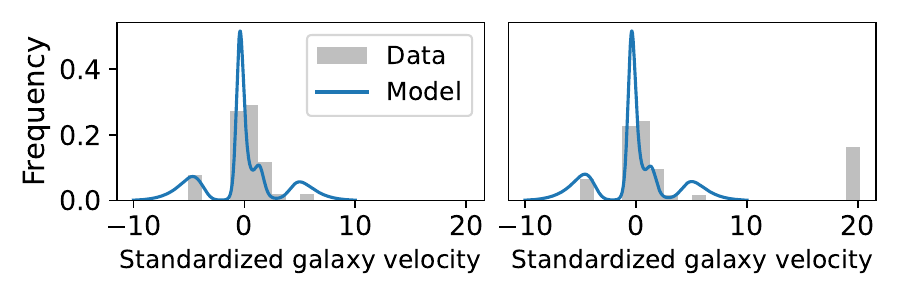}
        \caption{Level}
        \label{fig: kef level}
    \end{subfigure}
    \hfill
    \begin{subfigure}[b]{0.4\linewidth}
        \centering
        \includegraphics[width=.75\textwidth]{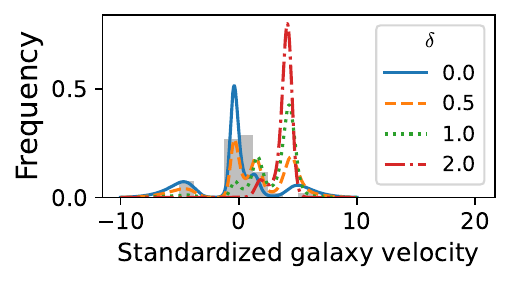}
        \caption{Power}
        \label{fig: kef power}
    \end{subfigure}
    \hspace{2em}
    \vspace{-.5em}
    \caption{Probability of rejection (top) and data and fitted models (bottom) for the KEF experiment. \emph{(a)} Extra outliers are added to the data; data and fitted models with no contamination (bottom left) and contamination ratio $\epsilon = 0.2$ (bottom right). \emph{(b)} Different offset values $\delta$ are added to the fitted parameter to create model deviation. The radius $\theta$ is chosen to control contamination of up to a proportion of $\epsilon_0 = 0.2$.}
    \label{fig:kef}
\end{figure}

Our next example concerns density estimation using a kernel exponential family (KEF) model \citep{canu2006kernel}. Given a reproducing kernel $k_0: \R^d \times \R^d \to \R$ and a reference density $p_0$ on $\R^d$, we have $p_\eta(\bx) \propto p_0(\bx) \exp(- f_\eta(\bx) )$, where $f_\eta$ lies in the RKHS $ \cH_{k_0}$ associated with $k_0$. We follow the setting in \citet{matsubara2022robust} which uses a finite-basis approximation of the RKHS function so that $f_\eta(\bx) = \sum_{l=1}^{25} \eta_{l} \phi_l(\bx)$, where $\eta = (\eta_1,\ldots,\eta_{25})^\top \in \R^{25}$ and $\phi_l(\bx) = (x^l/\sqrt{l!}) \exp(- x^2 / 2)$ for $l = 1, 2, \ldots, 25$. 

We are interested in testing the goodness-of-fit of the KEF model under data contamination. Suppose that the practitioners believe the observed data are subject to contamination, so that a \emph{robust} estimator for $\eta$ is used. Robust parameter estimation for exponential family models was studied in \citet{barp2019minimum}, which proposed a minimum-distance estimator using KSD with tilted kernels. With a suitable weighting function, the resulting estimator is \emph{bias-robust} \citep[Proposition 7]{barp2019minimum} and thus not susceptible to contamination. However, how to test the goodness-of-fit of the estimated model in this setting remains an open question. The standard IMQ-KSD test is not suitable since it is not robust and can thus falsely reject a good model due to contamination. Using a tilted kernel is also not enough since it is not quantitatively robust and hence can suffer from uncontrolled Type-I error. We therefore propose to use our robust-KSD test for this task. One challenge with the KEF model is that its density does not have a closed form, and although sampling from $P$ is feasible in this low-dimensional example ($d=1$), it would become challenging in high dimensions, making MMD-based tests not well-suited for this model. Our robust-KSD test, however, does not suffer from this limitation.

We use the data set as \citet{matsubara2022robust,key2021composite}, which is a $1$-dimensional data set of $82$ galaxy velocities \citep{postman1986probes,roeder1990density}. 
To avoid using the same data for model training and testing, we randomly split the data into equal halves, each containing $n_\mathrm{data} = 41$ data points. To each half, we then add $n_{\mathrm{ol}}$ independent draws from $R=\cN(z, 0.1^2)$ for some value of $z$ for the mean contamination, so that the resulting data set has a size of $n = n_{\mathrm{data}} + n_{\mathrm{ol}}$ and a contamination ratio of $\epsilon = n_{\mathrm{ol}} / (n_{\mathrm{data}} + n_{\mathrm{ol}})$. We then fit the KEF model to one half of the data, and perform the tests on the other half. Both the robust minimum KSD estimation and the robust-KSD test use the same tilted kernel $k(x, y) = w(x) h_\Lambda(x-y) w(y)$ with $w(x) = (1 + x^2)^{-1/2}$ and a sum of IMQ kernels $h_\Lambda(x-y) = \sum_{\lambda^2 \in \Lambda}(1 + | x - y|^2 / (2\lambda^2))^{-1/2}$ where $\Lambda = \{0.6, 1, 1.2\}$. The choice of the weighting function follows \citet{matsubara2022robust} and the use of a sum kernel is recommended in \citet{key2021composite}. Direct computation shows that the Stein kernel $ u_p$ is bounded in this case. We choose $\epsilon_0 = 0.2$ when computing $\theta$ for R-KSD. 

The results are shown in \Cref{fig:kef}. We first study the level of our tests in \Cref{fig: kef level}. As the outlier mean $z$ increases, the IMQ-KSD test rejects the null $H_0: Q = P$ with increasing probability, suggesting that this is due to the presence of outliers. Notably, this happens even though the parameter estimator is not susceptible to contamination; see the bottom row of \Cref{fig: kef level} for the fitted model with and without outliers. When the tilted kernel was used, the standard KSD test has a lower rejection probability, although it is still above the nominal level. On the other hand, the robust-KSD test is well-calibrated for all values of $z$.

We then study the power of KSD tests in \Cref{fig: kef power}.  To show the robust test has non-trivial power when no outliers are present, we fit the model parameters $\hat{\eta}$ on half of the data (without contamination) and replace the first component by $\hat{\eta}_1 + \delta$, where $\delta \in \{0, 0.5, 1, 2\}$ is an error offset, so that the resulting model has a poor fit. Other experimental setups including the choice of $\theta$ remain the same as before. As expected, the robust test has lower power than the standard tests; however, as the offset value $\delta$ increases, the robust test eventually rejects with high probability, thus showing its power under no contamination.

\subsection{Limitations for Multimodal Models}
\label{sec: limitations}

\begin{figure}[t]
    \centering
    \includegraphics[width=1\textwidth]{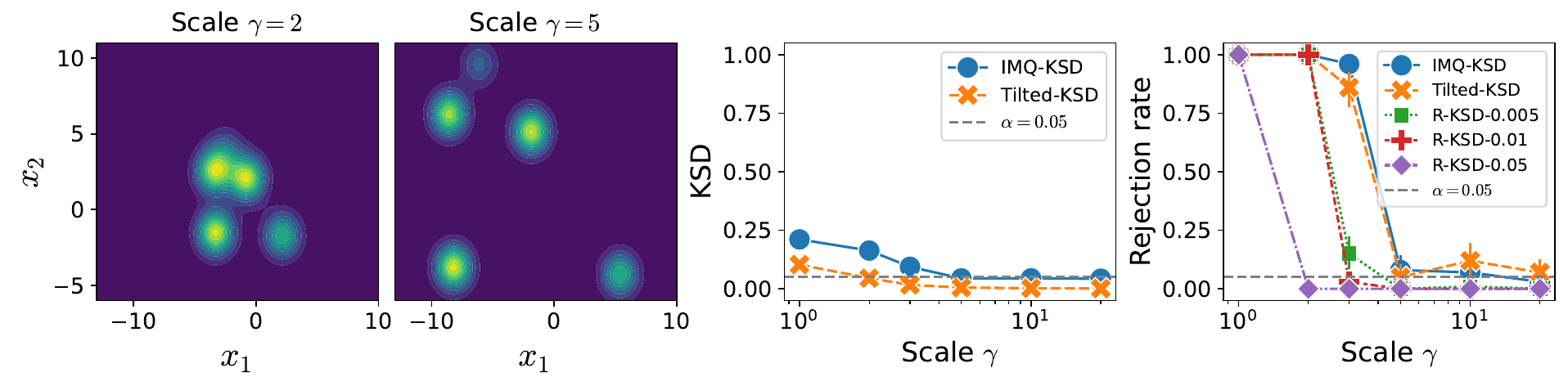}
    \caption{Mixture-of-Gaussian experiment. \emph{Left}. Contour plots of model densities. \emph{Middle}. KSD value with IMQ and tilted kernels. \emph{Right}. Rejection probability of the standard and robust tests. ``R-KSD-0.01'' refers to the R-KSD test with uncertainty radius $\theta$ chosen by setting the maximal contamination tolerance to $\epsilon_0=0.01$, and similarly for the others.}
    \label{fig: mixture}
\end{figure}

We investigate the performance of the robust test when the model has multimodality and the data are drawn from the same model with misspecified mixing ratios, thus revealing a limitation of our robust-KSD test. It is well-known that the standard KSD test performs poorly with mixture models with well-separated modes \citep{liu2023using}. This is intuitively due to the myopia of the score function, making the test blind to disparities in the mixing ratios, and it is a general issue for KSD and other score-based discrepancies \citep{wenliang2020blindness,zhang2022healing}. We demonstrate that our robust test can also suffer from this limitation, which is made even more prominent due to the relaxation of the point null to a KSD-ball. We consider a 2-dimensional mixture of Gaussian model $P = \sum_{j=1}^5 \pi_j \cN(\gamma \mu_j, I_2)$, where $\pi_j \propto u_j$ with $u_1, \ldots, u_5 $ drawn independently from $\mathrm{Uniform}(0, 1)$, each mean $\mu_j \in \R^d$ has entries drawn independently from $\mathrm{Uniform}(-2, 2)$, and $\gamma > 0$ is a scaling factor controlling the mode separation. The data-generating distribution $Q$ has a different set of mixture ratios that are randomly generated. For R-KSD, the uncertainty radius is chosen to be $\theta=\epsilon_0 \tau_\infty^{1/2}$ with different values of $\epsilon_0$, and we label the corresponding result as R-KSD-$\epsilon_0$. 

The results are reported in \Cref{fig: mixture}. As the scale $\gamma$ increases, the modes become more separated, leading to smaller values of $\ksd(Q, P)$, as shown in the left and middle plots. Consequently, $Q$ eventually becomes indistinguishable from $P$ under KSD, and all tests fail to reject with high probability. This is a result of the aforementioned myopia of KSD, which affects both the standard and the robust tests and regardless of whether IMQ or tilted kernels are used. The opposite is observed when $\gamma$ is small and the modes overlap, where all tests can correctly reject the null with high probability. Compared with the standard tests, the robust one is more conservative because it controls all possible contamination; nevertheless, it still achieves non-trivial power when $\epsilon_0$, thus also the radius $\theta$, is set to a small value.

\section{Conclusion and Discussion}
\label{sec: Discussions and Conclusion}

This paper studied the robustness of the kernel GOF test of \citet{liu2016kernelized,chwialkowski2016kernel}. We showed that existing KSD-based tests lack qualitative robustness when stationary kernels are used, but achieve it when employing tilted kernels inspired by robust estimation methods \citep{barp2019minimum,matsubara2022robust}. Since qualitative robustness alone does not ensure calibration under fixed contamination levels, we proposed a novel robust KSD test that provably controls the Type-I error when the data-generating distribution lies in a KSD-ball around the reference model, while consistently rejecting alternatives outside of it. We then discussed how to select the radius of this ball under different contamination models, namely Huber's contamination model and the density-band model. Empirical results on synthetic and real data sets demonstrate that our robust GOF test is well-calibrated and has non-trivial power. We conclude by discussing promising directions for future work.

Firstly, we established consistency of the proposed robust KSD test against \emph{all} alternative distributions outside the null set, i.e., $\Hc_1: Q \not\in \cB^\KSD(P; \theta)$. In the broader robust testing literature, however, it is also common to restrict the alternative set to a neighborhood centered at a specific alternative $P' \neq P$ and study minimax properties against such hypothesis \citep{huber1965robust,levy2008robust,gul2016robust,sun2023kernel}. Understanding whether our robust KSD test achieves minimax optimality in this setting is an interesting future work.

Secondly, as discussed in \Cref{sec: robustness GOF test}, our proposed robust KSD test does not require the Stein kernel $u_p$ to be bounded. A bounded Stein kernel is convenient as the resulting KSD-ball $\cB^\KSD(P; \theta)$ contains many contamination models of interest, such as Huber's model with an \emph{arbitrary} noise distribution. However, this could be too strict in some applications where outliers cannot take an arbitrary value, e.g., when the domain is bounded or the data acquisition procedure is highly regulated to prevent extreme outliers. In these cases, even with an unbounded kernel, the KSD-ball could still capture all contamination models of interest. Unbounded Stein kernels possess appealing theoretical properties such as convergence-control of moments \citep{kanagawa2022controlling}. Exploring whether such kernels can lead to more powerful robust-KSD tests is another worthwhile future research direction.


\acks{The authors would like to thank Antonin Schrab, Andrew Duncan, Matias Altamirano, and Charita Dellaporta for helpful discussions. XL was supported by the President’s PhD Scholarships of Imperial College London, the EPSRC StatML CDT programme [EP/S023151/1], and the Alan Turing Enrichment Scheme Award. FXB was supported through the EPSRC grant [EP/Y011805/1].}


\newpage

\appendix

{
\begin{center}
\Large
    \textbf{Appendix}
\end{center}
}

\etocsettocdepth.toc{subsection}
\tableofcontents

\section{Proofs of Theoretical Results}
\label{app:proof_theoretical_results}
This section holds proofs of the theoretical results. Throughout the appendix, given i.i.d.\ random variables $\bX_1, \ldots, \bX_n \sim Q \in \cP(\R^d)$ and an integrable function $f: (\R^d)^n \to \R$, we will write for brevity $\E[ f(\bX_1, \ldots, \bX_n) ] = \int_{\R^d} \cdots \int_{\R^d} f(\bx_1, \ldots, \bx_n) Q(\diff\bx_1) \cdots Q(\diff\bx_n)$.

\subsection{Proof of \Cref{thm: non robust stationary} and Related Preliminary Results}
\label{app: non-robust stationary}

\paranoheading{Overview of proof}
For any $\epsilon \in [0, 1]$ and $R \in \cP(\R^d)$, a random variable $\bX_i$ drawn from $Q = (1-\epsilon) P + \epsilon R$ can be written as $\bX_i = (1 - \xi_i) \bX_i^\ast + \xi_i \bZ_i$, where $\xi_i \sim \mathrm{Bernoulli}(\epsilon)$, $\bX_i^\ast \sim P$ and $\bZ_i \sim R$ are independent, and the equality is understood as equality in distribution. Defining the random variable $M' = \sum_{i=1}^n \xi_i$, then $M' \sim \mathrm{Binomial}(n, \epsilon)$ and represents the number of contaminated data in a random sample $\X_n$ drawn from $Q$. We will first show that the conclusion in \Cref{thm: non robust stationary} holds \emph{conditional} on the event $\{ M' \geq m'_0 \}$ for some integer $m'_0 \in [0, n]$, i.e., when there are at least $m'$ contaminated data in the sample. This will be sufficient to show \Cref{thm: non robust stationary} by proving that this event occurs with high probability. 

We will use the following notation in the rest of this section: Let $\X_n^\ast$ be a random sample drawn from $Q$. For any integer $m' \in [0, n]$ giving the number of corrupted data and any $\bz \in \R^d$, we define $\X_{n,\bz} = \X^\ast_{n-m'} \cup \{ \bz \}^{m'}$.

The rest of this section is organized as follows:
\begin{itemize}
    \item We first present three preliminary results. \Cref{lem: bound vstat} and \Cref{lem: bound boot quantile} respectively bound the test statistic $D^2(\X_{n,\bz})$ and $q^2_{\infty,1-\alpha}(\X_{n,\bz})$ for a given number of corrupted data, $m'$, by quantities of the form $a_n u_p(\bz, \bz) + b_n$ for some $a_n, b_n$. These results allow us to compare the relative rate at which $D^2(\X_{n,\bz})$ and $q^2_{\infty,1-\alpha}(\X_{n,\bz})$ scale with $\bz$. Both results will be used to prove \Cref{prop: non-robust stationary}, which states that the claim in \Cref{thm: non robust stationary} holds conditionally on the number of corrupted data.
    \item \Cref{pf: bound vstat}, \ref{pf: bound boot quantile}, \ref{pf: prop non-robust stationary} and \ref{pf: thm non-robust stationary} provide the proofs for \Cref{lem: bound vstat}, \Cref{lem: bound boot quantile}, \Cref{prop: non-robust stationary} and \Cref{thm: non robust stationary}, respectively.
\end{itemize}

\begin{assumption}
\label{assump: score and kernel}
    Assume $\E_{\bX \sim P}[\| \bs_p(\bX) \|_2^4] < \infty$ and $k(\bx, \bx') = h(\bx - \bx')$ for some $h \in \cC_b^2$ with $h(0) > 0$.
\end{assumption}

\begin{lemma}
    \label{lem: bound vstat}
    Assume $\sigma_\infty^2 \coloneqq \sup_{\bz \in \R^d} \E_{\bX \sim P}[ u_p(\bX, \bz)^2 ] < \infty$ and suppose \Cref{assump: score and kernel} holds. Then, for any $\delta > 0$ and integers $0 < m' < n$, the following event holds with probability at least $1 - \delta$:
    \begin{align*}
        D^2(\X_{n,\bz}) \;\geq\; \epsilon_{m, n}^2 u_p(\bz, \bz) - \frac{2\sigma_\infty\epsilon_{m, n}}{\sqrt{\delta m}} 
        \;,
    \end{align*}
    where $\epsilon_{m,n} = m'/n$ and $m = n - m'$.
\end{lemma}
\begin{lemma} 
    \label{lem: bound boot quantile}
    Suppose \Cref{assump: score and kernel} holds. Assume \emph{(i)} $\sigma_\infty^2 \coloneqq \sup_{\bz \in \R^d} \E_{\bX \sim P}[ u_p(\bX, \bz)^2 ] < \infty$ and \emph{(ii)} $\xi_P^2 \coloneqq \E_{\bX \sim P}[u_p(\bX, \bX)^2] < \infty$.  Then there exist constants $C, C_{P, 1}$, and $C_{P, 2}$ depending only on $P$ such that, for any $\delta > 0$ and integers $0 < m' < n$, the following event holds with probability at least $1 - \delta$:
    \begin{align*}
        q^2_{\infty, 1 - \alpha}(\X_{n,\bz})
        \;\leq\; 
        \frac{C\log(1 / \alpha)}{n \sqrt{\delta}} \epsilon_{m, n} u_p(\bz, \bz)
        + \frac{C_{P, 1}}{n \delta} + \frac{C_{P, 2} \log(1/\alpha)}{n \sqrt{\delta}} 
        \;,
    \end{align*}
    where $\epsilon_{m,n} = m'/n$ and $m = n - m'$, and the probability is over the joint distribution of $\X_{n,\bz}$ and $\bW \sim \textrm{Multinomial}(n; 1/n, \ldots, 1/n)$.
\end{lemma}
\begin{proposition}
    \label{prop: non-robust stationary}
    There exists a constant $C>0$ such that the following holds. Suppose \Cref{assump: score and kernel} holds. Also assume the following integrability conditions
    \begin{align*}
        \sup_{\bz \in \R^d} \|\bs_p(\bz)\|_2^4 \E_{\bX \sim P} \big[ h(\bX - \bz)^4 \big]
        \;<\; \infty \;,
         \quad \text{and} \quad 
        \sup_{\bz \in \R^d} \|\bs_p(\bz)\|_2^2 \E_{\bX \sim P} \big[ \|\nabla h(\bx - \bz)\|_2^2 \big]
        \;<\; \infty
        \;.
    \end{align*} 
    Further assume that $\bz \mapsto \| \bs_p(\bz) \|_2$ is unbounded. Then for any $s\in[0, 1)$, there exists $\bz \in \R^d$ such that, for any sample size $n$ and integer $m' \in \big[\frac{1}{2}n^{1-s}, n\big)$ giving the number of corrupted data, the condition
    \begin{align*}
        n^{\frac{(1-s)}{2}} \;>\; 2\sqrt{2} C \log\left(\frac{4}{\alpha}\right) 
    \end{align*}
    implies
    \begin{align*}
        \Pr\nolimits_{\X_{n-m'}^\ast \sim P, \bW}\big( D^2(\X_{n,\bz}) > q^2_{\infty, 1-\alpha}(\X_{n,\bz}) \big) \;\geq\; 1 - n^{-(1-s)}.
    \end{align*}
\end{proposition}
The assumption $m' \in \big[\frac{1}{2}n^{1-s}, n\big)$ ensures there is a considerable amount of corrupted data in the sample, while excluding the case where all data are contaminated ($m'=n$). The constant factor $1/2$ is arbitrary.

\subsubsection{Proof of \Cref{lem: bound vstat}}
\label{pf: bound vstat}
\begin{proof}
The test statistic can be decomposed in three terms: a term with uncorrupted samples, a term with corrupted samples, and a term for interactions between these two groups:
\begin{align*}
    D^2(\X_{n,\bz})
    \;&=\;
    \frac{1}{n}\sum_{1 \leq i, j \leq n} u_P(\bX_i, \bX_j)
    \\
    \;&=\;
    \frac{1}{n^2} \sum_{1 \leq i, j \leq m} u_p(\bX_i^\ast, \bX_j^\ast)
    + \frac{2}{n^2} \sum_{1 \leq i \leq m < j \leq n} u_p(\bX_i^\ast, \bz)
    + \frac{1}{n^2} \sum_{m < i, j \leq n} u_p(\bz, \bz)
    \\
    \;&=\;
    \frac{m^2}{n^2} D^2(\X_m^\ast) + \frac{2(n-m)}{n^2} \sum_{i=1}^m u_p(\bX_i^\ast, \bz) + \frac{(n-m)^2}{n^2} u_p(\bz, \bz)
    \\
    \;&=\;
    (1-\epsilon_{m,n})^2 D^2(\X_m^\ast) + 2\epsilon_{m,n}(1-\epsilon_{m,n}) S_{m, \bz} + \epsilon_{m,n}^2 u_p(\bz, \bz)
    \tagaligneq \label{eq: KSD U-stat decomposition}
    \;,
\end{align*}
where in the last line we have defined $S_{m, \bz} \coloneqq m^{-1} \sum_{i=1}^m u_p(\bX_i, \bz)$ and $\epsilon_{m,n} \coloneqq (n-m)/n$. 

The remainder of the proof proceeds by lower-bounding the first two terms with high probability. The first term can be lower-bounded by 0, since $D^2(\X_m^\ast)$ is a V-statistic and is hence non-negative. This lower bound will be relatively good as we would expect $D^2((\X_m^\ast)$ to approach zero at a root-m rate. 

Bounding the second term in \eqref{eq: KSD U-stat decomposition} requires bounding $S_{m, \bz}$. Because by assumption $\bX_i = \bX_i^\ast \sim P$ are i.i.d.\ for $i = 1, \ldots, m$, the term $S_{m, \bz}$ is a sum of i.i.d.\ random variables $u_p(\bX_i^\ast, \bz)$. Moreover, $u_p(\bX_i^\ast, \bz)$ are zero-mean, because the proof of \citet[Proposition 1]{gorham2017measuring} shows that $\E_{\bX^\ast \sim P}[u_p(\bX^\ast, \cdot)] = 0$ whenever $\E_{\bX^\ast \sim P}[ \| \bs_p(\bX^\ast) \|_2 ] < \infty$, which holds under \Cref{assump: score and kernel}. We can therefore compute the variance of $S_{m,\bz}$ as
\begin{align}
    \Var(S_{m, \bz})
    \;=\;
    \E\left[S_{m, \bz}^2\right]
    \;=\;
    \frac{1}{m} \E_{\bX_1^\ast \sim P}\left[u_p(\bX_1^\ast, \bz)^2\right]
    \;\leq\;
    \frac{\sigma_\infty^2}{m}
    \label{eq: var bound}
    \;,
\end{align} 
which is finite since $\sigma_\infty^2 < \infty$ by assumption. We can hence apply Chebyshev's inequality. Choosing
\begin{align}
    t \;\coloneqq\; \left( \frac{\sigma_\infty^2}{\delta m} \right)^{\frac{1}{2}}
    \;=\; \frac{\sigma_\infty}{\sqrt{\delta m}}
    \;,
    \label{eq: t2 choice}
\end{align} 
the set $\cA_1 \coloneqq \big\{ | S_{m, \bz} | \leq t \big\}$ holds with probability at least $1-\delta$. Furthermore, on $\cA_1$, we have by \eqref{eq: KSD U-stat decomposition} that
\begin{align*}
    D^2(\X_n)
    \;&\geq\;
    (1-\epsilon_{m,n})^2 D^2(\X_m^\ast) - 2\epsilon_{m,n}(1-\epsilon_{m,n}) t + \epsilon_{m,n}^2 u_p(\bz, \bz)
    \;\geq\;
    - 2\epsilon_{m, n} t + \epsilon_{m, n}^2 u_p(\bz, \bz)
    \;,
\end{align*}
where in the last step we have used $\epsilon_{m,n} \in [0, 1]$. This shows the claim.
\end{proof}

\subsubsection{Proof of \Cref{lem: bound boot quantile}}
\label{pf: bound boot quantile}
To prove this lemma, we will need the following concentration bound for quadratic forms of multinomial random vectors. This is a special instance of the Hanson-Wright inequality due to \citet{adamczak2015note} when the dependent random vectors follow multinomial distributions. 
\begin{lemma}[Hanson-Wright inequality; \citealt{adamczak2015note}, Theorem 2.5]
    \label{lem: hanson-wright inequality}
    Let $\bW = (W_1, \ldots, W_n) \sim \textrm{Multinomial}(n; 1/n, \ldots, 1/n)$. Then there exists constant $C > 0$ such that, for any $n \times n$ matrix $A$ with entries $a_{ij} \in \R$ and any $\delta > 0$, we have
    \begin{align*}
        \Pr\nolimits_{\bW} \left( \left| \bW^\top A \bW - \mu \right|
        \geq
        C \| A\|_{\mathrm{F}} \log \left( \frac{1}{\delta} \right) \right)
        \;\leq\;
        \delta
        \;,
    \end{align*}
    where $\mu \coloneqq - n^{-1}\sum_{1 \leq i, j \leq n} a_{ij} + \sum_{1 \leq i \leq n} a_{ii}$, and $\| A \|_{\mathrm{F}} = \big(\sum_{1 \leq i,j \leq n} a_{ij}^2\big)^{1/2}$ is the Frobenius norm.
\end{lemma}

\begin{proof}[Proof of \Cref{lem: hanson-wright inequality}]
    \citet[Remark 2.2]{adamczak2015note} shows that multinomial random vectors satisfy the convex concentration inequality \citep[Definition 2.2]{adamczak2015note}. Therefore, we can apply \citet[Theorem 2.5]{adamczak2015note} to conclude that there exists constants $C_1, C_2$ such that, for any $n \times n$ matrix $A$ and any $t > 0$, 
    \begin{align*}
        \Pr\nolimits_{\bW}\left( \left| \bW^\top A \bW - \mu \right|   \;\geq\; t \right)
        \;&\leq\;
        2 \exp\left( - \frac{1}{C_1} \min\left( \frac{t^2}{C_2 \| A \|_{\mathrm{F}}^2}, \; \frac{t}{\| A \|_{\mathrm{op}}} \right) \right)
        \\
        \;&\leq\;
        2 \exp\left( - \frac{1}{C_1} \min\left( \frac{t^2}{C_2 \| A \|_{\mathrm{F}}^2}, \; \frac{t}{\| A \|_{\mathrm{F}}} \right) \right)
        \;,
    \end{align*}    
    where $\| A\|_{\mathrm{op}} = \sup_{\|\bx\|_2 \leq 1} \| A \bx \|_2$ is the operator norm, $\mu = \E_{\bW}\big[ \sum_{1 \leq i, j \leq n} (W_i - 1)(W_j - 1) a_{ij}\big]$, and the last step holds due to the well-known ordering between operator norm and Frobenius norm $\| A\|_{\mathrm{op}} \leq \| A \|_{\mathrm{F}}$ \citep[Eq.~2.3.7]{golub2013matrix}. Moreover, we can compute $\mu$ using explicit expressions for the joint central moments of multinomial random variables \citep[Theorem 2]{ouimet2020explicit} as follows
    \begin{align*}
        \mu 
        \;&=\;
        \sum_{1 \leq i, j \leq n} a_{ij} \E_{\bW}\left[ (W_i - 1)(W_j - 1) \right]
        \\
        \;&=\;
        \sum_{1 \leq i, j \leq n} a_{ij} \E_{\bW}\left[ (\bW_i - 1) (\bW_j - 1) \right] + \sum_{1 \leq i \leq n} a_{ii} \E_{\bW}\left[ (\bW_i)^2 \right]
        \\
        \;&=\;
        - \frac{1}{n}\sum_{1 \leq i \neq j \leq n} a_{ij} + \left(1 - \frac{1}{n} \right) \sum_{1 \leq i \leq n} a_{ii}
        \\
        \;&=\;
        - \frac{1}{n}\sum_{1 \leq i, j \leq n} a_{ij} + \sum_{1 \leq i \leq n} a_{ii}
        \;.
    \end{align*}
    For any $\delta > 0$, we claim that the choice
    \begin{align}
        t \;=\; \max\left(C_1^{\frac{1}{2}} C_2^{\frac{1}{2}}, \; C_1 \right) \| A\|_{\mathrm{F}} \log \left( \frac{1}{\delta} \right)
        \label{eq: choice of t}
    \end{align}
    implies that $\Pr\nolimits_{\bW}(|\bW^\top A \bW - \mu| \geq t) \leq \delta$, which would complete the proof by setting $C = \max\big(C_1^{1/2}C_2^{1/2}, C_1\big)$. To see this, we note that the the following are equivalent 
    \begin{align*}
        \;&\qquad\quad\;
        \delta \;\geq\; 2 \exp\left( - \frac{1}{C_1} \min\left( \frac{n^4 t^2}{C_2 \| A \|_{\mathrm{F}}^2}, \; \frac{n^2 t}{\| A \|_{\mathrm{F}}} \right) \right)
        \\
        \;&\iff\;
        C_1 \log \left(\frac{2}{\delta}\right) \;\leq\; \min\left( \frac{n^4 t^2}{C_2 \| A \|_{\mathrm{F}}^2}, \; \frac{n^2 t}{\| A \|_{\mathrm{F}}} \right)
        \\
        \;&\iff\;
        C_1 \log \left(\frac{2}{\delta}\right) \;\leq\; \frac{n^4 t^2}{C_2 \| A \|_{\mathrm{F}}^2}
        \qquad\qquad\qquad\quad\mathrm{and}\quad
        C_1 \log \left(\frac{2}{\delta}\right) \;\leq\; \frac{n^2 t}{\| A \|_{\mathrm{F}}} 
        \\
        \;&\iff\;
        t \;\geq\; C_1^{1/2}C_2^{1/2} \frac{\| A \|_{\mathrm{F}}}{n^2}  \sqrt{\log\left(\frac{2}{\delta} \right)}
        \qquad\qquad\mathrm{and}\quad
        t \;\geq\; C_1 \frac{\| A \|_{\mathrm{F}}}{n^2} \log \left(\frac{2}{\delta}\right)
        \\
        \;&\iff\;
        t \;\geq\; 
        \max\left(C_1^{1/2}C_2^{1/2}, \; C_1 \sqrt{\log \left(\frac{2}{\delta} \right)} \right) \frac{\| A \|_{\mathrm{F}}}{n^2} \sqrt{\log \left(\frac{2}{\delta} \right)}
        \;.
    \end{align*}
    Since $\log(2/\delta) > 1$ for $\delta \in (0, 1)$, the last inequality is met with $t$ defined in \eqref{eq: choice of t}. This shows that $\Pr\nolimits_{\bW}\left( \left| \bW^\top A \bW - \mu \right| \;\geq\; t \right) \leq \delta$, thus finishing the proof.
\end{proof}

We are now ready to prove \Cref{lem: bound boot quantile}. 
\begin{proof}[Proof of \Cref{lem: bound boot quantile}]
    Denote by $F_\infty$ the cumulative distribution function for the conditional distribution of the bootstrap sample $D^2_{\bW}(\X_{n,\bz})$ (defined in \eqref{eq: bootstrap sample}) given $\X_{n,\bz}$, and denote its $(1-\alpha)$-quantile as
    \begin{align*}
        q^2_{\infty, 1 - \alpha}(\X_{n,\bz}) \;\coloneqq\; \inf \{ t \in \R: \; 1 - \alpha \leq F_\infty(t) \}
        \;.
    \end{align*}
    To bound $q^2_{\infty, 1 - \alpha / 2}(\X_{n,\bz})$, we use a similar argument as in \citet[Appendix E4]{schrab2023mmd} by bounding the exceedance probability of $D^2_\bW(\X_{n,\bz})$. Notably, their proof cannot be applied directly here, because they use a wild bootstrap \citep{leucht2013dependent} that corresponds to replacing $\bW_i - 1$ with independent Rademacher weights, whereas in our case $W_i$ are correlated and multinomially distributed. To proceed, we first note that we can write $D^2_\bW(\X_{n,\bz})$ as a quadratic form as $D_\bW^2(\X_{n,\bz}) = n^{-2} \bW^\top A \bW$, where $A = (a_{ij})_{ij}$ with $a_{ij} \coloneqq u_p(\bX_i, \bX_j)$. We can then use the Hanson-Wright inequality stated in \Cref{lem: hanson-wright inequality} applied with $\delta = \alpha / 2$ to show that there exist positive constant $C$ such that, for any almost sure realizations $\X_{n,\bz}$,
    \begin{align*}
        &
        \Pr\nolimits_{\bW}\left( D_{\bW}^2(\X_{n,\bz}) \;\geq\; \mu / n^2 + \frac{C \| A \|_{\mathrm{F}}}{n^2} \log \left(\frac{1}{\alpha}\right) \;\Big|\; \X_{n,\bz} \right)
        \\
        \;&=\;
        \Pr\nolimits_{\bW}\left( \bW^\top A \bW \;\geq\; \mu + C \| A \|_{\mathrm{F}} \log \left(\frac{1}{\alpha}\right) \;\big|\; \X_{n,\bz} \right)
        \\
        \;&\leq\;
        \Pr\nolimits_{\bW}\left( \big| \bW^\top A \bW - \mu \big| \;\geq\; C \| A \|_{\mathrm{F}} \log \left(\frac{1}{\alpha}\right) \;\big|\; \X_{n,\bz} \right)
        \\
        \;&\leq\;
        \alpha
        \;,
    \end{align*}
    where $\mu \coloneqq - n^{-1}\sum_{1 \leq i, j \leq n} a_{ij} + \sum_{1 \leq i \leq n} a_{ii}$. This implies that the $(1-\alpha)$-quantile of $D_{\bW}^2(\X_n)$ can be bounded as
    \begin{align*}
        q^2_{\infty, 1 - \alpha}
        \;\leq\;
        \frac{\mu}{n^2} + t
        \;&=\;
        \frac{\mu}{n^2} + \frac{\rho_\alpha \| A\|_{\mathrm{F}} }{n^2}
        \;,
        \tagaligneq \label{eq: boot quantile infty bound}
    \end{align*}
    where we have defined $\rho_\alpha \coloneqq C \log (1/\alpha)$. We now further simplifies this bound by bounding $\mu$ and $\| A \|_{\mathrm{F}}$. The matrix $A$ is the Gram matrix of the Stein kernel $u_p$, thus positive semi-definite. This implies $\sum_{1 \leq i, j \leq n} a_{ij} \geq 0$ and $a_{ii} \geq 0$ for all $i$, so we have
    \begin{align}
        \mu
        \;=\;
        - \frac{1}{n} \sum_{1 \leq i, j \leq n} a_{ij} + \sum_{1 \leq i \leq n} a_{ii}
        \;\leq\;
        \sum_{1 \leq i \leq n} a_{ii}
        \;&=\;
        \sum_{1 \leq i \leq m} a_{ii} + \sum_{m < i \leq n} a_{ii}
        \nonumber
        \\
        \;&=\;
        \sum_{1 \leq i \leq m} a_{ii} + (n - m) u_p(\bz, \bz)
        \label{eq: upper bound on mu}
        \;,
    \end{align}
    where the last quality holds since by assumption $\bX_i = \bz$ for $i > m$. Moreover, since $a_{ii} = u_p(\bX_i, \bX_i) \geq 0$ for any $i$, we can apply Markov inequality to bound the first term on the RHS. This  gives that the following holds with probability at least $1 - \delta / 4$, 
    \begin{align*}
        \sum_{1 \leq i \leq m} a_{ii}
        \;\leq\;
        \frac{4}{\delta} \sum_{1\leq i \leq m} \E[ a_{ii} ]
        \;=\;
        \frac{4m}{\delta} \E_{\bX^\ast \sim P}[ u_p(\bX^\ast, \bX^\ast) ]
        \;&\leq\;
        \frac{4m}{\delta} \big( \E_{\bX^\ast \sim P}[ u_p(\bX^\ast, \bX^\ast)^2 ] \big)^{\frac{1}{2}}
        \\
        \;&=\;
        \frac{4m}{\delta} \xi_P
        \;,
    \end{align*}
    where the second step holds since by assumption $\bX_i = \bX_i^\ast \sim P$ are i.i.d.\ for $i=1, \ldots, m$, and where we have substituted $\xi_P^2 = \E_{\bX_1 \sim P}[ u_p(\bX_1, \bX_1)^2 ]$. Combining with \eqref{eq: upper bound on mu} implies that, with probability at least $1 - \delta / 4$,
    \begin{align}
        \mu
        \;\leq\;
        \frac{4m}{\delta} \xi_P + (n - m) u_p(\bz, \bz)
        \label{eq: expectation bound}
        \;.
    \end{align} 
    We now bound $\| A \|_{\mathrm{F}}^2$. We first show that $\Sigma_P^2 \coloneqq \E_{\bX, \bX' \sim P}[ u_p(\bX, \bX')^2 ]$ is finite, where $\bX, \bX' \sim P$ are independent. Since $u_p$ is a reproducing kernel \citep[Theorem 1]{barp2022targeted}, we can use the reproducing property and the Cauchy-Schwarz inequality to yield
    \begin{align*}
        u_p(\bx, \bx') \;=\; \langle u_p(\cdot, \bx), u_p(\cdot, \bx') \rangle_{\cH_{u}} \leq \| u_p(\cdot, \bx) \|_{\cH_{u}} \| u_p(\cdot, \bx') \|_{\cH_{u}} \;=\; u_p(\bx, \bx)^{\frac{1}{2}} u_p(\bx', \bx')^{\frac{1}{2}}
        \;,
    \end{align*}
    where $\cH_u$ denotes the RKHS associated with $u_p$. This implies
    \begin{align*}
        \Sigma_P^2 
        \;=\;
        \E_{\bX, \bX' \sim P}[ u_p(\bX, \bX')^2 ]
        \;\leq\;
        \E_{\bX, \bX' \sim P}[ u_p(\bX, \bX) u_p(\bX', \bX') ]
        \;&\stackrel{(a)}{=}\;
        \big( \E_{\bX \sim P}[ u_p(\bX, \bX) ] \big)^2
        \\
        \;&\stackrel{(b)}{\leq}\;
        \E_{\bX \sim P}[ u_p(\bX, \bX)^2 ]
        \\
        \;&=\;
        \xi_P^2
        \;<\;
        \infty
        \;,
    \end{align*}
    where $(a)$ holds since $\bX, \bX'$ are i.i.d.\ and $(b)$ follows from Jensen's inequality. By applying Markov's inequality again, with probability at least $1 - \delta / 4$, where the probability is taken over the randomness of $\X_m^\ast$, we have 
    \begin{align*}
        \;&\;
        \| A \|_{\mathrm{F}}^2
        \\
        \;&=\;
        \sum_{1 \leq i, j \leq n} u_p(\bX_i, \bX_j)^2
        \\
        \;&\leq\;
        \frac{4}{\delta}\sum_{1 \leq i, j \leq n} \E[ u_p(\bX_i, \bX_j)^2 ]
        \\
        \;&=\;
        \frac{4}{\delta} \bigg( 
        \sum_{1 \leq i, j \leq m} \E[ u_p(\bX_i^\ast, \bX_j^\ast)^2 ]
        + 2 (n-m)\sum_{1 \leq i \leq m} \E[ u_p(\bX_i^\ast, \bz)^2 ]
        + (n-m)^2 \E[ u_p(\bz, \bz)^2 ]
        \bigg)
        \\
        \;&=\;
        \frac{4}{\delta} \Big( \big(m(m-1) \Sigma_P^2 + m \xi_P^2 \big) + 2 m(n-m) \E_{\bX_1^\ast \sim P}[ u_p(\bX_1^\ast, \bz)^2 ] + (n-m)^2 u_p(\bz, \bz)^2 \Big)
        \\
        \;&\leq\;
        \frac{4}{\delta} \Big( \big(m(m-1) \Sigma_P^2 + m \xi_P^2 \big) + 2 m(n-m) \sigma_\infty^2 + (n-m)^2 u_p(\bz, \bz)^2 \Big)
        \tagaligneq \label{eq: A norm bound}
        \;,
    \end{align*}
    where the second last line holds since  
    \begin{align*}
        \sum_{1 \leq i, j \leq m} \E[ u_p(\bX_i^\ast, \bX_j^\ast)^2 ]
        \;&=\;
        \sum_{1 \leq i \neq j \leq m} \E[ u_p(\bX_i^\ast, \bX_j^\ast)^2 ]
        + \sum_{1 \leq i \leq m} \E[ u_p(\bX_i^\ast, \bX_i^\ast)^2 ]
        \\
        \;&=\;
        m(m-1) \E[ u_p(\bX_i^\ast, \bX_j^\ast)^2 ]
        + m \E[ u_p(\bX_i^\ast, \bX_i^\ast)^2 ]
        \\
        \;&=\;
        m(m-1) \Sigma_P^2 + m \xi_P^2
        \;,
    \end{align*}
    and since we have substituted $\Sigma_P^2 = \E_{\bX, \bX' \sim P}[ u_p(\bX, \bX')^2 ]$, $\xi_P^2 = \E_{\bX \sim P}[u_p(\bX, \bX)^2]$ and $\sigma_\infty^2 = \sup_{\bz \in \R^d} \E_{\bX \sim P}[ u_p(\bX, \bz)^2 ] $. Using the above inequality and \eqref{eq: expectation bound} to bound \eqref{eq: boot quantile infty bound}, we conclude that the following event holds with probability at least $1 - \delta/2 - \delta/4 - \delta/4 = 1 - \delta$, where the randomness is taken over the joint distribution of $\X_m^\ast$ and $\bW$,
    \begin{align*}
        &
        q_{\infty, 1 - \alpha}^2(\X_{n,\bz})
        \\
        \;&\leq\;
        \frac{4}{n \delta} \xi_P  + \frac{n - m}{n^2} u_p(\bz, \bz)
        \\ 
        &\;\quad
        + \frac{\rho_\alpha}{n^2} \sqrt{\frac{4}{\delta} \big( m(m-1) \Sigma_P^2 + m \xi_P^2 + 2 m(n-m) \sigma_\infty^2 + (n-m)^2 u_p(\bz, \bz)^2 \big)}
        \\
        \;&\leq\;
        \frac{4}{n \delta} \xi_P + \frac{n - m}{n^2} u_p(\bz, \bz)
        \\ 
        &\;\quad
        + \frac{\rho_\alpha}{n^2 \sqrt{\delta}} \Big( 2\sqrt{m(m-1)} \Sigma_P + \sqrt{m} \xi_P + \sqrt{2 m(n-m)} \sigma_\infty + (n-m) u_p(\bz, \bz) \Big)
        \;,
        \\
        \;&\leq\;
        \frac{4}{n \delta} \xi_P + \frac{\epsilon_{m,n}}{n} u_p(\bz, \bz)
        + \frac{\rho_\alpha}{n \sqrt{\delta}} \Big( 2\Sigma_P + \xi_P + \sqrt{2 \epsilon_{m,n}} \sigma_\infty + \epsilon_{m,n} u_p(\bz, \bz) \Big)
        \\
        \;&\leq\;
        \frac{4}{n \delta} \xi_P 
        + \frac{2\rho_\alpha  \epsilon_{m,n}}{n \sqrt{\delta}} u_p(\bz, \bz)
        + \frac{\rho_\alpha}{n \sqrt{\delta}} \big( 2\Sigma_P + \xi_P + \sqrt{2} \sigma_\infty \big)
        \tagaligneq \label{eq: quantile lower bound}
        \;,
    \end{align*}
    where the first inequality holds by \eqref{eq: expectation bound} and \eqref{eq: A norm bound} (and using $m \leq n$), the second inequality follows from using twice the inequality $\sqrt{a + b} \leq \sqrt{a} + \sqrt{b}$ for any $a, b \geq 0$, the third inequality holds since $m \leq n$ and $\epsilon_{m,n} = (n-m)/n$, and \eqref{eq: quantile lower bound} holds since $\epsilon_{m,n} \leq 1$ and since the conditions $\delta \in (0, 1)$ and $\rho_\alpha \geq 1$ imply
    \begin{align*}
        \frac{\epsilon_{m,n}}{n} u_p(\bz, \bz)
        \;\leq\;
        \frac{\epsilon_{m,n}}{n \sqrt{\delta}} u_p(\bz, \bz)
        \;\leq\;
        \frac{\rho_\alpha \epsilon_{m,n}}{n \sqrt{\delta}} u_p(\bz, \bz)
        \;.
    \end{align*}
    Substituting the definition of $\rho_\alpha = C\log(1/\alpha)$, and defining the constants $C' \coloneqq 2C$ and $C_{P, 1} \coloneqq 4\xi_P$ and $C_{P, 2} \coloneqq C (2\Sigma_P + \xi_P + \sqrt{2} \sigma_\infty)$, we have from \eqref{eq: quantile lower bound} that
    \begin{align*}
        q_{\infty, 1 - \alpha}^2(\X_{n,\bz})
        \;&\leq\;
        \frac{C'\log(1/\alpha)  \epsilon_{m,n}}{n \sqrt{\delta}} u_p(\bz, \bz)
        + \frac{C_{P, 1}}{n \delta}
        + \frac{C_{P, 2} \log(1/\alpha)}{n \sqrt{\delta}}
        \;.
    \end{align*}
\end{proof}

\subsubsection{Proof of \Cref{prop: non-robust stationary}}
\label{pf: prop non-robust stationary}
\begin{proof}
We first show that the moment conditions (i)-(ii) in \Cref{lem: bound vstat} and \Cref{lem: bound boot quantile} are met, and apply these lemmas to conclude the proof. 

To verify the moment conditions, we first bound the squared Stein kernel evaluated at any $\bx, \bz \in \R^d$ as
\begin{subequations}
\begin{align*}
    &
    u_p(\bx, \bz)^2
    \\
    \;&=\;
    \big[ \bs_p(\bx)^\top \bs_p(\bz) h(\bx - \bz) - \bs_p(\bx)^\top \nabla h(\bx - \bz) + \bs_p(\bz)^\top \nabla h(\bx - \bz) + \nabla^\top \nabla h(\bx - \bz) \big]^2
    \\
    \;&\leq\;
    4\Big[
    \bigl(\bs_p(\bx)^\top \bs_p(\bz) h(\bx - \bz)\bigr)^2 + \bigl(\bs_p(\bx)^\top \nabla h(\bx - \bz) \bigr)^2 + \bigl(\bs_p(\bz)^\top \nabla h(\bx - \bz)\bigr)^2 
    \\
    &\;\;\qquad
    + \bigl( \nabla^\top \nabla h(\bx - \bz) \bigr)^2
    \Big]
    \\
    \;&\leq\;
    4\Big[
    \|\bs_p(\bx)\|_2^2 \|\bs_p(\bz)\|_2^2 h(\bx - \bz)^2
    + \|\bs_p(\bx)\|_2^2 \|\nabla h(\bx - \bz)\|_2^2 
    + \|\bs_p(\bz)\|_2^2 \|\nabla h(\bx - \bz)\|_2^2 
    \tagaligneq \\
    &\;\;\qquad
    + \big(\nabla^\top \nabla h(\bx - \bz) \big)^2
    \Big]
    \tagaligneq 
    \;,
\end{align*}
\label{eq: bound on stein kernel}%
\end{subequations}
where the first inequality holds since $(a + b + c + d)^2 \leq 4(a^2 + b^2 + c^2 + d^2)$ for any $a, b, c, d \in \R$, and the second inequality is due to Cauchy-Schwarz inequality. The assumption $h \in \cC_b^2$ immediately implies that the last term in \eqref{eq: bound on stein kernel} is uniformly bounded over $\bx, \bz$. For the first term, taking expectation over $P$ and supremum over $\bz$ gives
\begin{align*}
    &
    \sup_{\bz \in \R^d} \E_{\bX \sim P} \big[\|\bs_p(\bX)\|_2^2 \|\bs_p(\bz)\|_2^2 h(\bX - \bz)^2 \big]
    \\
    \;&\leq\;
    \sup_{\bz \in \R^d} \|\bs_p(\bz)\|_2^2 \big(\E_{\bX \sim P} \big[ h(\bX - \bz)^4 \big]\big)^{1/2} \big(\E_{\bX \sim P} \big[ \| \bs_p(\bX) \|_2^4 \big] \big)^{1/2}
    \\
    \;&=\;
    C_P^{1/2} \sup_{\bz \in \R^d} \|\bs_p(\bz)\|_2^2 \big(\E_{\bX \sim P} \big[ h(\bX - \bz)^4 \big]\big)^{1/2}
    \;,
\end{align*}
where the first step holds by H\"{o}lder's inequality, and $C_P \coloneqq \E_{\bX \sim P} \big[ \| \bs_p(\bX) \|_2^4 \big]$. The RHS of the last line is finite by assumption. Similarly, the second term in \eqref{eq: bound on stein kernel} can also be bounded by H\"{o}lder's inequality as
\begin{align*}
    \sup_{\bz \in \R^d} \E_{\bX \sim P} \big[ \|\bs_p(\bx)\|_2^2 \|\nabla h(\bx - \bz)\|_2^2  \big]
    \;\leq\;
    C_P^{1/2} \sup_{\bz \in \R^d} \E_{\bX \sim P} \big( \big[ \|\nabla h(\bx - \bz)\|_2^4  \big] \big)^{1/2}
    \;,
\end{align*}
which is finite since $h \in \cC_b^2$. For the third term in \eqref{eq: bound on stein kernel},
\begin{align*}
    \sup_{\bz \in \R^d} \E_{\bX \sim P} \big[ \|\bs_p(\bz)\|_2^2 \|\nabla h(\bx - \bz)\|_2^2 \big]
    \;=\;
    \sup_{\bz \in \R^d} \|\bs_p(\bz)\|_2^2 \E_{\bX \sim P} \big[ \|\nabla h(\bx - \bz)\|_2^2 \big]
    \;,
\end{align*}
which is again finite by assumption. Substituting these bounds into \eqref{eq: bound on stein kernel}, we have shown that $\sigma_\infty^2 = \sup_{\bz \in \mathbb{R}^d} \E_{\bX \sim P} [u_p(\bX, \bz)^2 ] < \infty$, which verifies condition (i). 

To show (ii), we note that a translation-invariant kernel corresponds to choosing a constant weighting function $w(\bx) \equiv 1$ in the tilted kernel of \Cref{lem: bounded Stein kernel}, so that we can leverage the derivations in the proof of \Cref{lem: bounded Stein kernel}. Setting $w(\bx) \equiv 1$ in \eqref{eq: stein kernel diag simplificaiton} to simplify $u_p(\bx, \bx)$ yields
\begin{align*}
    \xi_P^2
    \;\coloneqq\;
    \E_{\bX \sim P}\big[ u_p(\bX, \bX)^2 \big]
    \;&=\;
    \E_{\bX \sim P}\Big[ \big( \| \bs_p(\bX) \|_2^2 h(0) - \nabla^\top \nabla h(0) \big)^2 \Big]
    \\
    \;&\leq\;
    2h(0)^2 \E_{\bX \sim P}\big[ \| \bs_p(\bX) \|_2^4 \big] + 2\big| \nabla^\top \nabla h(0) \big|^2
    \;<\; 
    \infty
    \;,
\end{align*}
where the first inequality holds since $(a+b)^2 \leq 2a^2 + 2b^2$ for any $a, b \in \R$. This proves condition (ii).

Having showed the moment conditions, we are now ready to prove the main theorem. Pick any $s \in [0, 1)$ and any integer $m' \in \big[\frac{1}{2}n^{1-s}, n\big)$. Define $\epsilon_{m,n} \coloneqq m' / n$ and $m = n - m'$. Let $C, C_{P, 1}, C_{P, 2}$ be the constants defined in \Cref{lem: bound boot quantile}, and define the events
\begin{subequations}
    \begin{align}
        \cA_1 &\coloneqq \big\{ D^2(\X_n) \;\geq\; \epsilon_{m,n}^2 u_p(\bz, \bz) - t_1(m') \big\}
        \;, \\
        \cA_2 &\coloneqq \bigg\{ q^2_{B, 1-\alpha}(\X_n) \leq \frac{C\log(1 / \alpha) }{n \sqrt{\delta / 2}}  \epsilon_{m, n}  u_p(\bz, \bz) + t_2(m') \bigg\}
        \;,
    \end{align}
    \label{eq: events A1 and A2}
\end{subequations}
where 
\begin{align*}
    t_1(m') \;\coloneqq\; \frac{2\sigma_\infty\epsilon_{m,n}}{\sqrt{\delta m/2}} 
    \;,\qquad\qquad
    t_2(m') \;&\coloneqq\; \frac{2C_{P, 1}}{n \delta} + \frac{C_{P, 2} \log(1/\alpha)}{n \sqrt{\delta / 2}} 
    \;,
\end{align*}
and the dependence on $m'$ in these quantities is through $m = n-m'$. It follows from \Cref{lem: bound vstat} and \Cref{lem: bound boot quantile} applied to $\delta/2$ that the event $\cA \coloneqq \cA_1 \cap \cA_2$ occurs with probability $\Pr(\cA) \geq 1 - \delta/2  - \delta/2 = 1 - \delta$. These events allow us to bound $D^2(\X_{n,\bz})$ and $q_{B,1-\alpha}^2(\X_{n,\bz})$, which then allows us to show that $D^2(\X_{n,\bz}) > q_{B,1-\alpha}^2(\X_{n,\bz})$ for sufficiently large $\bz$. On $\cA$, we have 
\begin{align*}
    D^2(\X_n) - q^2_{B, 1-\alpha}(\X_{n,\bz})
    \;&\geq\;
    \epsilon_{m, n}^2 u_p(\bz, \bz) - t_1(m') 
    - \frac{C\log(1 / \alpha)}{n \sqrt{\delta / 2}} \epsilon_{m, n}  u_p(\bz, \bz) - t_2(m') 
    \\
    \;&=\;
    \epsilon_{m, n} u_p(\bz, \bz) \bigg( \epsilon_{m, n} - \frac{C'}{n\sqrt{\delta}} \bigg) - T(m')
    \;,
\end{align*}
where in the last line we have defined $C' \coloneqq \sqrt{2}C \log(1 / \alpha)$ and $T(m') \coloneqq t_1(m') + t_2(m')$. This implies 
\begin{align*}
    &
    \Pr\nolimits_{\X_m^\ast \sim P, \bW}\big( D^2(\X_{n,\bz}) > \hat{q}_{1-\alpha}^B(\X_{n,\bz}) \big)
    \\
    \;&\geq\;
    \Pr\nolimits_{\X_m^\ast \sim P, \bW}\big( D^2(\X_{n,\bz}) > \hat{q}_{1-\alpha}^B(\X_{n,\bz}) \;|\; \cA \big) \Pr\nolimits_{\X_m^\ast \sim P, \bW}(\cA)
    \\
    \;&\geq\;
    (1 - \delta) \Pr\nolimits_{\X_m^\ast \sim P, \bW}\big( D^2(\X_{n,\bz}) > \hat{q}_{1-\alpha}^B(\X_{n,\bz}) \;|\; \cA \big)
    \\
    \;&\geq\;
    (1 - \delta) \indicator \bigg\{ \epsilon_{m, n} u_p(\bz, \bz) \bigg( \epsilon_{m, n} - \frac{C'}{n\sqrt{\delta}} \bigg) - T(m') \;>\; 0 \bigg\}
    \tagaligneq \label{eq: lower bound on exceedance probability}
    \;,
\end{align*}
where the first step holds due to the law of total probability. 
We claim that the indicator function in \eqref{eq: lower bound on exceedance probability} takes value 1 for all $m' \in \big[\frac{1}{2}n^{1-s}, n\big)$ when the following inequalities hold
\begin{align}
    n^{\frac{(1-s)}{2}} \;>\; 4C'
    \;,\qquad\qquad
    \|\bs_p(\bz)\|_2^2 \;>\; \frac{1}{h(0)}\Big( 4T(n-1) n^{2s} + \nabla_1^\top \nabla_2 h(0) \Big)
    \;.
    \label{eq: condition on delta and score}
\end{align}
Indeed, since $\delta > 0$ was arbitrary, we can set $\delta = n^{-(1-s)}$, so we have
\begin{align}
    \epsilon_{m, n} - \frac{C'}{n\sqrt{\delta}}
    \;=\;
    \epsilon_{m, n} - \frac{C'}{n^{\frac{1}{2}+\frac{s}{2}}}
    \;\stackrel{(a)}{>}\;
    \epsilon_{m, n} - \frac{C'}{4C' n^s}
    \;=\;
    \epsilon_{m, n} - \frac{1}{4n^s}
    \;\stackrel{(b)}{\geq}\;
    \epsilon_{m, n} - \frac{\epsilon_{m, n}}{2}
    \;=\;
    \frac{\epsilon_{m, n}}{2}
    \label{eq: non robustness positive diff}
    \;,
\end{align}
where $(a)$ holds since the first condition in \eqref{eq: condition on delta and score} implies $n^{\frac{1}{2}+\frac{s}{2}} = n^{s + \frac{(1 - s)}{2}} \geq 4C'n^s$, and $(b)$ holds since the assumption $m' \geq \frac{1}{2} n^{1-s}$ implies $n^s \leq n/(2m') = \epsilon_{m,n} /2$. Using this and the second inequality in \eqref{eq: condition on delta and score},
\begin{align*}
    \epsilon_{m,n} u_p(\bz, \bz) \bigg( \epsilon_{m, n} - \frac{C'}{n\sqrt{\delta}} \bigg) 
    \;&=\;
    \epsilon_{m,n} \big( \| \bs_p(\bz) \|_2^2 h(0) - \nabla_1^\top \nabla_2 h(0) \big) \bigg( \epsilon_{m,n} - \frac{C'}{n\sqrt{\delta}} \bigg)
    \\
    \;&>\;
    \epsilon_{m,n} \times 4T(n-1) n^{2s} \times \frac{\epsilon_{m,n}}{2}
    \\
    \;&\geq\;
    \epsilon_{m,n} \times \frac{2T(n-1)}{\epsilon_{m,n}^2} \times \frac{\epsilon_{m,n}}{2}
    \\
    \;&=\;
    T(n-1)
    \\
    \;&\geq\; 
    T(m')
    \;,
\end{align*}
where the first equality holds by setting $w(\bx) \equiv 1$ in \eqref{eq: stein kernel diag simplificaiton} to compute $u_p(\bz, \bz)$, the second line follows from the second inequality in \eqref{eq: condition on delta and score} and \eqref{eq: non robustness positive diff}, the third line holds since $n^{2s} \geq n^{2s}/(m')^2 = \epsilon_{m,n}^{-2}$, and the last line holds as direct computation shows $T(n-1) \geq T(m')$ since $m' < n$ by assumption. 
To summarize, we have shown that for any $n$ and $\bz$ satisfying \eqref{eq: condition on delta and score}, the following holds 
\begin{align*}
    \Pr\nolimits_{\X_m^\ast \sim P, \bW}\big( D^2(\X_{n,\bz}) > q^2_{\infty, 1-\alpha}(\X_{n,\bz}) \big)
    \;\geq\;
    1 - \delta
    \;=\;
    1 - n^{-(1-s)}
    \;.
\end{align*}
It then remains to show that there exists $\bz$ not dependent on $m'$ that satisfies the second inequality in \eqref{eq: condition on delta and score}, which would imply the claimed result. Such $\bz$ always exists, because, for the second inequality in \eqref{eq: condition on delta and score}, the RHS does not depend on $\bz$ nor $m'$ and is finite since $\epsilon_{m,n} = m' / n \in (n^{-s}/2, 1)$ by assumption, while the LHS explodes with $\bz$ as we have assumed $\bz \mapsto \| \bs_p(\bz) \|_2$ is unbounded. This concludes the proof.
\end{proof}

\subsubsection{Proof of \Cref{thm: non robust stationary}}
\label{pf: thm non-robust stationary}

\begin{proof}
It suffices to show the claim for any sequence of the form $\epsilon_n = n^{-s}$, where $s \in [0, 1)$ is arbitrary. Let the random variable $M'$ be defined as at the start of 
\Cref{app: non-robust stationary}. For any random sample $\X_n$, define the event $\cA(\X_n) = \{ D^2(\X_n) > q^2_{\infty, 1-\alpha}(\X_n) \}$. For $n$ sufficiently large so that $n^{(1-s)/2} > 2\sqrt{2} C\log(1/\alpha)$, we can apply \Cref{prop: non-robust stationary} to conclude that, for all $n$, there exists $\bz_n \in \R^d$ such that the following holds for all $m' \in \big[\frac{1}{2}n^{1-s}, n\big) = \big[\frac{1}{2}\epsilon_n n, n\big)$,
\begin{align*}
    \Pr\nolimits_{\X_n \sim Q, \bW}\big( \cA(\X_n) \;|\; M' = m' \big)
    \;&=\;
    \Pr\nolimits_{\X_{n-m'}^\ast \sim P, \bW}\big( D^2(\X_{n,\bz}) > q^2_{\infty, 1-\alpha}(\X_{n,\bz}) \big)
    \\
    \;&\geq\; 1 - n^{-(1-s)} 
    \tagaligneq \label{eq: lb for all m_prime}
    \;,
\end{align*}
where we have used the shorthand notation $\X_{n,\bz_n} = \X^\ast_{n-m'} \cup \{ \bz_n \}^{m'}$. We will show that the sequence of probability measures $Q \coloneqq (1-\epsilon_n)P + \epsilon_n \delta_{\bz_n}$ satisfies 
\begin{align}
    \Pr\nolimits_{\X_n \sim Q, \bW}\big( D^2(\X_{n}) > q^2_{\infty, 1-\alpha}(\X_{n}) \big)
     - \Pr\nolimits_{\X_n^\ast \sim P, \bW}\big( D^2(\X_{n}^\ast) > q^2_{\infty, 1-\alpha}(\X_{n}^\ast) \big)
     \;\to\;
     1 - \alpha
     \label{eq: error diff limit}
     \;,
\end{align}
which will imply the claimed result of our theorem. 

Using a concentration inequality for Binomial distributions \citep[e.g.,][Lemma 2.1]{chung2002connected} applied to the case of $\text{Binomial}(n,\epsilon_n)$, the event $\cB \coloneqq \{ \epsilon_n n/2 \leq M' \leq 3\epsilon_n n/2 \}$ holds with high probability
\begin{align*}
    \Pr\nolimits_{M'}\left( \cB \right)
    \;&=\;
    1 - \Pr\nolimits_{M'}\left(\big| M' - \epsilon_n n \big| < \frac{\epsilon_n n }{2} \right)
    \\
    \;&\geq\;
    1 - \exp\left( - \frac{(\epsilon_n n)^2}{8 \epsilon_n n}\right)
    - \exp\left( - \frac{(\epsilon_n n)^2}{2 (\epsilon_n n + \epsilon_n n /6)}\right)
    \\
    \;&=\;
    1 - \exp\left( - \frac{\epsilon_n n}{8}\right)
    - \exp\left( - \frac{3\epsilon_n n}{7}\right)
    \\
    \;&\eqqcolon\;
    1 - f(n)
    \tagaligneq \label{eq: binom concentration 1}
    \;,
\end{align*}
where in the last line we have defined $f(n) = \exp\left( - \epsilon_n n / 8\right) - \exp\left( - 3\epsilon_n n / 7\right)$. Define the index set $I_n$ to be the set of integers in the interval $[\epsilon_n n /2, 3\epsilon_n n /3]$. Then
\begin{align*}
    &
    \Pr\nolimits_{\X_n \sim Q, \bW}\big( D^2(\X_{n}) > q^2_{\infty, 1-\alpha}(\X_{n}) \big)
    \\
    \;&\geq\;
    \Pr\nolimits_{\X_n \sim Q, \bW}\big( \cA(\X_n) \cap \cB \big) 
    \\
    \;&\stackrel{(a)}{=}\;
    \sum_{m' \in I_n} \Pr\nolimits_{\X_n \sim Q, \bW}\big( \cA(\X_n) \;|\; M' = m' \big) \Pr(M' = m')
    \\
    \;&\geq\;
    \Big(\min_{m' \in I_n} \Pr\nolimits_{\X_n \sim Q, \bW}\big( \cA(\X_n) \;|\; M' = m' \big) \Big)
    \sum_{m' = m'_0}^n \Pr(M' = m')
    \\
    \;&\stackrel{(b)}{=}\;
    \min_{m' \in I_n} \Pr\nolimits_{\X_n \sim Q, \bW}\big( \cA(\X_n) \;|\; M' = m' \big) \Pr(\cB)
    \\
    \;&\stackrel{(c)}{\geq}\;
    \big( 1 - n^{-(1-s)} \big) 
    \big(1 - f(n) \big)
    \;\stackrel{(d)}{=}\;
    1 - \cO\big(n^{-(1-s)}\big)
    \tagaligneq \label{eq: rej prob lower bound}
    \;,
\end{align*}
where the first inequality follows from the law of total probability, $(a)$ and $(b)$ hold by the partition $\cB = \cup_{m' \in I_n}\{ M' = m' \}$ and again the law of total probability, $(c)$ follows from \eqref{eq: lb for all m_prime} and \eqref{eq: binom concentration 1}, and $(d)$ holds since the assumptions $\epsilon_n = n^{-s}$ and $s \in [0, 1)$ imply $f(n) = o\big(n^{-(1-s)}\big)$. On the other hand, the assumed moment conditions imply $\E_{\bX^\ast \sim P}[u_p(\bX^\ast, \cdot)] = 0$ as argued in the paragraph before \eqref{eq: var bound} and $\E_{\bX^\ast_1, \bX^\ast_2 \sim P}[u_p(\bX^\ast_1, \bX^\ast_2)^2] \leq \E_{\bX^\ast \sim P}[u_p(\bX^\ast, \bX^\ast)^2] < \infty$ as argued in \Cref{pf: prop non-robust stationary}, so we can apply \citet[Theorem 3.5]{arcones1992bootstrap} with $r=2$ to yield the asymptotic validity of the KSD test when $Q = P$, namely 
\begin{align*}
    \Pr\nolimits_{\X_n^\ast \sim P, \bW}\big( D^2(\X_{n}^\ast) > q^2_{\infty, 1-\alpha}(\X_{n}^\ast) \big)
    \;\to\; \alpha \;.
\end{align*}
Combining this with \eqref{eq: rej prob lower bound} shows the desired convergence \eqref{eq: error diff limit}.
\end{proof}

\subsubsection{Implications for KSDs on Non-Euclidean Spaces}
The core argument underpinning the proofs in \Cref{thm: non robust stationary} is that the function $\bz \mapsto u_p(\bz,\bz)$ is unbounded on the support of the model, which in our case is $\R^d$. The proof per se does not rely on the support being an Euclidean space. Consequently, this argument can potentially be extended to KSDs constructed for non-Euclidean data, such as Riemannian data \citep{xu2021interpretable}, sequence data with varying dimensionality \citep{baum2023kernel}, functional data \citep{wynne2022spectral}, or censored time-to-event data \citep{fernandez2020kernelized}.

\subsection{Proof of \Cref{lem: bounded Stein kernel}}
\label{app: pf of bounded Stein kernel}

\begin{proof}
Direct expansion of the Stein kernel in \eqref{eq:Stein_kernel} with a tilted kernel $k$ gives
\begin{align*}
    u_p(\bx, \bx')
    \;&=\;
    \langle w(\bx)\bs_p(\bx), \; w(\bx') \bs_p(\bx') \rangle h(\bx, \bx')
    \\
    &\;\quad
    + \langle  w(\bx') \bs_p(\bx'), \; \nabla w(\bx) \rangle h(\bx - \bx') + \langle w(\bx') \bs_p(\bx'), \nabla h(\bx - \bx') \rangle w(\bx) 
    \\
    &\;\quad
    + \langle w(\bx) \bs_p(\bx), \; \nabla w(\bx') \rangle h(\bx - \bx') - \langle w(\bx) \bs_p(\bx), \nabla h(\bx - \bx') \rangle w(\bx')
    \\
    &\;\quad
    + \langle \nabla w(\bx),\; \nabla w(\bx') \rangle h(\bx - \bx') + \langle w(\bx) \nabla h(\bx - \bx'),\; \nabla w(\bx') \rangle
    \\
    &\;\quad
    - \langle \nabla w(\bx),\; w(\bx') \nabla h(\bx - \bx') \rangle - w(\bx) w(\bx') \nabla^\top \nabla h(\bx - \bx')
    \\
    \;&=\;
    \langle \bs_{p, w}(\bx),\; \bs_{p, w}(\bx') \rangle h(\bx - \bx') 
    \\
    &\;\quad
    + \langle \bs_{p, w}(\bx'), \; \nabla w(\bx) \rangle h(\bx - \bx') + \langle \bs_{p, w}(\bx'),\; \nabla h(\bx - \bx') \rangle w(\bx)
    \\
    &\;\quad
    + \langle \bs_{p, w}(\bx), \; \nabla w(\bx') \rangle h(\bx - \bx') 
    - \langle \bs_{p, w}(\bx),\; \nabla h(\bx - \bx') \rangle w(\bx')
    \\
    &\;\quad
    + \langle \nabla w(\bx),\; \nabla w(\bx') \rangle h(\bx - \bx') + \langle \nabla h(\bx - \bx'),\; \nabla w(\bx') \rangle w(\bx) 
    \\
    &\;\quad
    - \langle \nabla w(\bx),\; \nabla h(\bx - \bx') \rangle w(\bx') - w(\bx) w(\bx') \nabla^\top \nabla h(\bx - \bx')
    \;,
\end{align*}
where $\bs_{p, w}(\cdot) \coloneqq w(\cdot) \bs_p(\cdot) $ and $\nabla h(\bx-\bx')$ denotes $\nabla h$ evaluated at $\bx-\bx'$. By setting $\bx = \bx'$ and eliminating identical terms,
\begin{align}
    &
    u_p(\bx, \bx)
    \nonumber \\
    \;&=\;
    \| \bs_{p, w}(\bx) \|_2^2 h(0)  + 2 \langle \bs_{p, w}(\bx), \; \nabla w(\bx) \rangle h(0)
    + \| \nabla w(\bx) \|_2^2 h(0) - w(\bx)^2 \nabla^\top \nabla h(0)
    \label{eq: stein kernel diag simplificaiton}
    \\
    \;&\leq\;
    \| \bs_{p, w}(\bx) \|_2^2 h(0) 
    + 2 \| \bs_{p, w}(\bx) \|_2 \| \nabla w(\bx) \|_2 h(0) 
    + \| \nabla w(\bx) \|_2^2 h(0) + w(\bx)^2 | \nabla^\top \nabla h(0) |
    \label{eq: UB on stein kernel diag}
    \;,
\end{align}
where the last line follows from Cauchy-Schwarz inequality. The RHS of \eqref{eq: UB on stein kernel diag} is bounded over $\bx$ assuming the stated conditions on $h$, $w$ and the supremum of $\|s_p(x) w(x)\|_2 = \|s_{p,w}(x)\|_2$. To prove the bound on $u_p(\bx, \bx')$, we denote by $\cH_u$ the RKHS associated with $u_p$, which is a reproducing kernel \citep[see, for example,][Theorem 1]{barp2022targeted}. We use the reproducing property of the Stein reproducing kernel $u_p$ and the Cauchy-Schwarz inequality for the corresponding RKHS norm $\| \cdot \|_{\cH_u}$ to yield
\begin{align}
    u_p(\bx, \bx') \;=\; \langle u_p(\cdot, \bx), u_p(\cdot, \bx') \rangle_{\cH_{u}} \leq \| u_p(\cdot, \bx) \|_{\cH_{u}} \| u_p(\cdot, \bx') \|_{\cH_{u}} \;=\; u_p(\bx, \bx)^{\frac{1}{2}} u_p(\bx', \bx')^{\frac{1}{2}}
    \label{eq: stein kernel diag bound}
    \;.
\end{align}
Hence, $\sup_{\bx, \bx' \in \mathbb{R}^d} u_p(\bx, \bx') \leq \infty$. In particular, $u_p(\bx, \bx)^{1/2}$ is well-defined, since $u_p(\bx, \bx) \geq 0$, which is because the Stein kernel $u_p$ is positive definite. 
\end{proof}

\subsection{Proof of \Cref{thm: robust tilted} and Related Preliminary Results}
\label{app: robustness tilted}

\paranoheading{Overview of proof}
We will show a general result in \Cref{prop: qualitative robust tilted general}, which states that the robust-KSD test that rejects $\Hc_0: \cB^\KSD(P; \theta)$ if $\Delta_\theta(\X_n) \coloneqq \max(0, D(\X_n) - \theta) > q_{B,1-\alpha}(\X_n)$ is qualitatively robust for any $\theta \geq 0$. This immediately shows \Cref{thm: robust tilted} by setting $\theta = 0$. To show \Cref{prop: qualitative robust tilted general}, we follow a similar approach in the proof of \Cref{thm: non robust stationary}, where we will first show that the result holds conditionally on the number of contaminated data, and use a high-probability argument to complete the proof. 

The rest of this section is organized as follows:
\begin{itemize}
    \item \Cref{pf: diff of quantiles tilted kernel} shows \Cref{lem: diff of quantiles tilted kernel}, which states that the bootstrap quantiles $q_{B,1-\alpha}(\X_n)$ and $q_{B,1-\alpha}(\X_n^\ast)$ computed using any two samples $\X_n$ and $\X_n^\ast$ that differ by at most $m'$ elements are close to each other with high probability.
    \item \Cref{pf: bound on test stat tilted kernel} shows \Cref{lem: bound on test stat tilted kernel}, which states that the difference in the exceedance probabilities of $\Delta_\theta(\X_n)$ and of $\Delta_\theta(\X_n^\ast)$ is small.
    \item \Cref{pf: robustness tilted conditional} shows \Cref{lem: robustness tilted conditional}, which states that the rejection probability of a robust-KSD test using $\X_n$ is close to that of a test using $\X_n^\ast$, assuming $\X_n$ and $\X_n^\ast$ differ by no more than $o(n^{1/2})$ elements. Its proof relies on \Cref{lem: diff of quantiles tilted kernel} and \Cref{lem: bound on test stat tilted kernel}.
    \item \Cref{pf: qualitative robust tilted general} uses \Cref{lem: robustness tilted conditional} to show \Cref{prop: qualitative robust tilted general}, which states that the robust-KSD test is qualitatively robust for any $\theta \geq 0$. This result immediately implies \Cref{thm: robust tilted}, the proof of which is also contained in this subsection.
\end{itemize}
\begin{lemma}
    \label{lem: diff of quantiles tilted kernel}
    Assume $k$ is a tilted kernel satisfying the conditions in \Cref{lem: bounded Stein kernel}. Then there exists absolute constants $C_1, C_2 > 0$ such that, for any $\delta > 0$ and any (deterministic) realizations $\X_n =\{\bx_i\}_{i=1}^n$ and $\X_n^\ast = \{\bx_i^\ast\}_{i=1}^n$ that differ by at most $m'$ elements, we have
    \begin{align*}
        \Pr\nolimits_{\bW} \left( \left| D^2_{\bW}(\X_n) - D^2_{\bW}(\X_n^\ast) \right| \;\leq\; \frac{\tau_\infty \epsilon_{m,n}^{\frac{1}{2}}}{n} \left( C_1 + C_2 \log\left(\frac{8}{\delta}\right) \right) \right) \;\geq\; 1 - \delta
        \;,
    \end{align*}
    where $D^2_{\bW}(\X_n)$, defined in \eqref{eq: bootstrap sample}, is the bootstrap sample computed using $\X_n$, and $\epsilon_{m,n} = m' / n$. Moreover, the above inequality implies
    \begin{align*}
        \Pr\nolimits_{\bW} \left( \left| q_{\infty, 1-\alpha}(\X_n) - q_{\infty, 1-\alpha}(\X_n^\ast) \right| \;\leq\; \frac{\tau_\infty^{\frac{1}{2}} \epsilon_{m,n}^{\frac{1}{4}}}{n^{\frac{1}{2}}} \sqrt{C_1 + C_2 \log\left(\frac{8}{\delta}\right) } \;\right) \geq\; 1 - \delta
        \;.
    \end{align*}
\end{lemma}

\begin{lemma}
    \label{lem: bound on test stat tilted kernel}
    Assume $\E_{\bX \sim P}[\| \bs_p(\bX) \|_2] < \infty$ and that $k$ is a tilted kernel satisfying the conditions in \Cref{lem: bounded Stein kernel}. For any $\delta, \gamma > 0$ and any integer $m' \in [0, n]$, it holds that
    \begin{align*}
        &
        \sup_{\bz_1, \ldots, \bz_{m'}\in\R^d} \left| \Pr\nolimits_{\X_m^\ast \sim P}\left( D\left(\X_{n-m'}^\ast \cup \{\bz_i\}_{i=1}^{m'}\right) \;>\; \gamma \right) - \Pr\nolimits_{\X_n^\ast \sim P}\left( D(\X_n^\ast) \;>\; \gamma \right)
        \right|
        \\
        \; &\leq\; 
        \Pr\nolimits_{\X_n^\ast \sim P}\left(  \left| D(\X_n^\ast) - \gamma \right| \;\leq\; t_{m,n,\delta} \right) + \delta
        \;,
    \end{align*}
    where $t_{m, n, \delta} \coloneqq \big(4\sqrt{2} \tau_\infty \epsilon_{m, n} / \sqrt{\delta n} + 2 \tau_\infty \epsilon_{m, n}^2 \big)^{1/2}$.
\end{lemma}
\begin{lemma}
    \label{lem: robustness tilted conditional}
    Assume $\E_{\bX \sim P}[\| \bs_p(\bX) \|_2] < \infty$ and that $k$ is a tilted kernel satisfying the conditions in \Cref{lem: bounded Stein kernel}. For any integers $\theta \geq 0$, $\delta > 0$, and any sequence $f_n = o(n^{1/2})$, there exists $n_0$ such that for any $n \geq n_0$, we have
    \begin{align*}
        \max_{m' \leq f_n} \omega(m') \;<\; \delta \;,
    \end{align*}
    where the $\max$ is over all non-negative integers $m' \leq f_n$, and
    \begin{align*}
        \omega(m')
        \;&\coloneqq\; 
        \big| \Pr\nolimits_{\X_m^\ast \sim P, \bW}\big( \Delta_\theta(\X_m^\ast \cup \Z_{m'}) > q_{\infty, 1-\alpha}(\X_m^\ast \cup \Z_{m'}) \big) 
        \\
        &\;\quad\quad
        - \Pr\nolimits_{\X_n^\ast \sim P, \bW}\big( \Delta_\theta(\X_n^\ast) > q_{\infty, 1-\alpha}(\X_n^\ast) \big) \big|
        \;.
    \end{align*}
\end{lemma}
\begin{proposition}
\label{prop: qualitative robust tilted general}
    Assume $\E_{\bX \sim P}[\| \bs_p(\bX) \|_2] < \infty$ and that $k$ is a tilted kernel satisfying the conditions in \Cref{lem: bounded Stein kernel}. If $\theta \geq 0$ and $\epsilon_n = o(n^{-1/2})$, then as $n \to \infty$,
    \begin{align*}
        \sup_{Q \in \cP(P; \epsilon_n)} \big| \Pr\nolimits_{\X_n \sim Q, \bW}\big( \Delta_\theta(\X_n) > q_{\infty, 1-\alpha}(\X_n) \big) - \Pr\nolimits_{\X_n^\ast \sim P, \bW}\big( \Delta_\theta(\X_n^\ast) > q_{\infty, 1-\alpha}(\X_n^\ast) \big) \big|
        \;\to\;
        0
        \;,
    \end{align*}
    where $\cP(P; \epsilon_n)$ is the Huber's contamination model defined in \eqref{eq: Huber model}.
\end{proposition}

\subsubsection{Proof of \Cref{lem: diff of quantiles tilted kernel}}
\label{pf: diff of quantiles tilted kernel}

\begin{proof}
Without loss of generality, we assume $\X_n$ and $\X_n^\ast$ are ordered so that $\bx_i = \bx_i^\ast$ for $i = 1, \ldots, m$. We also define the $n \times n$ matrices $U \coloneqq (u_{ij})_{1 \leq i, j \leq n}$ and $U' \coloneqq (u_{ij}')_{1 \leq i, j \leq n}$, where $u_{ij} \coloneqq u_p(\bx_i, \bx_j)$ and $u_{ij}' \coloneqq u_p(\bx_i^\ast, \bx_j^\ast)$. Since $u_p$ is a reproducing kernel, $U$ and $U'$ are Gram matrices, thus symmetric and positive semi-definite. Furthermore, we define $W_i^0 \coloneqq W_i - 1$ for any $i = 1, \ldots, n$, and define the vectors $\bV_1 \coloneqq (W_1^0, \ldots, W_m^0, 0, \ldots, 0)^\top$ and $\bV_2 \coloneqq (0, \ldots, 0, W_{m+1}^0, \ldots, W_n^0)^\top$, so that $\bW^0 = \bW - 1 = \bV_1 + \bV_2$. It follows that the bootstrap sample $D^2_\bW(\X_n)$ defined in \eqref{eq: bootstrap sample} can be decomposed as
\begin{align*}
    D^2_{\bW}(\X_n)
    \;=\;
    \frac{1}{n^2} \sum_{1 \leq i, j \leq n} W_i^0 W_j^0 u_{ij}
    \;&=\;
    \frac{1}{n^2} (\bW^0)^\top U \bW^0
    \\
    \;&=\;
    \frac{1}{n^2} \left( \bV_1^\top U \bV_1 + \bV_2^\top U \bV_2 + 2\bV_1^\top U \bV_2 \right)
    \;.
\end{align*}
Similarly, we can decompose $D^2_{\bW}(\X_n^\ast)$ as
\begin{align*}
    D^2_{\bW}(\X_n^\ast)
    \;=\;
    \frac{1}{n^2} \Big( \bV_1^\top U' \bV_1 + \bV_2^\top U' \bV_2 + 2\bV_1^\top U' \bV_2 \Big)
    \;.    
\end{align*}
Since $\bx_i = \bx_i^\ast$ for all $i \leq m$ by construction, we have $u_{ij} = u_{ij}'$ for $1 \leq i, j \leq m$, so $\bV_1^\top U \bV_1 = \sum_{1 \leq i, j \leq m} W_i^0 W_j^0 u_{ij} = \sum_{1 \leq i, j \leq m} W_i^0 W_j^0 u_{ij}' = \bV_1^\top U' \bV_1$, and we can bound the following difference as
\begin{align*}  
    &
    n^2\big| D^2_{\bW}(\X_n) - D^2_{\bW}(\X_n^\ast) \big|
    \\
    \;&=\;
    \big| \bV_2^\top U \bV_2 + 2\bV_1^\top U \bV_2 - \bV_2^\top U' \bV_2 - 2\bV_1^\top U' \bV_2 \big|
    \\
    \;&\leq\;
    | \bV_2^\top U \bV_2 | + 2 | \bV_1^\top U \bV_2 | + | \bV_2^\top U' \bV_2 | + 2 | \bV_1^\top U' \bV_2 |
    \\
    \;&\leq\;
    | \bV_2^\top U \bV_2 | + 2| \bV_1^\top U \bV_1 |^{\frac{1}{2}} | \bV_2^\top U \bV_2 |^{\frac{1}{2}}
    + |\bV_2^\top U' \bV_2 | + 2| \bV_1^\top U' \bV_1 |^{\frac{1}{2}} | \bV_2^\top U' \bV_2 |^{\frac{1}{2}}
    \tagaligneq \label{eq: boot bound diff}    
    \;,
\end{align*}
where the last line follows from Cauchy-Schwarz inequality applied to the $U$-weighted inner product $(\bx, \bx') \mapsto \bx^\top U \bx'$ and the $U'$-weighted inner product $(\bx, \bx') \mapsto \bx^\top U' \bx'$, which are well-defined since $U, U'$ are positive semi-definite. The proof proceeds by bounding each term separately. 

We discuss how to bound $| \bV_1 U \bV_1|$ and $|\bV_2 U \bV_2|$, and the argument for $U'$ is identical. We first define the $n \times n$ matrix with $(A_1)_{ij} = u_{ij} $ for $i, j \leq m$ and $(A_1)_{ij} = 0$, so that we can write $\bV_1^\top U \bV_1 = \bW^\top A_1 \bW$. We can apply the Hanson-Wright lemma (\Cref{lem: hanson-wright inequality}) to conclude that there exists positive constants $C$ such that, for any $\delta > 0$ and any almost-sure sequence $\X_n$, the following event occurs with probability at least $1 - \delta / 4$,
\begin{align*}
    | \bV_1^\top A \bV_1 |
    \;=\;
    | \bW^\top A_1 \bW |
    \;&\leq\;
    \left| - \frac{1}{n} \sum_{1 \leq i, j \leq m} u_{ij} + \sum_{1 \leq i \leq m} u_{ii} \right|
    + C \log\left(\frac{8}{\delta}\right) \left( \sum_{1 \leq i, j \leq m} u_{ij}^2 \right)^{\frac{1}{2}}
    \\
    \;&\leq\;
    \frac{1}{n} \sum_{1 \leq i, j \leq m} |u_{ij}| + \sum_{1 \leq i \leq m} |u_{ii}|
    + C \log\left(\frac{8}{\delta}\right) \left( \sum_{1 \leq i, j \leq m} u_{ij}^2 \right)^{\frac{1}{2}}
    \\
    \;&\leq\;
    \frac{m^2}{n} \tau_\infty + m \tau_\infty + C m \tau_\infty \log\left(\frac{8}{\delta}\right)
    \\
    \;&\leq\;
    2 m \tau_\infty + C m \tau_\infty \log\left(\frac{8}{\delta}\right)
    \tagaligneq \label{eq: boot bound term 1}
    \;,
\end{align*}
where the second last inequality holds since $| u_{ij} | \leq \sup_{\bx, \bx' \in \mathbb{R}^d} | u_p(\bx, \bx') | \leq \tau_\infty$, and the last line holds since $m \leq n$. Similarly, defining the matrix $A_2$ with $(A_2)_{ij} = 0$ for $i, j \leq m$ and $(A_2)_{ij} = u_{ij}$ otherwise, we can write $\bV_2^\top U \bV_2 = \bW^\top A_2 \bW$, and the same argument as before shows that the following holds with probability at least $1 - \delta / 4$,
\begin{align*}
    | \bv_2^\top A \bv_2 |
    \;=\;
    | \bW^\top A_2 \bW |
    \;&\leq\;
    \left| - \frac{1}{n} \sum_{m < i, j \leq n} u_{ij} + \sum_{m < i \leq n} u_{ii} \right|
    + C \log\left(\frac{8}{\delta}\right) \left( \sum_{m < i, j \leq n} u_{ij}^2 \right)^{\frac{1}{2}}
    \\
    \;&\leq\;
    \frac{(n-m)^2 \tau_\infty}{n} + (n-m) \tau_\infty + C (n-m) \tau_\infty \log\left(\frac{8}{\delta}\right)
    \\
    \;&\leq\;
    2(n-m) \tau_\infty + C (n-m) \tau_\infty \log\left(\frac{8}{\delta}\right)
    \tagaligneq \label{eq: boot bound term 2}
    \;,
\end{align*}
Combining \eqref{eq: boot bound term 1} and \eqref{eq: boot bound term 2}, we conclude that the following holds with probability at least $1 - \delta / 2$,
\begin{align*}
    &
    | \bV_2^\top U \bV_2 | + 2| \bV_1^\top U \bV_1 |^{\frac{1}{2}} | \bV_2^\top U \bV_2 |^{\frac{1}{2}}
    \\
    \;&\leq\;
    \left( 2 (n-m) \tau_\infty + C (n-m) \tau_\infty \log\left(\frac{8}{\delta}\right) \right)
    \\ 
    &\;\quad
    + 2 \left( 2 m \tau_\infty + C m \tau_\infty \log\left(\frac{8}{\delta}\right) \right)^{\frac{1}{2}} \left( 2 (n-m) \tau_\infty + C (n-m) \tau_\infty \log\left(\frac{8}{\delta}\right) \right)^{\frac{1}{2}}
    \\
    \;&=\;
    \left( (n-m) + 2m^{\frac{1}{2}}(n-m)^{\frac{1}{2}} \right) \times \tau_\infty \left( 2 + C\log\left(\frac{8}{\delta}\right) \right)
    \;.
\end{align*}
The same argument shows that the last two terms of \eqref{eq: boot bound diff} can be bounded by the same quantity on an event with probability at least $1 - \delta / 2$. To conclude, we have shown that with probability at least $1 - \delta$,
\begin{align*}
    \left| D^2_{\bW}(\X_n) - D^2_{\bW}(\X_n^\ast) \right|
    \;&\leq\;
    \frac{2 }{n^2} \times \left( (n-m) + 2(n-m)^{\frac{1}{2}}m^{\frac{1}{2}} \right) \times \tau_\infty\left( 2 + C\log\left(\frac{8}{\delta}\right) \right)
    \\
    \;&=\;
    \frac{2}{n} \times \frac{(n-m)^{\frac{1}{2}}}{n^{\frac{1}{2}}} \left( \frac{(n-m)^{\frac{1}{2}}}{n^{\frac{1}{2}}} + \frac{2m^{\frac{1}{2}}}{n^{\frac{1}{2}}} \right) \times \tau_\infty \left( 2 + C\log\left(\frac{8}{\delta}\right) \right)
    \\
    \;&\leq\;
    \frac{6\tau_\infty \epsilon_{m,n}^{\frac{1}{2}}}{n} \left( 2 + C\log\left(\frac{8}{\delta}\right) \right)
    \;=\;
    \frac{\tau_\infty \epsilon_{m,n}^{\frac{1}{2}}}{n} \left( 12 + 6C\log\left(\frac{8}{\delta}\right) \right)
    \;,
\end{align*}
where in the second last line we have defined $\epsilon_{m,n} = (n-m)/n$ and used $n-m \leq n$ and $m \leq n$. Defining $C_1 = 12$ and $C_2 = 6C$ shows the first claim. The second claim can be shown by first noting that the above inequality implies that their $(1-\alpha)$-quantiles must satisfy 
\begin{align*}
    \left| q^2_{\infty, 1-\alpha}(\X_n) - q^2_{\infty, 1-\alpha}(\X_n^\ast) \right| 
    \;\leq\; 
    \frac{\tau_\infty \epsilon_{m,n}^{\frac{1}{2}}}{n} \left( C_1 + C_2 \log\left(\frac{8}{\delta}\right) \right)
    \;\eqqcolon\;
    \rho_{m,n,\delta}
    \;.
\end{align*}
We will argue separately for the two cases $(q_{\infty, 1-\alpha}(\X_n) + q_{\infty, 1-\alpha}(\X_n^\ast)) \geq \rho_{m,n,\delta}^{1/2}$ and $(q_{\infty, 1-\alpha}(\X_n) + q_{\infty, 1-\alpha}(\X_n^\ast)) < \rho_{m,n,\delta}^{1/2}$. In the former case, the above inequality implies 
\begin{align*}
    | q_{\infty, 1-\alpha}(\X_n) - q_{\infty, 1-\alpha}(\X_n^\ast) |
    \;&=\;
    \frac{\big| q_{\infty, 1-\alpha}(\X_n) - q_{\infty, 1-\alpha}(\X_n^\ast) \big| \times \big| q_{\infty, 1-\alpha}(\X_n) + q_{\infty, 1-\alpha}(\X_n^\ast) \big|}{q_{\infty, 1-\alpha}(\X_n) + q_{\infty, 1-\alpha}(\X_n^\ast)}
    \\
    \;&=\;
    \frac{\big| q_{\infty, 1-\alpha}^2(\X_n) - q_{\infty, 1-\alpha}^2(\X_n^\ast) \big|}{q_{\infty, 1-\alpha}(\X_n) + q_{\infty, 1-\alpha}(\X_n^\ast) }
    \\
    \;&\leq\; 
    \frac{\rho_{m,n,\delta}}{\rho_{m,n,\delta}^{\frac{1}{2}}} 
    \;=\; 
    \rho_{m,n,\delta}^{\frac{1}{2}}
    \;.
\end{align*}
When instead $(q_{\infty, 1-\alpha}(\X_n) + q_{\infty, 1-\alpha}(\X_n^\ast)) < \rho_{m,n,\delta}^{1/2}$, we have
\begin{align*}
    | q_{\infty, 1-\alpha}(\X_n) - q_{\infty, 1-\alpha}(\X_n^\ast) |
    \;&\leq\;
    \max(q_{\infty, 1-\alpha}(\X_n), \; q_{\infty, 1-\alpha}(\X_n^\ast))
    \\
    \;&\leq\;
    q_{\infty, 1-\alpha}(\X_n) + q_{\infty, 1-\alpha}(\X_n^\ast)
    \\
    \;&\leq\;
    \rho_{m,n,\delta}^{\frac{1}{2}} 
    \;,
\end{align*}
which combined with the previous inequality shows the second claim.
\end{proof}

\subsubsection{Proof of \Cref{lem: bound on test stat tilted kernel}}
\label{pf: bound on test stat tilted kernel}

\begin{proof}
Pick $\{\bz_i\}_{i=1}^{m'} \subset \R^d$, and define $\Z_{m'} \coloneqq \{\bz_i\}_{i=1}^{m'}$. Decomposing the test statistic $D^2(\X_m^\ast \cup \Z_{m'})$ using a similar approach as in \eqref{eq: KSD U-stat decomposition}, we have
\begin{align*}
    D^2(\X_m^\ast \cup \Z_{m'})
    \;&=\;
    \frac{m^2}{n^2} D^2(\X_m^\ast) + \frac{2}{n^2} S_{m} + \frac{1}{n^2} \sum_{m < i, j \leq n} u_p(\bz_i, \bz_j)
    \tagaligneq \label{eq: test stat decomposition}
    \;,
\end{align*}
where $S_{m} \coloneqq \sum_{i=1}^m T_i$ and $T_i \coloneqq \sum_{j=m+1}^n u_p(\bX_i^\ast, \bz_j)$. Similarly, the test statistic computed using the pure sample $\X_n^\ast$ can be written as 
\begin{align*}
    D^2(\X_n^\ast)
    \;&=\;
    \frac{m^2}{n^2} D^2(\X_m^\ast) + \frac{2}{n^2} S_{m}^\ast + \frac{1}{n^2} \sum_{m < i, j \leq n} u_p(\bX_i^\ast, \bX_j^\ast)
    \;,
\end{align*}
where $S_{m}^\ast \coloneqq \sum_{i=1}^m T_i^\ast$ and $T_i^\ast \coloneqq \sum_{j = m+1}^n u_p(\bX_i^\ast, \bX_j^\ast)$. Using a triangle inequality to bound their difference,
\begin{align*}
    n^2\big| D^2(\X_m^\ast \cup \Z_{m'}) - D^2(\X_n^\ast) \big|
    \;&=\;
    \bigg| 2S_m + \sum_{m < i, j \leq n} u_p(\bz_i, \bz_j) - 2S_m^\ast - \sum_{m < i, j \leq n} u_p(\bX_i, \bX_j) \bigg| 
    \\
    \;&\leq\;
    2|S_m| + 2|S_m^\ast| + \sum_{m < i, j \leq n} |u_p(\bz_i, \bz_j) | + | u_p(\bX_i, \bX_j) |
    \\
    \;&\leq\;
    2|S_m| + 2|S_m^\ast| + 2(n-m)^2 \tau_\infty
    \tagaligneq \label{eq: diff vstat bound}
    \;,
\end{align*}
where the last line holds since $\sup_{\bx, \bx' \in \mathbb{R}^d} u_p(\bx, \bx') \leq \tau_\infty$ under the assumed kernel conditions due to \Cref{lem: bounded Stein kernel}. We then bound $|S_m|$ and $|S_m^\ast|$ with high probability. When $\E_{\bX^\ast \sim P}[\| \bs_p(\bX^\ast) \|_2] < \infty$, we have $\E_{\bX^\ast \sim P}[ u_p(\bX^\ast, \cdot) ] = 0$ as argued in the paragraph before \eqref{eq: var bound}, so it is straightforward to see that $\E_{\bX_i^\ast \sim P}[T_i] = 0$ for all $i$, and thus $S_m$ is a sum of zero-mean i.i.d.\ random variables $T_i$. Therefore,
\begin{align*}
    \Var(S_m^2)
    \;&=\;
    m \E_{\bX_1^\ast \sim P}\big[T_1^2\big]
    \;=\;
    m \sum_{m < j, l \leq n} \E_{\bX_1^\ast \sim P}[ u_p(\bX_1^\ast, \bz_j) u_p(\bX_1^\ast, \bz_l) ]
    \;\leq\;
    m (n-m)^2 \tau_\infty^2
    \;,
\end{align*}
where the last step holds again since $\sup_{\bx, \bx' \in \mathbb{R}^d} u_p(\bx, \bx') \leq \tau_\infty$. For any $\delta > 0$, Chebyshev's inequality implies that the event $\cA_1 \coloneqq \{ | S_m |  \leq \sqrt{m}(n-m) \tau_\infty / \sqrt{\delta/2} \}$ occurs with probability at least $1 - \delta/2$. In other words, $|S_m|$ can be upper bounded on the high-probability event $\cA_1$. A similar argument applied to $S_m^\ast$ shows that the event $\cA_2 \coloneqq \big\{ |S_m^\ast| \leq  \sqrt{m} (n-m) \tau_\infty / \sqrt{\delta/2} \}$ also occurs with probability at least $1 - \delta/2$. On $\cA \coloneqq \cA_1 \cap \cA_2$, which occurs with probability at least $1-\delta$, we have by \eqref{eq: diff vstat bound} that
\begin{align*}
    \big| D^2(\X_m^\ast \cup \Z_{m'}) - D^2(\X_n^\ast) \big|
    \;&\leq\;
    \frac{4 \sqrt{m} (n-m) \tau_\infty}{n^2 \sqrt{\delta / 2}} + \frac{2(n-m)^2\tau_\infty}{n^2}
    \\
    \;&\leq\;
    \frac{4\sqrt{2}(n-m)\tau_\infty}{n^{3/2}\sqrt{\delta}} + \frac{2(n-m)^2\tau_\infty}{n^2}
    \\
    \;&=\;
    4\sqrt{2} \epsilon_{m,n} \frac{\tau_\infty}{\sqrt{\delta n}} + 2 \tau_\infty \epsilon_{m,n}^2
    \;\eqqcolon\;
    t_{m, n, \delta}^2
    \tagaligneq \label{eq: UB on D_n on A}
    \;.
\end{align*}
where the second step holds since $\sqrt{m} \leq \sqrt{n}$, and in the last line we have substituted $\epsilon_{m,n} = (n-m) / n$. We claim that this implies that the following holds on $\cA$
\begin{align}
    \big| D(\X_m^\ast \cup \Z_{m'}) - D(\X_n^\ast) \big| \;\leq\; t_{m,n,\delta} \;.
    \label{eq: bound on non-sq D}
\end{align}
To see this, we first assume $D(\X_m^\ast \cup \Z_{m'}) + D(\X_n^\ast) \geq t_{m,n,\delta}$. In this case, \eqref{eq: UB on D_n on A} implies
\begin{align*}
    \big| D(\X_m^\ast \cup \Z_{m'}) - D(\X_n^\ast) \big|
    \;&=\;
    \frac{\big| D(\X_m^\ast \cup \Z_{m'}) - D(\X_n^\ast)\big| \times  \big| D(\X_n) + D(\X_n^\ast) \big|}{D(\X_m^\ast \cup \Z_{m'}) + D(\X_n^\ast)}
    \\
    \;&=\;
    \frac{D^2(\X_m^\ast \cup \Z_{m'}) - D^2(\X_n^\ast)}{D(\X_m^\ast \cup \Z_{m'}) + D(\X_n^\ast)}
    \\
    \;&\leq\;
    t_{m,n,\delta}
    \;.
\end{align*}
On the other hand, when $D(\X_m^\ast \cup \Z_{m'}) + D(\X_n^\ast) < t_{m,n,\delta}$, we have
\begin{align*}
    \big| D(\X_m^\ast \cup \Z_{m'}) - D(\X_n^\ast) \big|
    \;\leq\;
    \max\big(D(\X_m^\ast \cup \Z_{m'}), \; D(\X_n^\ast) \big)
    \;\leq\;
    D(\X_m^\ast \cup \Z_{m'}) + D(\X_n^\ast)
    \;<\;
    t_{m,n,\delta}
    \;,
\end{align*}
where the first inequality holds since $D(\X_m^\ast \cup \Z_{m'}), D(\X_n^\ast)$ are non-negative. Combining these two cases shows \eqref{eq: bound on non-sq D}. It then follows that, for any $\gamma > 0$,
\begin{align*}
    &
    \Pr\nolimits_{\X_n^\ast \sim P}\big( D(\X_m^\ast \cup \Z_{m'}) > \gamma \big)
    \\
    \;&=\;
    \Pr\nolimits_{\X_n^\ast \sim P}\big( \{D(\X_m^\ast \cup \Z_{m'}) > \gamma \} \cap \cA \big) + \Pr\nolimits_{\X_n^\ast \sim P}\big( \{ D(\X_m^\ast \cup \Z_{m'}) > \gamma \} \cap \cA^\complement \big)
    \\
    \;&\leq\;
    \Pr\nolimits_{\X_n^\ast \sim P}\big( D(\X_n^\ast) + t_{m,n,\delta} > \gamma \big) + \delta
    \\
    \;&=\;
    \Pr\nolimits_{\X_n^\ast \sim P}\big( D(\X_n^\ast) > \gamma - t_{m,n,\delta} \big) + \delta
    \;,
\end{align*}
where the first step follows from the law of total probability, and the second line holds since $\Pr\nolimits_{\X_n^\ast \sim P}\big( \{ D(\X_m^\ast \cup \Z_{m'}) > \gamma \} \cap \cA^\complement \big) \leq \Pr(\cA^\complement) \leq \delta$. 
This implies one side of the desired inequality
\begin{align*}
    &
    \Pr\nolimits_{\X_n^\ast \sim P}\big( D(\X_m^\ast \cup \Z_{m'}) > \gamma \big) - \Pr\nolimits_{\X_n^\ast \sim P}\big( D(\X_n^\ast) > \gamma \big)
    \\
    \;&\leq\;
    \Pr\nolimits_{\X_n^\ast \sim P}\big( D(\X_n^\ast) > \gamma - t_{m,n,\delta} \big) - \Pr\nolimits_{\X_n^\ast \sim P}\big( D(\X_n^\ast) > \gamma \big) + \delta
    \\
    \;&=\;
    \Pr\nolimits_{\X_n^\ast \sim P}\big( \gamma - t_{m,n,\delta} \;\leq\; D(\X_n^\ast) \;\leq\; \gamma \big) + \delta
    \\
    \;&\leq\;
    \Pr\nolimits_{\X_n^\ast \sim P}\big( \big| D(\X_n^\ast) - \gamma \big| \;\leq\; t_{m,n,\delta}\big) + \delta
    \;,
\end{align*}
where in the last line we have used 
\begin{align*}
    \Pr\nolimits_{\X_n^\ast \sim P}\big( \gamma - t_{m,n,\delta} \;\leq\; D(\X_n^\ast) \;\leq\; \gamma \big) 
    \;&\leq\;
    \Pr\nolimits_{\X_n^\ast \sim P}\big( \gamma - t_{m,n,\delta} \;\leq\; D(\X_n^\ast) \;\leq\; \gamma + t_{m,n,\delta} \big) 
    \\
    \;&=\;
    \Pr\nolimits_{\X_n^\ast \sim P}\big( \big| D(\X_n^\ast) - \gamma \big| \;\leq\; t_{m,n,\delta}\big)
    \;.
\end{align*}
A similar argument shows the other direction
\begin{align*}
    &
    \Pr\nolimits_{\X_n^\ast \sim P}\big( D(\X_m^\ast \cup \Z_{m'}) > \gamma \big) - \Pr\nolimits_{\X_n^\ast \sim P}\big( D(\X_n^\ast) > \gamma \big)
    \\
    \;&\geq\;
    \Pr\nolimits_{\X_n^\ast \sim P}\big( \{ D(\X_m^\ast \cup \Z_{m'}) > \gamma \} \cap \cA \big)
    - \Pr\nolimits_{\X_n^\ast \sim P}\big( D(\X_n^\ast) > \gamma \big)
    \\
    \;&\geq\;
    \Pr\nolimits_{\X_n^\ast \sim P}\big( D(\X_n^\ast) > \gamma +  t_{m,n,\delta} \big)
    - \Pr\nolimits_{\X_n^\ast \sim P}\big( D(\X_n^\ast) > \gamma \big)
    \\
    \;&=\;
    - \Pr\nolimits_{\X_n^\ast \sim P}\big( \gamma \leq D(\X_n^\ast) \leq \gamma +  t_{m,n,\delta} \big)
    \\
    \;&\geq\;
    - \Pr\nolimits_{\X_n^\ast \sim P}\big( \gamma - t_{m,n,\delta} \leq D(\X_n^\ast) \leq \gamma +  t_{m,n,\delta} \big)
    \\
    \;&=\;
    - \Pr\nolimits_{\X_n^\ast \sim P}\big( \big| D(\X_n^\ast) - \gamma \big| \;\leq\; t_{m,n,\delta} \big)
    \;,
\end{align*}
where the first step holds by the law of total probability, and the second step follows from \eqref{eq: bound on non-sq D}. 
\end{proof}

\subsubsection{Proof of \Cref{lem: robustness tilted conditional}}
\label{pf: robustness tilted conditional}

\begin{proof}
Fix any $\theta, \gamma \geq 0$. Let $m', n$ be any positive integers with $m' \leq f_n$, where $f_n = o(n^{1/2})$. Define $m = n - m'$. Pick $Z_{m'} = \{ \bz_i\}_{i=1}^{m'} \subset \R^d$, and denote the LHS of the inequality by
\begin{align*}
    T(\bz_1, \ldots, \bz_m')
    \;&\coloneqq\;
    \big| \Pr\nolimits_{\X_m^\ast \sim P, \bW}\big( \Delta_\theta(\X_m^\ast \cup \Z_{m'}) > q_{\infty, 1-\alpha}(\X_m^\ast \cup \Z_{m'}) \big) 
    \\
    &\;\qquad
    - \Pr\nolimits_{\X_n^\ast \sim P, \bW}\big( \Delta_\theta(\X_n^\ast) > q_{\infty, 1-\alpha}(\X_n^\ast) \big) \big|
    \;.
\end{align*}
The LHS of the inequality can be bounded as
\begin{align*}
    T(\bz_1, \ldots, \bz_m')
    \;&=\;
    \big| \Pr\nolimits_{\X_m^\ast \sim P, \bW}\big( D(\X_m^\ast \cup \Z_{m'}) > q_{\infty, 1-\alpha}(\X_m^\ast \cup \Z_{m'}) + \theta \big) 
    \\ 
    &\;\qquad
    - \Pr\nolimits_{\X_n^\ast \sim P, \bW}\big( D(\X_n^\ast) > q_{\infty, 1-\alpha}(\X_n^\ast) + \theta \big) \big|
    \\
    \;&\leq\;
    \Big| \Pr\nolimits_{\X_m^\ast \sim P, \bW}\big( D(\X_m^\ast \cup \Z_{m'}) > q_{\infty, 1-\alpha}(\X_m^\ast \cup \Z_{m'}) + \theta \big) 
    \\ 
    &\;\qquad
    - \Pr\nolimits_{\X_n^\ast \sim P, \bW}\big( D(\X_n^\ast) > q_{\infty, 1-\alpha}(\X_m^\ast \cup \Z_{m'}) + \theta \big) \Big|
    \\
    &\;\quad + \Big| \Pr\nolimits_{\X_m^\ast \sim P, \bW}\big( D(\X_n^\ast) > q_{\infty, 1-\alpha}(\X_m^\ast \cup \Z_{m'}) + \theta \big) 
    \\
    &\;\qquad\;\;\;
    - \Pr\nolimits_{\X_n^\ast \sim P, \bW}\big( D(\X_n^\ast) > q_{\infty, 1-\alpha}(\X_n^\ast)+ \theta \big) \Big|
    \\
    \;&\eqqcolon\;
    T_1 + T_2
    \tagaligneq \label{eq: bound diff probs T1 T2}
    \;,
\end{align*}
where the first equality holds since, for any $\gamma \geq 0$, it can be checked that the inequality $\Delta_\theta = \max(0, D(\X_n) - \theta) > \gamma$ holds if and only if $D(\X_n) < \gamma + \theta$. We will now bound the terms $T_1$ and $T_2$, respectively. Denote for brevity $\gamma_\bz \coloneqq q_{\infty, 1-\alpha}(\X_m^\ast \cup \Z_{m'}) + \theta$ and $\gamma_\ast \coloneqq q_{\infty, 1-\alpha}(\X_n^\ast) + \theta$. Fix any $\delta > 0$. Applying \Cref{lem: bound on test stat tilted kernel} with $\gamma = \gamma_\bz$ yields
\begin{align}
    T_1
    \;&\leq\;
    \Pr\nolimits_{\X_n^\ast \sim P, \bW}\big(  \big| D(\X_n^\ast) - \gamma_\bz \big| \;\leq\; t_{m,n,\delta} \big) + \delta
    \nonumber \\
    \;&\leq\;
    \Pr\nolimits_{\X_n^\ast \sim P, \bW}\big(  \big| D(\X_n^\ast) - \gamma_\ast \big| \;\leq\; t_{m,n,\delta} + | \gamma_\bz - \gamma_\ast | \big) + \delta
    \label{eq: bound on T1 proof of robustness}
    \;,
\end{align}
where $t_{m,n,\delta} \coloneqq \big(4\sqrt{2} \tau_\infty \epsilon_{m, n} / \sqrt{\delta n} + 2 \tau_\infty \epsilon_{m, n}^2 \big)^{1/2}$ is defined in \Cref{lem: bound on test stat tilted kernel}, and \eqref{eq: bound on T1 proof of robustness} follows from a triangle inequality. Moreover, by \Cref{lem: diff of quantiles tilted kernel}, there exists an event, say $\cA$, with probability at least $1-\delta$ such that, on $\cA$,
\begin{align*}
   | \gamma_\bz - \gamma_\ast |
    \;=\;
    \big| q_{\infty, 1-\alpha}(\X_m^\ast \cup \Z_{m'}) - q_{\infty, 1-\alpha}(\X_n^\ast) \big| 
     \; \leq\; 
    \frac{\tau_\infty^{\frac{1}{2}} \epsilon_{m,n}^{\frac{1}{4}}}{n^{\frac{1}{2}}} \sqrt{ C_1 + C_2 \log\left(\frac{8}{\delta}\right) }
    \;\eqqcolon\;
    \rho_{m,n,\delta}
    \;.
\end{align*}
Combining this with \eqref{eq: bound on T1 proof of robustness} implies that $T_1$ can be bounded as 
\begin{align*}
    T_1
    \;&\leq\;
    \Pr\nolimits_{\X_n^\ast \sim P, \bW}\big( \big\{ \big| D(\X_n^\ast) - \gamma_\ast \big| \;\leq\; t_{m,n,\delta} + | \gamma_\bz - \gamma_\ast | \big\} \cap \cA \big) + \Pr(\cA^\complement) + \delta
    \\
    \;&\leq\;
    \Pr\nolimits_{\X_n^\ast \sim P, \bW}\big( \big| D(\X_n^\ast) - \gamma_\ast \big| \;\leq\; t_{m,n,\delta} + \rho_{m,n,\delta} \big) + 2\delta
    \tagaligneq \label{eq: bound on T1 proof of robustness 2}
    \;,
\end{align*}
where the second line holds since $\Pr(\cA^\complement) \leq \delta$.
To bound $T_2$, we first note that
\begin{align*}
    T_2
    \;&=\;
    \big| \Pr\nolimits_{\bX_n^\ast \sim P}\big( D(\X_n^\ast) > \gamma_\bz \big)
    - \Pr\nolimits_{\bX_n^\ast \sim P}\big( D(\X_n^\ast) > \gamma_\ast \big) \big|
    \\
    \;&=\;
    \Pr\nolimits_{\bX_n^\ast \sim P}\big( \min(\gamma_\bz,\; \gamma_\ast) \;<\; D(\X_n^\ast) \;\leq\; \max(\gamma_\bz,\; \gamma_\ast)  \big)
    \\
    \;&\leq\;
    \Pr\nolimits_{\bX_n^\ast \sim P}\big( \big\{ \min(\gamma_\bz,\; \gamma_\ast) \;<\; D(\X_n^\ast) \;\leq\; \max(\gamma_\bz,\; \gamma_\ast) \big\} \cap \cA  \big) + \Pr(\cA^\complement)
    \\
    \;&\leq\;
    \Pr\nolimits_{\bX_n^\ast \sim P}\big( \big| D(\X_n^\ast) - \gamma_\ast \big| \;\leq\; \rho_{m,n,\delta} \big) + \delta
    \tagaligneq \label{eq: bound on T2 proof of robustness}
    \;,
\end{align*}
where the second line holds since $|\Pr(Y > a) - \Pr(Y > b)| = \Pr( \min(a,b) < Y \leq \max(a, b))$ for any constants $a, b$ and random variable $Y$, and the last line holds since $\Pr(\cA^\complement) \leq \delta$ and since on $\cA$ we have
\begin{align*}
    \gamma_\ast - \rho_{m,n,\delta}
    \;\leq\;
    \gamma_\ast - \big| \gamma_\bz - \gamma_\ast \big|
    \;\leq\;
    \min(\gamma_\bz,\; \gamma_\ast)
    \;\leq\;
    \max(\gamma_\bz,\; \gamma_\ast)
    \;&\leq\;
    \gamma_\ast + \big| \gamma_\bz - \gamma_\ast \big|
    \\
    \;&\leq\;
    \gamma_\ast + \rho_{m,n,\delta}
    \;.
\end{align*}
Substituting the bounds \eqref{eq: bound on T1 proof of robustness 2} and \eqref{eq: bound on T2 proof of robustness} into \eqref{eq: bound diff probs T1 T2} gives
\begin{align*}
    T(\bz_1, \ldots, \bz_m')
    \;&\leq\;
    \Pr\nolimits_{\X_n^\ast \sim P, \bW}\big( \big| D(\X_n^\ast) - \gamma_\ast \big| \;\leq\; t_{m,n,\delta} + \rho_{m,n,\delta} \big)
    \\
    &\;\quad
    + \Pr\nolimits_{\X_n^\ast \sim P, \bW}\big( \big| D(\X_n^\ast) - \gamma_\ast \big| \;\leq\; \rho_{m,n,\delta} \big)
    + 3\delta
    \\
    \;&\leq\;
    2 \Pr\nolimits_{\X_n^\ast \sim P, \bW}\big( \big| D(\X_n^\ast) - \gamma_\ast \big| \;\leq\; t_{m,n,\delta} + \rho_{m,n,\delta} \big)
    + 3\delta
    \tagaligneq \label{eq: bound on diff in rej probs}
    \;,
\end{align*}
where the last inequality holds since $t_{m,n,\delta} \geq 0$. Since $m' \leq f_n$ and $f_n = o(n^{1/2})$ by assumption, we have $\epsilon_{m,n} = m' / n \leq f_n / n = o(n^{-1/2})$. This implies
\begin{align*}
    t_{m, n, \delta} \;=\; \left(4\sqrt{2} \tau_\infty \delta^{-\frac{1}{2}} \epsilon_{m, n} n^{-\frac{1}{2}} + 2 \tau_\infty \epsilon_{m, n}^2 \right)^{\frac{1}{2}}
    \;&\leq\;
    \left(4\sqrt{2} \tau_\infty \delta\right)^{\frac{1}{2}} \epsilon_{m, n}^{\frac{1}{2}} n^{-\frac{1}{4}} + \sqrt{2} \tau_\infty^{\frac{1}{2}} \epsilon_{m, n}
    \\
    \;&\in\;
    o\left(n^{-\frac{1}{2}}\right)
    \;,
\end{align*}
which holds since $\tau_\infty, \delta$ are constants, and
\begin{align}
   \rho_{m,n,\delta} \;=\; \frac{\tau_\infty^{\frac{1}{2}} \epsilon_{m,n}^{\frac{1}{4}}}{n^{\frac{1}{2}}} \sqrt{ C_1 + C_2 \log\big(8 n^{2(1-2s)} \big) }
    \;&=\;
    \frac{\tau_\infty^{\frac{1}{2}} \epsilon_{m,n}^{\frac{1}{4}}}{n^{\frac{1}{2}}} \sqrt{ C_1' + C_2' \log(n) }
    \nonumber \\
    \;&\in\;
    o\left( n^{-\frac{1}{2}} \right)
    \label{eq: rho rate}
    \;,
\end{align}
where we have defined $C_1' \coloneqq C_1 + C_2 \log 8$ and $C_2' \coloneqq 2C_2 (1-2s)$, and the last step holds since $\epsilon_{m,n}^{1/4} \sqrt{\log(n)} = o\big(n^{-1/8} \sqrt{\log(n)} \big) = o(1)$. 

Define $\eta(m') \coloneqq t_{m,n,\delta} + \rho_{m,n,\delta}$, where the dependence on $m'$ is through $m = n - m'$. Then the above derivations show that $\eta(m') = o(n^{-1/2})$ for all $m' \leq f_n$. In particular, $\eta(f_n) = o(n^{-1/2})$. Moreover, substituting these into \eqref{eq: bound on diff in rej probs} gives
\begin{align*}
    T(\bz_1, \ldots, \bz_m')
    \;&\leq\;
    2 \Pr\nolimits_{\X_n^\ast \sim P, \bW}\big( \big| D(\X_n^\ast) - \gamma_\ast \big| \;\leq\; \eta(m') \big)
    + 3\delta
    \\
    \;&\leq\;
    2 \Pr\nolimits_{\X_n^\ast \sim P, \bW}\left( \big| D(\X_n^\ast) - \gamma_\ast \big| \;\leq\; \eta\left(f_n\right) \right)
    + 3\delta
    \;,
\end{align*}
where the last line holds since $m' \leq f_n$ by assumption and $m' \mapsto \eta(m')$ is monotone increasing by direct computation. Taking supremum over $\Z_{m'} \subset \R^d$ and $m' \leq f_n$ gives
\begin{align*}
    \max_{m' \leq f_n} \sup_{\bz_1, \ldots, \bz_{m'} \in \R^d} T(\bz_1, \ldots, \bz_m')
    \;&\leq
    2 \Pr\nolimits_{\X_n^\ast \sim P, \bW}\left( \big| D(\X_n^\ast) - \gamma_\ast \big| \;\leq\; \eta\left(f_n\right) \right)
    + 3\delta
\end{align*}
It remains to show that the first term on the RHS can be bounded by $\delta$, from which the claimed result in \Cref{lem: robustness tilted conditional} would follow by redefining $\delta$. To proceed, we write
\begin{subequations}
\begin{align*}
    &\;
    \Pr\nolimits_{\X_n^\ast \sim P, \bW}\big( \left| \Delta_\theta(\X_n^\ast) - q_{\infty, 1-\alpha}(\X_n^\ast) \right| \;\leq\; \eta\left(f_n\right) \big)
    \\
    \;&=\;
    \Pr\nolimits_{\X_n^\ast \sim P, \bW}\big( \Delta_\theta(\X_n^\ast) \;\leq\; q_{\infty, 1-\alpha}(\X_n^\ast) + \eta\left(f_n\right) \big)
    \\
    &\;\quad
    - \Pr\nolimits_{\X_n^\ast \sim P, \bW}\big( \Delta_\theta(\X_n^\ast) \;\leq\; q_{\infty, 1-\alpha}(\X_n) - \eta\left(f_n\right) \big) 
    \\
    \;&=\;
    \Pr\nolimits_{\X_n^\ast \sim P, \bW}\big( \sqrt{n} D(\X_n^\ast) \;\leq\; \sqrt{n} \zeta_{n, \theta} + \sqrt{n}\eta\left(f_n\right) \big)
    \tagaligneq \\
    &\;\quad
    - \Pr\nolimits_{\X_n^\ast \sim P, \bW}\big( \sqrt{n} D(\X_n^\ast) \;\leq\; \sqrt{n} \zeta_{n, \theta} - \sqrt{n}\eta\left(f_n\right) \big)
    \tagaligneq
\end{align*}
\label{eq: error prob bound}%
\end{subequations}
where the first equality holds again because $|\Pr(Y > a) - \Pr(Y > b)| = \Pr( \min(a,b) < Y \leq \max(a, b))$ for any constants $a, b$ and random variable $Y$, the second equality holds by noting that $\Delta_\theta(\X_n^\ast) = \max(0, D(\X_n^\ast) - \theta) \leq t$ for some $ t > 0$ if and only if $D(\X_n^\ast) - \theta \leq t$, and \eqref{eq: error prob bound} holds by multipling $\sqrt{n}$ on both sides of the inequalities within the probabilities and defining $\zeta_{n,\theta} \coloneqq q_{\infty, 1-\alpha}(\X_n^\ast) + \theta$.

Under the assumed conditions on $k$ and assuming $\E_{\bX^\ast \sim P}[\| \bs_p(\bX) \|_2] < \infty$, \citet[Proposition 1]{gorham2017measuring} shows that $\E_{\bX^\ast \sim P}[u_p(\bX^\ast, \cdot)] = 0$, so the V-statistic $D^2(\X_n^\ast)$ is \emph{degenerate} \citep[see, e.g.,][Chapter 6]{serfling2009approximation}, and classic results on the asymptotics of degenerate V-statistics \citep[Theorem 6.4.1 A]{serfling2009approximation} shows that $n D^2(\X_n^\ast)$ converges weakly to a non-negative distribution. Since the square-root function $x \mapsto \sqrt{x}$ is continuous on $[0, \infty)$ and continuous functions preserve weak limits by the Continuous Mapping Theorem \citep[Theorem 2.3]{vandervaart2000asymptotic}, the squared-root statistic, $\sqrt{n} D(\X_n^\ast)$, also converges weakly to a non-negative distribution. In particular, the scaled quantile $\sqrt{n} q_{\infty, 1-\alpha}(\X_n^\ast)$, and thus also $\sqrt{n} \zeta_{n,\theta}$, converge to a positive number as $n \to \infty$. Also since $\eta\left(f_n\right) = o(n^{-1/2})$ as argued below \eqref{eq: rho rate}, we have $\sqrt{n} \eta\left(f_n\right) \to 0$. In summary, we have shown that
\begin{align*}
    &
    \lim_{n\to\infty}
    \Pr\nolimits_{\X_n^\ast \sim P, \bW}\left( \left| \Delta_\theta(\X_n^\ast) - q_{\infty, 1-\alpha}(\X_n^\ast) \right| \;\leq\; \eta\left(f_n\right) \right)
    \\
    \;&=\;
    \lim_{n\to\infty} \Big(
    \Pr\nolimits_{\X_n^\ast \sim P, \bW}\big( \sqrt{n} D(\X_n^\ast) \;\leq\; \sqrt{n} \zeta_{n, \theta} + \sqrt{n}\eta\left(f_n\right) \big)
    \\
    &\;\qquad\qquad\;
    - \Pr\nolimits_{\X_n^\ast \sim P, \bW}\big( \sqrt{n} D(\X_n^\ast) \;\leq\; \sqrt{n} \zeta_{n, \theta} - \sqrt{n}\eta\left(f_n\right) \big)
    \Big)
    \\
    \;&=\;
    \lim_{n\to\infty}
    \Pr\nolimits_{\X_n^\ast \sim P, \bW}\big( \sqrt{n} D(\X_n^\ast) \;\leq\; \sqrt{n} \zeta_{n, \theta} \big)
    - \Pr\nolimits_{\X_n^\ast \sim P, \bW}\big( \sqrt{n} D(\X_n^\ast) \;\leq\; \sqrt{n} \zeta_{n, \theta} \big)
    \\
    \;&=\;
    0 \;,
\end{align*}
which completes the proof.
\end{proof}

\subsubsection{Proofs for \Cref{prop: qualitative robust tilted general} and \Cref{thm: robust tilted}}
\label{pf: qualitative robust tilted general}

\begin{proof}[Proof of \Cref{prop: qualitative robust tilted general}]
    For each $n$ and $R \in \cP(\R^d)$, define $Q = (1-\epsilon_n) P + \epsilon_n R$. Let $\X_n$ and $\X_n^\ast$ be random samples drawn from $Q$ and $P$, respectively. As argued in \Cref{app: non-robust stationary}, each random variable in $\X_n$ can be written as $\bX_i \stackrel{d}{=} (1-\xi_i) \bX_i^\ast + \xi_i \bZ_i$, where $\bZ_i \sim R$ and $\xi_i \sim \mathrm{Bernoulli}(n, \epsilon_n)$ are independent, and $\stackrel{d}{=}$ denotes equality in distribution. Define the random variable $M' = \sum_{i=1}^n \xi_i$, then $M' \sim \mathrm{Binomial}(n, \epsilon_n)$.
    
    Pick $\delta > 0$. Given any sequence $\epsilon_n = o(n^{-1/2})$, there must exists $f_n = o(n^{1/2})$ such that $\epsilon_n \leq f_n/n$ and $f_n \to \infty$ as $n \to \infty$.  Take such $f$ and define the event $\cB \coloneqq \{ M' - \epsilon_n n \leq f_n \}$. Intuitively, $\cB$ is the event where the number of outliers $M'$ does not deviate from its mean $\epsilon_n n $ by more than $f_n$. Our proof proceeds by first showing that $\cB$ occurs with high probability, then proving that the claimed result holds on this event.
    
    To show $\cB$ occurs with high probability, we use a concentration inequality for Binomial distributions; see, e.g., \citet[Lemma 2.1, Eq.~2.2]{chung2002connected}. Applying this result to $\text{Binomial}(n,\epsilon_n)$, we have
    \begin{align}
        \Pr\nolimits_{M'}\left( \cB^\complement \right)
        \;=\;
        \Pr\nolimits_{M'}\left(M' - \epsilon_n n > f_n \right)
        \;\leq\;
        \exp\left( - \frac{f_n^2}{2(\epsilon_n n + f_n / 3)}\right)
        \;\eqqcolon\;
        t_n
        \label{eq: binom concentration 2}
        \;.
    \end{align}
    Define the event $\cA(\X_n) = \{ \Delta_\theta(\X_n) > q_{\infty, 1-\alpha}(\X_n)\}$ and similarly for $\cA(\X_n^\ast)$. Using the above inequality to decompose $\Pr_{\X_n \sim Q, \bW}(\cA(\X_n))$ yields
    \begin{align*}
        &
        \Pr\nolimits_{\X_n \sim Q, \bW}( \cA(\X_n) )
        \\
        \;&\leq\;
        \Pr\nolimits_{\X_n \sim Q, \bW}( \cA(\X_n) \cap \cB ) + \Pr(\cB^\complement)
        \\
        \;&=\;
        \sum_{m' \leq \epsilon_n n + f_n} \Pr\nolimits_{\X_n \sim Q, \bW}( \cA(\X_n) \;|\; M' = m' ) \Pr(M' = m') + \Pr(\cB^\complement)
        \\
        \;&=\;
        \sum_{m' \leq \epsilon_n n + f_n}\Pr\nolimits_{\X_{n-m'}^\ast \sim P,\; \Z_{m'} \sim R, \bW}( \cA(\X_{n-m'}^\ast \cup \Z_{m'}) ) \Pr(M' = m') + \Pr(\cB^\complement)
        \\
        \;&\leq\;
        \max_{m' \leq \epsilon_n n + f_n}\Pr\nolimits_{\X_{n-m'}^\ast \sim P,\; \Z_{m'} \sim R, \bW}( \cA(\X_{n-m'}^\ast \cup \Z_{m'}) ) \sum_{m' \leq \epsilon_n n + f_n} \Pr(M' = m') + \Pr(\cB^\complement)
        \\
        \;&=\;
        \max_{m' \leq \epsilon_n n + f_n}\Pr\nolimits_{\X_{n-m'}^\ast \sim P,\; \Z_{m'} \sim R, \bW}( \cA(\X_{n-m'}^\ast \cup \Z_{m'}) ) \Pr(\cB) + \Pr(\cB^\complement)
        \\
        \;&\leq\;
        \max_{m' \leq \epsilon_n n + f_n}\Pr\nolimits_{\X_{n-m'}^\ast \sim P,\; \Z_{m'} \sim R, \bW}( \cA(\X_{n-m'}^\ast \cup \Z_{m'}) ) + t_n
        \;,
    \end{align*}
    where in the last line we have used \eqref{eq: binom concentration 2} and that $\Pr(\cB) \leq 1$. This implies
    \begin{align*}
        &
        \Pr\nolimits_{\X_n \sim Q, \bW}( \cA(\X_n) ) - \Pr\nolimits_{\X_n^\ast \sim P}( \cA(\X_n^\ast) )
        \\
        \;&\leq\;
        \max_{m' \leq \epsilon_n n + f_n} \big( \Pr\nolimits_{\X_{n-m'}^\ast \sim P,\; \Z_{m'} \sim R, \bW}( \cA(\X_{n-m'}^\ast \cup \Z_{m'}) ) - \Pr\nolimits_{\X_n^\ast \sim P}( \cA(\X_n^\ast) ) \big) 
        + t_n
        \\
        \;&\leq\;
        \max_{m' \leq \epsilon_n n + f_n} \sup_{\bz_1, \ldots, \bz_{m'} \in \mathbb{R}^d} \big| \Pr\nolimits_{\X_m^\ast \sim P}( \cA(\X_m^\ast \cup \{\bz_i\}_{i=1}^{m'} ) ) - \Pr\nolimits_{\X_n^\ast \sim P}( \cA(\X_n^\ast) ) \big| 
        + t_n 
        \\
        \;&\eqqcolon\;
        \max_{m' \leq \epsilon_n n + f_n} \omega(m') + t_n
        \;,
    \end{align*}    
    where in the second inequality we have taken supremum over all possible values of $\Z_{m'}$, and in the last line we have defined
    \begin{align*}
        \omega(m') 
        \;\coloneqq\;
        \sup_{\bz_1, \ldots, \bz_{m'} \in \mathbb{R}^d} \big| \Pr\nolimits_{\X_{n-m'}^\ast \sim P}\big( \cA\big(\X_{n-m'}^\ast \cup \{\bz_i\}_{i=1}^{m'} \big) \big)  - \Pr\nolimits_{\X_n^\ast \sim P}\big( \cA(\X_n^\ast) \big) \big| 
        \;.
    \end{align*}
    Similarly, we have the lower bound
    \begin{align*}
        &
        \Pr\nolimits_{\X_n \sim Q, \bW}( \cA(\X_n) )
        \\
        \;&\geq\;
        \Pr\nolimits_{\X_n \sim Q, \bW}( \cA(\X_n) \cap \cB)
        \\
        \;&=\;
        \sum_{m' \leq \epsilon_n n + f_n}\Pr\nolimits_{\X_n \sim Q, \bW}( \cA(\X_n) \;|\; M' = m' ) \Pr(M' = m')
        \\
        \;&=\;
        \sum_{m' \leq \epsilon_n n + f_n}\Pr\nolimits_{\X_{n-m'}^\ast \sim P,\; \Z_{m'} \sim R, \bW}( \cA(\X_{n-m'}^\ast \cup \Z_{m'}) ) \Pr(M' = m')
        \\
        \;&\geq\;
        \min_{m' \leq \epsilon_n n + f_n}\Pr\nolimits_{\X_{n-m'}^\ast \sim P,\; \Z_{m'} \sim R, \bW}( \cA(\X_{n-m'}^\ast \cup \Z_{m'}) ) \sum_{m' \leq \epsilon_n n + f_n} \Pr(M' = m')
        \\
        \;&=\;
        \min_{m' \leq \epsilon_n n + f_n}\Pr\nolimits_{\X_{n-m'}^\ast \sim P,\; \Z_{m'} \sim R, \bW}( \cA(\X_{n-m'}^\ast \cup \Z_{m'}) ) \Pr(\cB)
        \;,
    \end{align*}
    where in the second last line we have taken infimum over the values of $\Z_{m'}$. Taking infimum over $\bz_1, \ldots, \bz_{m'}$, the last line is lower bounded by
    \begin{align*}
        &
        \min_{m' \leq \epsilon_n n + f_n}\Pr\nolimits_{\X_{n-m'}^\ast \sim P,\; \Z_{m'} \sim R, \bW}( \cA(\X_{n-m'}^\ast \cup \Z_{m'}) ) \Pr(\cB)
        \\
        \;&\geq\;
        \min_{m' \leq \epsilon_n n + f_n} \inf_{\bz_1, \ldots, \bz_{m'}} \Pr\nolimits_{\X_{n-m'}^\ast \sim P}( \cA(\X_{n-m'}^\ast \cup \{\bz_i\}_{i=1}^{m'}) ) \Pr(\cB)
        \\
        \;&\geq\;
        \min_{m' \leq \epsilon_n n + f_n} \inf_{\bz_1, \ldots, \bz_{m'}} \Pr\nolimits_{\X_{n-m'}^\ast \sim P}( \cA(\X_{n-m'}^\ast \cup \{\bz_i\}_{i=1}^{m'}) ) \times (1 - t_n)
        \\
        \;&=\;
        \min_{m' \leq \epsilon_n n + f_n} \inf_{\bz_1, \ldots, \bz_{m'}} \Big( \Pr\nolimits_{\X_{n-m'}^\ast \sim P}( \cA(\X_{n-m'}^\ast \cup \{\bz_i\}_{i=1}^{m'}) ) 
        \\
        \;&\qquad\qquad\qquad\qquad\qquad 
        - t_n \Pr\nolimits_{\X_{n-m'}^\ast \sim P}( \cA(\X_{n-m'}^\ast \cup \{\bz_i\}_{i=1}^{m'}) )
        \Big)
        \\
        \;&\geq\;
        \min_{m' \leq \epsilon_n n + f_n} \inf_{\bz_1, \ldots, \bz_{m'}} \Pr\nolimits_{\X_{n-m'}^\ast \sim P}( \cA(\X_{n-m'}^\ast \cup \{\bz_i\}_{i=1}^{m'}) ) 
        - t_n
        \;,
        \tagaligneq \label{eq: binom bound claim 1 pf}
    \end{align*}
    where the second inequality holds since \eqref{eq: binom concentration 2} implies $\Pr(\cB) \geq 1 - t_n$, and the last line holds as a probability is always smaller than or equal to $1$. This implies
    \begin{align*}
        &
        \Pr\nolimits_{\X_n \sim Q, \bW}( \cA(\X_n) ) - \Pr\nolimits_{\X_n^\ast \sim P}( \cA(\X_n^\ast) )
        \\
        \;&\geq\;
        \min_{m' \leq \epsilon_n n + f_n} \inf_{\bz_1, \ldots, \bz_{m'}} \big( \Pr\nolimits_{\X_{n-m'}^\ast \sim P}\big( \cA\big(\X_{n-m'}^\ast \cup \{\bz_i\}_{i=1}^{m'}\big) \big)  - \Pr\nolimits_{\X_n^\ast \sim P}\big( \cA(\X_n^\ast) \big) \big) 
        - t_n
        \\
        \;&\geq\;
        - \max_{m' \leq \epsilon_n n + f_n} \sup_{\bz_1, \ldots, \bz_{m'} \in \mathbb{R}^d} \big| \Pr\nolimits_{\X_{n-m'}^\ast \sim P}\big( \cA\big(\X_{n-m'}^\ast \cup \{\bz_i\}_{i=1}^{m'} \big) \big)  - \Pr\nolimits_{\X_n^\ast \sim P}\big( \cA(\X_n^\ast) \big) \big| 
        - t_n
        \\
        \;&=\;
        - \max_{m' \leq \epsilon_n n + f_n} \omega(m') - t_n
        \;.
    \end{align*}
    We have therefore shown that
    \begin{align*}
        \big| \Pr\nolimits_{\X_n \sim Q, \bW}( \cA(\X_n) ) - \Pr\nolimits_{\X_n^\ast \sim P}( \cA(\X_n^\ast) ) \big|
        \;&\leq\;
        \max_{m' \leq \epsilon_n n + f_n} \omega(m')
        + t_n
        \tagaligneq \label{eq: rej prob diff}
        \;.
    \end{align*}
    It remains to bound the two terms on the RHS of the above inequality. The second term can be bounded by noting that, since $\epsilon_n \leq f_n / n$ and $f_n = o(n^{1/2})$, it is clear that $t_n$, defined in \eqref{eq: binom concentration 2}, converges to 0 as $n \to \infty$, so $t_n \leq \delta/2$ for sufficiently large $n$. To bound the first term, since $\epsilon_n n + f_n \leq 2f_n = o(n^{1/2})$ and the assumptions in \Cref{lem: robustness tilted conditional} hold, we can apply \Cref{lem: robustness tilted conditional} to conclude that there exists $n_0$ such that for any $n \geq n_0$, we have 
    \begin{align*}
        \max_{m' \leq \epsilon_n n + f_n} \omega(m')
        \;&<\;
        \frac{\delta}{2}
        \;.
    \end{align*}
    Using these arguments to bound \eqref{eq: rej prob diff}, we have shown that, for sufficiently large $n$,
    \begin{align*}
        \big| \Pr\nolimits_{\X_n \sim Q, \bW}( \cA(\X_n) ) - \Pr\nolimits_{\X_n^\ast \sim P}( \cA(\X_n^\ast) ) \big|
        \;\leq\;
        \frac{\delta}{2} + \frac{\delta}{2}
        \;=\;
        \delta
        \;.
    \end{align*}
    Taking the supremum over $Q \in \cP(P; \epsilon_n)$ then implies, for sufficiently large $n$,
    \begin{align*}
        \sup_{Q \in \cP(P; \epsilon_n)} \big| \Pr\nolimits_{\X_n \sim Q, \bW}\big( \Delta_\theta(\X_n) > q_{\infty, 1-\alpha}(\X_n) \big) - \Pr\nolimits_{\X_n^\ast \sim P, \bW}\big( \Delta_\theta(\X_n^\ast) > q_{\infty, 1-\alpha}(\X_n^\ast) \big) \big|
        \;\leq\;
        \delta
        \;.
    \end{align*}
    Since $\delta > 0$ was arbitrary, this shows the claim.
\end{proof}

\begin{proof}[Proof of \Cref{thm: robust tilted}]
    This result immediately follows from \Cref{prop: qualitative robust tilted general} by setting $\theta = 0$, in which case
    \begin{align*}
        &
        \sup\nolimits_{Q \in \cP(P; \epsilon_n)} \big| \Pr\nolimits_{\X_n \sim Q, \bW}\big( D^2(\X_n) > q^2_{\infty, 1-\alpha}(\X_n) \big) - \Pr\nolimits_{\X_n^\ast \sim P, \bW}\big( D^2(\X_n^\ast) > q^2_{\infty, 1-\alpha}(\X_n^\ast) \big) \big|
        \\
        \;&=\;
        \sup\nolimits_{Q \in \cP(P; \epsilon_n)} \big| \Pr\nolimits_{\X_n \sim Q, \bW}\big( D(\X_n) > q_{\infty, 1-\alpha}(\X_n) \big) - \Pr\nolimits_{\X_n^\ast \sim P, \bW}\big( D(\X_n^\ast) > q_{\infty, 1-\alpha}(\X_n^\ast) \big) \big|
        \\
        \;&=\;
        \sup\nolimits_{Q \in \cP(P; \epsilon_n)} \big| \Pr\nolimits_{\X_n \sim Q, \bW}\big( \Delta_\theta(\X_n) > q_{\infty, 1-\alpha}(\X_n) \big) - \Pr\nolimits_{\X_n^\ast \sim P, \bW}\big( \Delta_\theta(\X_n^\ast) > q_{\infty, 1-\alpha}(\X_n^\ast) \big) \big|
        \\
        \;&\to\;
        0
        \;,
    \end{align*}
    where the second line holds by taking the square-root on both sides of the inequalities, and the last equality holds since $D(\X_n) = \Delta_\theta(\X_n)$ when $\theta = 0$.
\end{proof}

\subsection{$P$-Targeted Kernel Stein Discrepancy}
\label{app: P-KSD}
This section states preliminary results which will be needed to prove \Cref{thm: bootstrap validity}. The main tool we use is the following generalized definition of KSD, which was originally proposed in \citet[Definition 4]{shi2024Finiteparticle}.
\begin{definition}[$P$-KSD]
\label{def: P-KSD}
    Given $P \in \cP(\R^d)$ with a Lebesgue density $p \in \cC^1$ and a reproducing kernel $k \in \cC^{(1,1)}_b$, the \emph{$P$-targeted (Langevin) kernel Stein discrepancy} ($P$-KSD) between two probability measures $Q, R \in \cP(\R^d)$ is defined as
    \begin{align}
        \S_P(Q, R)
        \;\coloneqq\;
        \sup\nolimits_{f \in \cB} \big| \E_{\bX \sim Q}[ (\cA_P f)(\bX) ] - \E_{\bY \sim R}[ (\cA_P f)(\bY) ] \big|
        \;,
        \label{eq: P-KSD}
    \end{align}
    where $\cB \coloneqq \big\{ h = (h_1, \ldots, h_d): \; h_j \in \cH_k \textrm{ and } \sum_{j=1}^d \| h_j\|_{\cH_k}^2 \leq 1\big\}$ is the unit ball in the $d$-times Cartesian product $\cH_k^d$ of the RKHS $\cH_k$ associated with $k$.
\end{definition}
When any input probability measure, say $Q$, is an empirical measure based on a sample $\X_n = \{\bx_i\}_{i=1}^n \subset \R^d$, we will abuse the notation by writing $\S_P(Q, R) = \S_P(\X_n, R)$ to emphasize the dependence on $\X_n$.

It can be shown that $P$-KSD is a \emph{Maximum Mean Discrepancy} \citep[MMD,][]{fortet1953convergence,muller1997integral} with the Stein reproducing kernel $u_p$ \citep[Theorem 1]{barp2022targeted}. 
When either of the two arguments coincides with $P$, the $P$-KSD reduces to the standard KSD, as we will show in the next section. The main benefit of working with $P$-KSD rather than KSD is that $P$-KSD satisfies both symmetry and a triangle inequality \Citep[Lemma 1]{shi2024Finiteparticle}, i.e., $\S_P(Q, R) = \S_P(R, Q)$ and $\S_P(Q, R) \leq \S_P(Q, R') + \S_P(R', R)$, for any $Q, R, R' \in \cP(\R^d)$ for which they are defined. These properties allow us to show several lemmas, which will be crucial in proving the validity of our robust-KSD tests in later sections.
\begin{itemize}
    \item \Cref{lem: KSD as PKSD} shows that the standard KSD can be written as a $P$-KSD.
    \item \Cref{lem: P-KSD closed form} shows that $\S_P(Q, R)$ is equivalent to an MMD, and that it admits a closed form involving expectations over $Q$ and $R$. 
    \item \Cref{lem: KSD inf closed form} shows that the $P$-KSD-projection of a probability measure onto a (standard) KSD-ball centered at $P$ admits a closed-form expression. 
    \item \Cref{lem: P-KSD deviation bound} restates \citet[Proposition A.1]{tolstikhin2017minimax}, which is a McDiarmid-type inequality for MMD. It also applies to $P$-KSD, as an $P$-KSD is also an MMD. We will use it in \Cref{app: KSD dev} to derive a robust-KSD test that is well-calibrated for finite samples. Other deviation bounds instead of the McDiarmid bound could also be used to construct a similar test;see \Cref{rem: other deviation bounds} for a discussion.
\end{itemize}
The MMD has already been studied extensively in the context of both Bayesian and frequentist robust parameter estimation; see, e.g., \citet{Briol2019,Cherief-Abdellatif2020_MMDBayes,Cherief-Abdellatif2019,Dellaporta2022,Dellaporta2023,Alquier2020}. A key assumption to ensure robustness in all of these papers is that the kernel is assumed to be bounded, and this is exactly what we are able to achieve with the KSD and $P$-KSD through \Cref{lem: bounded Stein kernel}.

\subsubsection{Properties of $P$-KSD}

\begin{lemma}[KSD as $P$-KSD]
\label{lem: KSD as PKSD}
    Assume $k \in \cC^{(1, 1)}_b$ and $\E_{\bX \sim P}[ \| \bs_p(\bX) \|_2 ] < \infty$. For any $Q \in \cP(\R^d)$ such that $D(Q, P) < \infty$, we have $\S_P(Q, P) = \S_P(P, Q) = D(Q, P)$.
\end{lemma}
\begin{proof}[Proof of \Cref{lem: KSD as PKSD}]
    This follows directly from the symmetry of $P$-KSD \citep[Lemma 1]{shi2024Finiteparticle} and that the assumed conditions imply $\E_{\bX \sim P}[ (\cA_P f)(\bX) ] = 0$ for all $f \in \cB$ \citep[Proposition 1]{gorham2017measuring}, so one of the two expectations in \eqref{eq: P-KSD} vanishes when either of the two arguments coincides with $P$.
\end{proof}

\citet[Theorem 1]{barp2022targeted} shows that a KSD is equivalent to an MMD with the Stein kernel $u_p$. Since $P$-KSD is a KSD by \Cref{lem: KSD as PKSD}, we can conclude that the $P$-KSD is also an MMD with kernel $u_p$, and that $P$-KSD has a double-expectation form similar to MMD.

\begin{lemma}[$P$-KSD closed-form]
    \label{lem: P-KSD closed form}
    Let $Q, R \in \cP(\R^d)$ and $k \in \cC^{(1, 1)}_b$. Denote by $\cH_u$ the RKHS associated with the Stein kernel $\cH_k$. If $\E_{\bX \sim Q}[ u_p(\bX, \bX) ] < \infty$ and $\E_{\bY \sim R}[ u_p(\bY, \bY) ] < \infty$, then the following two identities hold
    \begin{align*}
        \S_P^2(Q, R)
        \;&=\;
        \big\| \E_{\bX \sim Q}[u_p(\cdot, \bX)] - \E_{\bY \sim R}[u_p(\cdot, \bY)] \big\|^2_{\cH_u}
        \\
        \;&=\;
        \E_{\bX, \bX' \sim Q}[u_p (\bX, \bX')] + \E_{\bY, \bY' \sim R}[u_p (\bY, \bY')]
        - 2 \E_{\bX \sim Q, \bY \sim R}[u_p (\bX, \bY)]
        \;.
    \end{align*}
\end{lemma}
\begin{proof}[Proof of \Cref{lem: P-KSD closed form}]
    The assumed conditions ensure $\S_P^2(Q, R)$ is well-defined. Since $P$-KSD is an MMD with reproducing kernel $u_p$, we can apply \citet[Lemma 4]{gretton2012kernel} to conclude the first equality, and \citet[Lemma 6]{gretton2012kernel} to yield the second.
\end{proof}

\begin{lemma}[$P$-KSD projection]
    \label{lem: KSD inf closed form}
    Let $R \in \cP(\R^d)$ and assume $\E_{\bY \sim R}[u_p(\bY, \bY)] < \infty$. Further assume the conditions in \Cref{lem: KSD as PKSD} hold. For any $\theta > 0$,
    \begin{align*}
        \inf\nolimits_{Q \in \cB^\KSD(P; \theta)} \S_P(R, Q)
        \;=\;
        \max\big(0, \; \S_P(R, P) - \theta\big)
        \;=\;
        \max\big(0, \; \ksd(R, P) - \theta\big)
        \;.
    \end{align*}
\end{lemma}
\begin{proof}[Proof of \Cref{lem: KSD inf closed form}]
    The $P$-KSD $\S_P(R, P)$ is well-defined under the assumed conditions, and hence so is $\S_P\big( \lambda R + (1-\lambda) P , P\big)$ for any $\lambda \in [0, 1]$. Also since $P$-KSD is an MMD with (potentially unbounded) reproducing kernel $u_p$, the claim then follows by \citet[Proposition 3]{sun2023kernel} and noting that their proof extends directly to potentially unbounded kernels.
\end{proof}

\begin{lemma}[$P$-KSD deviation bound]
    \label{lem: P-KSD deviation bound}
    Assume $k$ is a tilted kernel satisfying the conditions in \Cref{lem: bounded Stein kernel} and $P \in \cP(\R^d)$ has a density $p \in \mathcal{C}^1$. For any random sample $\X_n$ drawn from $Q \in \cP(\R^d)$ and any $\alpha > 0$, 
    \begin{align*}
        \Pr\nolimits_{\X_n \sim Q}\left( \S_P(\X_n, Q)  >  \sqrt{\frac{\tau_\infty}{n}} + \sqrt{\frac{- 2 \tau_\infty \log \alpha}{ n}} \right)
        \;\leq\;
        \alpha
        \;.        
    \end{align*}
\end{lemma}

\begin{proof}[Proof of \Cref{lem: P-KSD deviation bound}]
    Let $Q_n$ denote the empirical distribution based on $\X_n$. By the first equality in \Cref{lem: P-KSD closed form}, we can write $\S_P(\X_n, Q) = \| \E_{\bX \sim Q_n}[ u_p(\cdot, \bX)] - \E_{\bX \sim Q}[ u_p(\cdot, \bX) ] \|_{\cH_u}$. Under the assumptions in \Cref{lem: bounded Stein kernel}, the function $u_p(\cdot, \bx): \R^d \to \cH_u$ is continuous for all $\bx \in \R^d$. Moreover, $\| u_p(\cdot, \bx) \|_{\cH_{u}} = u_p(\bx, \bx)$ by the reproducing property of the Stein kernel $u_p$ (see also the argument before \eqref{eq: stein kernel diag bound}), and $\sup_{\bx \in \R^d} u_p(\bx, \bx) < \infty$ again by \Cref{lem: bounded Stein kernel}. We can hence apply the McDiarmid-type inequality for MMD in \citet[Proposition A.1]{tolstikhin2017minimax} to conclude the claimed result.
\end{proof}

\subsection{Validity of the Bootstrap Approach}
\label{app: bootstrap validity}

We provide a proof for the validity of the bootstrap approach in \Cref{thm: bootstrap validity}. We first discuss the intuition, and then present two lemmas in \Cref{sec: bootstrap the P-KSD estimate}. The proof of \Cref{thm: bootstrap validity} will be presented in \Cref{sec: proof of thm: bootstrap validity}.

The bootstrap procedure used in robust-KSD is the same as the one in the standard KSD test, and it is not immediately obvious that this gives a valid decision threshold. For the standard KSD test, the V-statistic $D^2(\X_n)$ is \emph{degenerate} under the point null $H_0: Q = P$, i.e., $\E_{\bX \sim Q}[ u_p(\cdot, P) ] = 0$ when $Q = P$ \citep[Theorem 4.1]{liu2016kernelized}, and classic bootstrapping methods for degenerate V-statistics \citep{arcones1992bootstrap,huskova1993consistency} show that the bootstrap quantile $q_{\infty, 1-\alpha}$ is a valid decision threshold. However, the same argument cannot be applied directly to our robust test, because our null $\Hc_0$ is a composite one that contains not only $P$ but also other distributions $Q \neq P$, in which case $D^2(\X_n)$ is no longer degenerate as shown by \citet[Theorem 4.1]{liu2016kernelized}. Nevertheless, our proof shows that $q_{\infty, 1-\alpha}$ is still a correct decision threshold.

\paranoheading{Intuition of the proof}
Let $\X_n = \{ \bX_i \}_{i=1}^n$ be a random sample from some $Q \in \cP(\R^d)$. We will first show that the rejection probability of the robust-KSD test can be bounded as
\begin{align*}
    \Pr\nolimits_{\X_n \sim Q}\big( \Delta_\theta(\X_n) \;>\; q_{\infty, 1-\alpha}(\X_n) \big) 
    \;\leq\; 
    \Pr\nolimits_{\X_n \sim Q}\big( \S_P(\X_n, Q) \;>\; q_{\infty, 1-\alpha}(\X_n) \big)
    \;,
\end{align*}
where $\S_P(\X_n, Q)$ is the $P$-KSD introduced in \Cref{app: P-KSD}. We will then prove that the RHS of the above inequality converges to the prescribed test level $\alpha$ as $n \to \infty$ by showing that $q_{\infty, 1-\alpha}(\X_n)$ is a valid bootstrap approximation for the $(1-\alpha)$-quantile of the distribution of $\S_P(\X_n, Q)$. 

To see this, we first use the closed-form expression for $P$-KSD (\Cref{lem: P-KSD closed form}) to write
\begin{align*}
    \S_P^2(\X_n, Q)
    \;&=\;
    \E_{\bX, \bX' \sim Q_n}[ u_p(\bX, \bX') ] - 2 \E_{\bX \sim Q_n, \bY \sim Q}[ u_p(\bX, \bY) ] + \E_{\bY, \bY' \sim Q}[ u_p(\bY, \bY') ]
    \\
    \;&=\;
    \E_{\bX, \bX' \sim Q_n}[ u_p(\bX, \bX'; Q) ] 
    \\
    \;&=\;
    \mfrac{1}{n^2} \sum_{1 \leq i, j \leq n} u_p(\bX_i, \bX_j; Q)
    \tagaligneq \label{eq: PKSD as vstat}
    \;,
\end{align*}  
where we have defined
\begin{align*}
    u_p(\bx, \bx'; Q)
    \;\coloneqq\;
    u_p(\bx, \bx') - \E_{\bY \sim Q}[ u_p(\bx, \bY) ] - \E_{\bY' \sim Q}[ u_p(\bY', \bx') ] + \E_{\bY, \bY' \sim Q}[ u_p(\bY, \bY') ]
    \;.
\end{align*}
In other words, $\S_P^2(\X_n, Q)$ is a V-statistic with symmetric function $u_p(\cdot, \cdot; Q)$. 

Our key observation is that $\S_P^2(\X_n, Q)$ is the second-order term in the \emph{Hoeffding's decomposition} \citep[Eq.~3.2]{arcones1992bootstrap} of the V-statistic estimate $D^2(\X_n, P)$ for the standard KSD defined in \eqref{eq: KSD V stat}. This suggests that \eqref{eq: bootstrap sample} is an appropriate bootstrap sample. Crucially, this argument does \emph{not} require $Q = P$, unlike in the proof of the standard KSD test. We formalize this argument in the next section. 

A naive alternative approach to bootstrap $\S^2_P(\X_n, Q)$ is to note that it is degenerate for any $Q$ and use standard bootstrapping techniques for degenerate V-statistics \citep[e.g.,][Theorem 3.5]{arcones1992bootstrap}. However, this approach is not applicable here since, to compute the bootstrap sample, it requires the function $u_p(\bx, \bx'; Q)$ to be computable at any $\bx, \bx'$, which is not the case since $u_p(\bx, \bx'; Q)$ involves intractable expectations over $Q$ (recall that $u_p$ has mean zero under $P$, but not necessarily under $Q$).

\subsubsection{Bootstrapping the $P$-KSD Estimate}
\label{sec: bootstrap the P-KSD estimate}
Our proof for \Cref{thm: bootstrap validity} relies on \citet[Lemma 3.4]{arcones1992bootstrap}, which states that the asymptotic distribution of the second-order term in the Hoeffding's decomposition of a V-statistic can be approximated by a bootstrap distribution. We restate this result for the case of $\S^2_P(\X_n, Q)$ in \Cref{lem: arcones lemma 3.4}.
\begin{lemma}[\citet{arcones1992bootstrap}, Lemma 3.4]
\label{lem: arcones lemma 3.4}
    Let $\X_\infty = \{\bX_i\}_{i=1}^\infty$ be a random sample where $\bX_i \sim Q$ are independent, and for any $n$ define $\X_n \coloneqq \{\bX_i\}_{i=1}^n$. Let $Q_n$ be the empirical measure based on $\X_n$ and define $\X_n^\ast = \{ \bX_i^\ast \}_{i=1}^n$, where $\bX_i^\ast \sim Q_n$ are independent conditionally on $\X_n$. Assume $\E_{\bX \sim Q}[ u_p(\bX, \bX)^2 ] < \infty$. Then, as $n \to \infty$, the following holds almost surely,
    \begin{align*}
        \sup_{t \in \R} \big| \Pr\nolimits_{\bX_n^\ast \sim Q_n}\big(n \cdot \S_P^2(\X_n^\ast,\; \X_n) \;\leq\; t \;|\; \X_\infty \big) - \Pr\nolimits_{\X_n \sim Q} \big( n \cdot \S_P^2(\X_n,\; Q) \;\leq\; t \big) \big|
        \;\to\;
        0 \;.
    \end{align*}
\end{lemma}

Using this result, we can prove that a valid bootstrap approximation for the distribution of $\S_P(\X_n, Q)$ can be obtained by using the bootstrap sample $D_\bW(\X_n)$ defined in \eqref{eq: bootstrap sample}. This is summarized in the next lemma.

\begin{lemma}
\label{lem: bootstrap validity}
    Assume $\E_{\bX \sim Q}[ u_p(\bX, \bX)^2 ] < \infty$ and let $\bW \coloneqq (W_1, \ldots, W_n)$ be a random vector distributed as $\textrm{Multinomial}(n; 1/n, \ldots, 1/n)$. Under the notation in \Cref{lem: arcones lemma 3.4}, the following holds as $n \to \infty$,
    \begin{align}
        \sup_{t \in \R} \big| \Pr\nolimits_{\bW}\big(\sqrt{n} \cdot D_{\bW}(\X_n) \;\leq\; t \;|\; \X_\infty \big) - \Pr\nolimits_{\X_n \sim Q}\big( \sqrt{n} \cdot \S_P(\X_n, Q) \;\leq\; t \big) \big|
        \;\to\;
        0 \;.
        \label{eq: bootstrap validity}
    \end{align}
\end{lemma}

\begin{proof}[Proof of \Cref{lem: bootstrap validity}]
    As argued in \eqref{eq: PKSD as vstat}, the squared $P$-KSD $\S^2_P(\X_n, Q)$ can be written as a V-statistic with symmetric function $u_p(\cdot, \cdot; Q)$. Moreover, direct computation shows that $\E_{\bX \sim Q}[u_p(\cdot, \bX; Q)] \equiv 0$, so the symmetric function $u_p(\cdot, \cdot; Q)$ is \emph{$Q$-degenerate of order 1} \citep[pp.~5]{arcones1992bootstrap}. On the other hand, it is well-known that the weighted bootstrap form can be equivalently written as an Efron's resampled bootstrap statistic \citep{janssen1994weighted,dehling1994random,janssen1997bootstrapping}, i.e.,
    \begin{align*}
        D^2_{\bW}(\X_n)
        \;&=\;
        \frac{1}{n^2} \sum_{1 \leq i, j \leq n} (W_i - 1)(W_j - 1) u_p(\bX_i, \bX_j) 
        \\
        \;&=\;
        \frac{1}{n^2} \sum_{1 \leq i, j \leq n} W_iW_j u_p(\bX_i, \bX_j) - W_i u_p(\bX_i, \bX_j) - W_j u_p(\bX_i, \bX_j) + u_p(\bX_i, \bX_j)
        \\
        \;&\stackrel{d}{=}\;
        \frac{1}{n^2} \sum_{1 \leq i, j \leq n} u_p(\bX_i^\ast, \bX_j^\ast) - u_p(\bX_i^\ast, \bX_j) - u_p(\bX_i, \bX_j^\ast) + u_p(\bX_i, \bX_j)
        \\
        \;&=\;
        \frac{1}{n^2} \sum_{1 \leq i, j \leq n} u_p(\bX_i^\ast, \bX_j^\ast; Q_n) 
        \\
        \;&=\;
        \S_P^2(\X_n^\ast, \X_n)
        \;,
    \end{align*}
    where $\bX_i^\ast$ and $Q_n$ are defined in \Cref{lem: arcones lemma 3.4}, the notation $\stackrel{d}{=}$ denotes equality in distribution, the second last line follows by the definition of $u_p(\cdot, \cdot; Q_n)$, and the last line follows from \Cref{lem: P-KSD closed form}. Under the assumed moment condition, we can apply \Cref{lem: arcones lemma 3.4} together with the above derivation to conclude that a version of \eqref{eq: bootstrap validity} with the \emph{squared} statistics holds, i.e.,
    \begin{align*}
        \sup_{t \in \R} \big| \Pr\nolimits_{\bW}\big(n \cdot D^2_{\bW}(\X_n) \;\leq\; t \;|\; \X_\infty \big) - \Pr\nolimits_{\X_n \sim Q}\big( n \cdot \S^2_P(\X_n, Q) \;\leq\; t \big) \big|
        \;\to\;
        0 \;.
    \end{align*}
    The claim then follows by noting that the mapping $u \mapsto \sqrt{u}$ is everywhere continuous on $[0, \infty)$ and that weak convergence is preserved by continuous function by the Continuous Mapping Theorem \citep[Theorem 2.3]{vandervaart2000asymptotic}. 
\end{proof}

\subsubsection{Proof of \Cref{thm: bootstrap validity}}
\label{sec: proof of thm: bootstrap validity}
\begin{proof}[Proof of \Cref{thm: bootstrap validity}]
We first show that $\E_{\bX \sim Q}[ u_p(\bX, \bX)^2 ] < \infty$, so in particular $\E_{\bX \sim Q}[ u_p(\bX, \bX) ] < \infty$ and we can apply \Cref{lem: KSD inf closed form}. Defining $\bs_{p, w}(x) = w(x) \bs_p(x)$, we have 
\begin{align*}
    \E_{\bX \sim Q}[ u_p(\bX, \bX)^2 ]
    \;&\leq\;
    \E_{\bX \sim Q} \bigg[ \Big(\| \bs_{p, w}(\bx) \|_2^2 h(0) 
    + 2 \| \bs_{p, w}(\bx) \|_2 \| \nabla w(\bx) \|_2 h(0) 
    + \| \nabla w(\bx) \|_2^2 h(0) 
    \\
    &\;\qquad\qquad\quad
    + w(\bx)^2 \big| \nabla^\top \nabla h(0) \big| \Big)^2 \bigg] 
    \\
    \;&\leq\;
    4\Big( \E_{\bX \sim Q} \big[ \| \bs_{p, w}(\bx) \|_2^4 \big] h^2(0) 
    + 2 \E_{\bX \sim Q} \big[ \| \bs_{p, w}(\bx) \|_2^2 \| \nabla w(\bx) \|_2^2\big] h^2(0) 
    \\
    &\quad\quad\;\;
    + \E_{\bX \sim Q} \big[ \| \nabla w(\bx) \|_2^4\big] h^2(0) 
    + \E_{\bX \sim Q} \big[ w(\bx)^4\big] \big| \nabla^\top \nabla h(0) \big|^2 \Big)
    \;,
\end{align*}
where the second line follows from \eqref{eq: UB on stein kernel diag}, and the last line holds since $(a + b + c + d)^2 \leq 4(a^2 + b^2 + c^2 + d^2)$ for any $a, b, c, d \in \R$. The RHS of the last inequality is finite under the assumed conditions on $k$.

We are now ready to prove the first result of our theorem. For any $Q \in \cB^\KSD(P; \theta) \cap \cP(\R^d; w)$, we have
\begin{align*}
    &
    \Pr\nolimits_{\X_n \sim Q, \bW}\big( \Delta_\theta(\X_n) \;>\; q_{\infty, 1-\alpha}(\X_n) \big)
    \\
    \;&=\;
    \Pr\nolimits_{\X_n \sim Q, \bW}\big( \max(0, \ksd(\X_n) - \theta) \;>\; q_{\infty, 1-\alpha}(\X_n) \big)
    \\
    \;&=\;
    \Pr\nolimits_{\X_n \sim Q, \bW}\big( \inf\nolimits_{Q' \in \cB^\KSD(P; \theta)} \S_P(\X_n, Q') \;>\; q_{\infty, 1-\alpha}(\X_n) \big)
    \\
    \;&\leq\;
    \Pr\nolimits_{\X_n \sim Q, \bW}\big( \S_P(\X_n, Q) \;>\; q_{\infty, 1-\alpha}(\X_n) \big)
    \tagaligneq \label{eq: rej prob upper bound}
    \;,
\end{align*}
where the first equality holds by \Cref{lem: KSD inf closed form} and noting that $\ksd(\X_n)$ is equivalent to the KSD between the empirical measure based on $\X_n$ and $P$, and the last line holds since $Q \in \cB^\KSD(P; \theta)$. To show the first claim of our theorem, it suffices to show that the RHS of \eqref{eq: rej prob upper bound} converges to $\alpha$ as $n \to \infty$. Defining the following bootstrapping error
\begin{align*}
    \delta_n
    \;\coloneqq\;
    \Pr\nolimits_{\X_n \sim Q, \bW}\big( \S_P(\X_n, Q) \;>\; q_{\infty, 1-\alpha}(\X_n) \big) - \Pr\nolimits_{\bW}\big( D_{\bW}(\X_n) \;>\; q_{\infty, 1-\alpha}(\X_n) \;|\; \X_\infty \big)
    \;,
\end{align*}
we can write the RHS of \eqref{eq: rej prob upper bound} as
\begin{align*}
    \Pr\nolimits_{\X_n \sim Q, \bW}\big( \S_P(\X_n, Q) \;>\; q_{\infty, 1-\alpha}(\X_n) \big)
    \;&=\;
    \Pr\nolimits_{\bW}\big( D_{\bW}(\X_n) \;>\; q_{\infty, 1-\alpha}(\X_n) \;|\; \X_\infty \big) + \delta_n
    \\
    \;&=\;
    \alpha + \delta_n
    \tagaligneq \label{eq: type i and boot error}
    \;,
\end{align*}
where the last step holds since $q_{\infty, 1-\alpha}(\X_n)$ is the $(1-\alpha)$-quantile of the conditional distribution of $D_{\bW}(\X_n)$ given $\X_\infty$. Moreover, since we have shown that $\E_{\bX \sim Q}[ u_p(\bX, \bX)^2 ] < \infty$, we can apply \Cref{lem: bootstrap validity} to conclude that $\delta_n \to 0$ as $n \to \infty$. Hence, the probability on the LHS of the above equation converges to $\alpha$. Combining with \eqref{eq: rej prob upper bound}, we have proved that $\lim_{n \to \infty} \Pr\nolimits_{\X_n \sim Q, \bW}\big( \Delta_\theta(\X_n) \;>\; q_{\infty, 1-\alpha}(\X_n) \big) \leq \alpha$. Taking supremum over all $Q \in \cB^\KSD(P; \theta) \cap \cP(\R^d; w)$ shows the first claim of our theorem.

Now assume $Q \not\in \cB^\KSD(P; \theta)$ and $Q \in \cP(\R^d; w)$, so in particular $\theta - \ksd(Q, P) < 0$. We have
\begin{align*}
    &
    \Pr\nolimits_{\X_n \sim Q, \bW}\big( \Delta_\theta(\X_n) > q_{\infty, 1-\alpha}(\X_n))
    \\
    \;&=\;
    \Pr\nolimits_{\X_n \sim Q, \bW}\big( \ksd(\X_n) - \theta > q_{\infty, 1-\alpha}(\X_n) \big)
    \tagaligneq \label{eq: equiv test expression}
    \\
    \;&=\;
    \Pr\nolimits_{\X_n \sim Q, \bW}\Big( \sqrt{n} \big( \ksd^2(\X_n) - \ksd^2(Q, P) \big) > \sqrt{n} ( q_{\infty, 1-\alpha}(\X_n) + \theta)^2 - \sqrt{n} \ksd^2(Q, P) \Big)
    \tagaligneq \label{eq: bootstrap power}
    \;,
\end{align*}
where \eqref{eq: equiv test expression} holds since it can be checked that $\Delta_\theta(\X_n) > \gamma$ if and only if $D(\X_n) - \theta > \gamma$ for any $\gamma \geq 0$. The argument in \citet[Theorem 4.1]{liu2016kernelized} shows that  that $\sqrt{n} (\ksd^2(\X_n) - \ksd^2(Q, P) )$ converges weakly to a Gaussian limit assuming $\E_{\bX \sim Q}[u_p(\bX, \bX')^2] < \infty$, which holds since \eqref{eq: stein kernel diag bound} and Jensen's inequality imply
\begin{align*}
    \E_{\bX, \bX' \sim Q}[u_p(\bX, \bX')^2] 
    \;\leq\;
    \E_{\bX, \bX' \sim Q}[u_p(\bX, \bX) u_p(\bX', \bX')] 
    \;&=\;
    \big( \E_{\bX\sim Q}[u_p(\bX, \bX)] \big)^2
    \\
    \;&\leq\;
    \E_{\bX\sim Q}[u_p(\bX, \bX)^2]
    \;,
\end{align*}
which is finite as shown before. On the other hand, the weak convergence of $\sqrt{n} (\ksd^2(\X_n) - \ksd^2(Q, P) )$ also implies $q_{\infty, 1-\alpha}(\X_n) \to 0$ as $n \to \infty$, so when $D(Q, P) > \theta$,
\begin{align*}
    \lim_{n\to\infty} ( q_{\infty, 1-\alpha}(\X_n) + \theta)^2 - \ksd^2(Q, P)
    \;=\;
    \theta^2 - \ksd^2(Q, P)
    \;<\;
    0
    \;.
\end{align*}
Combing these arguments shows that in the probability in \eqref{eq: bootstrap power}, the LHS converges to a non-degenerate distribution, while the RHS tends to $-\infty$. This implies \eqref{eq: bootstrap power} converges to 1, thus proving the second claim.
\end{proof}

\begin{remark}
\label{rem: uniform bound}
    If the bootstrapped quantile $q_{\infty, 1-\alpha}(\X_n)$ is replaced by the quantile $q^\ast_{1-\alpha}(\X_n)$ of the distribution of the $P$-KSD $\S_P(\X_n, Q)$, then the bootstrap approximation error in \eqref{eq: type i and boot error} vanishes, i.e., $\delta_n = 0$ for all $n$. In that case, we have the following stronger, uniform level control
        \begin{align*}
            \limsup_{n \to \infty} \sup_{Q \in \cB^\KSD(P; \theta) \cap \cP(\R^d; w)}\Pr\nolimits_{\X_n \sim Q, \bW}\big( \Delta_\theta(\X_n) \;>\; q^\ast_{1-\alpha}(\X_n) \big)
            \;\leq\;
            \alpha \;.
        \end{align*}
        In fact, the inequality holds for any finite $n$, since the RHS of \eqref{eq: type i and boot error} equals to $\alpha$ for all $n$.
\end{remark}

\subsection{Connections Between KSD Balls and Contamination Models}
\label{app: KSD balls and contam models}
This section provides proofs for the results in \Cref{sec: choosing uncertainty radius}. 
\begin{itemize}
    \item \Cref{sec: proof of KSD decomposition} states and proves \Cref{lem: KSD decomposition}, an intermediary result which provides a decomposition of the KSD under Huber's contamination models.
    \item \Cref{sec: proof of contam models and KSD-ball} proves \Cref{prop: KSD bound huber model} using \Cref{lem: KSD decomposition}.
    \item \Cref{sec: pf of KSD bound density band} proves \Cref{prop: KSD bound density band}.
    \item \Cref{app: KSDball_fat_tails} provides an example of using \Cref{prop: KSD bound density band} to bound the KSD between t- and Gaussian distributions.
\end{itemize}

\subsubsection{Statement and Proof of \Cref{lem: KSD decomposition}}
\label{sec: proof of KSD decomposition}

\begin{lemma}
    \label{lem: KSD decomposition}
    Let $\epsilon \in [0, 1]$ and $Q = (1 - \epsilon) P + \epsilon R$, for any $R \in \cP(\R^d)$ such that $\E_{\bY \sim R}[ u_p(\bY, \bY)^{1/2} ] < \infty$. Assume $k \in \cC^{(1,1)}_b$ and $\E_{\bX \sim P}[ \| \bs_p(\bX) \|_2 ] < \infty$. Then $\ksd(Q, P) = \epsilon \ksd(R, P)$.
\end{lemma}

\begin{proof}[Proof of \Cref{lem: KSD decomposition}]
The assumption $\E_{\bY \sim R}[ u_p(\bY, \bY)^{1/2} ] < \infty$ implies that $D(R, P)$ is well-defined by \citet[Proposition 2]{gorham2017measuring}. Moreover, when $k \in \cC^{(1,1)}_b$ and $\E_{\bX \sim P}[ \| \bs_p(\bX) \|_2 ] < \infty$, \citet[Proposition 1]{gorham2015measuring} shows that $\E_{\bX \sim P}[ u_p(\bX, \cdot) ] = 0$. Using this and the linearity of expectation,
\begin{align*}
    \ksd^2(Q, P)
    \;&=\;
    (1 - \epsilon)^2 \E_{\bX, \bX' \sim P}[ u_p(\bX, \bX') ]
    + 2 (1 - \epsilon) \epsilon \E_{\bX \sim P, \bY \sim R}[ u_p(\bX, \bY) ]
    \\
    &\;\quad
    + \epsilon^2 \E_{\bY, \bY' \sim R}[ u_p(\bY, \bY') ]
    \\
    \;&=\;
    (1 - \epsilon)^2 \ksd^2(P, P)
    + \epsilon^2 \ksd^2(R, P)
    \\
    \;&=\;
    \epsilon^2 \ksd^2(R, P)
    \;.
\end{align*}
Taking square-root of both sides gives the desired result.
\end{proof}

\subsubsection{Proof of \Cref{prop: KSD bound huber model}}
\label{sec: proof of contam models and KSD-ball}

\begin{proof}
The proof follows a similar approach as \citet[Lemma 3.3]{Cherief-Abdellatif2019}. Under the assumed kernel conditions, \Cref{lem: bounded Stein kernel} shows that $\sup_{\bx, \bx' \in \mathbb{R}^d} u_p(\bx, \bx') \leq \tau_\infty < \infty$, which implies $\ksd^2(Q, P) = \E_{\bX,\bX' \sim Q}[ u_p(\bX,\bX') ] \leq \tau_\infty < \infty$ for any probability measure $Q \in \cP(\R^d)$. In particular, $D(Q, P)$ is well-defined for any $Q \in \cP(\R^d)$ and the conditions in \Cref{lem: KSD decomposition} are met. For any $Q = (1 - \epsilon) P + \epsilon R $ with $Q \in \cP(P; \epsilon_0)$, applying \Cref{lem: KSD decomposition} shows that 
\begin{align}
    \ksd(Q, P) \;=\; \epsilon \ksd(R, P) \;\leq\; \epsilon_0 \ksd(R, P) 
    \;\leq\; \epsilon_0 \tau_\infty^{\frac{1}{2}}
    \label{eq: KSD contam model upper bound}
    \;.
\end{align}
To show that this bound is tight, we can lower bound the KSD as
\begin{align*}
    \sup_{Q \in \cP(P; \epsilon_0)} \ksd(Q, P) 
    \;\geq\;
    \sup_{\bz \in \mathbb{R}^d, \epsilon \in [0,\epsilon_0] } \ksd\big( (1-\epsilon)P + \epsilon \delta_\bz, P \big) 
    \;&=\; 
    \sup_{\bz \in \mathbb{R}^d, \epsilon \in [0,\epsilon_0] }  \epsilon \ksd(\delta_\bz, P) 
    \\
    \;&=\; 
    \epsilon_0 \sup_{\bz \in \mathbb{R}^d} u_p(\bz, \bz)^{\frac{1}{2}}
    \\
    \;&=\;
    \epsilon_0 \tau_\infty^{\frac{1}{2}}
    \;,
\end{align*}
where the first inequality holds since $R = \delta_\bz$ is only one possible type of perturbation, the first equality follows from \Cref{lem: KSD decomposition}, the second equality holds because the supremum over $\epsilon$ is reached at $\epsilon_0$, and the last step holds as $\sup_{\bz \in \mathbb{R}^d} u_p(\bz, \bz) = \tau_\infty$ by definition. Combining with the upper bound \eqref{eq: KSD contam model upper bound}, we have shown that $\sup_{Q \in \cP(P; \epsilon_0)} \ksd(Q, P) = \epsilon_0 \tau_\infty^{1/2}$.

\subsubsection{Proof of \Cref{prop: KSD bound density band}}
\label{sec: pf of KSD bound density band}

Under the assumed conditions on $k$, \Cref{lem: bounded Stein kernel} shows that the Stein kernel is bounded by $\tau_\infty = \sup_{\bx \in \mathbb{R}^d} u_p(\bx, \bx) < \infty$, so in particular $D(Q, P)$ is well-defined for all $Q \in \cP(\R^d)$. Under the assumed integrability condition, we have $\E_{\bX \sim P}[u_p(\bX, \cdot)] = 0$ as argued in the paragraph before \eqref{eq: var bound}. We can therefore rewrite the squared KSD as
\begin{align*}
    D^2(Q, P)
    \;&=\;
    \int_{\R^d}\int_{\R^d} u_p(\bx, \bx') q(\bx) q(\bx') \diff\bx \diff\bx'
    \\
    \;&=\;
    \int_{\R^d}\int_{\R^d} u_p(\bx, \bx') (q(\bx) - p(\bx)) (q(\bx') - p(\bx')) \diff\bx \diff\bx'
    \;,
\end{align*}
Using the assumed bound $| q(\bx) - p(\bx) | \leq \delta(\bx)$ and \Cref{lem: bounded Stein kernel}, we can bound the RHS of the above line by
\begin{align*}
    \int_{\R^d}\int_{\R^d} u_p(\bx, \bx') \delta(\bx) \delta(\bx') \diff\bx \diff\bx'
    \;&\leq\;
    \int_{\R^d}\int_{\R^d} | u_p(\bx, \bx') | \delta(\bx) \delta(\bx') \diff\bx \diff\bx'
    \\
    \;&\leq\;
    \tau_\infty \bigg( \int_{\R^d} \delta(\bx) \diff\bx \bigg)^2
    \;=\;
    \tau_\infty \delta_0^2
    \;.
\end{align*}
Taking the square-root of both sides implies $D(Q, P) \leq  \tau_\infty^{1/2} \delta_0$.
\end{proof}

\subsubsection{KSD Balls and Fat Tails}
\label{app: KSDball_fat_tails}

The next proposition serves as an example for how \Cref{prop: KSD bound density band} can be applied to design a robust KSD test when the model is Gaussian but data are drawn from t-distributions that are moment-matched to the model.

\begin{proposition}
    \label{prop: KSD ball fat tails}
    Let $P = \cN(0, 1)$ and let $Q_\nu = t_\nu \sqrt{(\nu - 2) / \nu}$, where $t_\nu$ is the t-distribution with degree-of-freedom (dof) $\nu > 2$. Denote by $p$ and $q_\nu$ their probability density functions, and by $F_\infty$ and $F_\nu$ their cumulative distribution functions, respectively. Then $p(x)$ and $q_\nu(x)$ have exactly two intersections $a_1 < a_2$ on $(0, \infty)$. 
    If furthermore $k$ satisfies the conditions in \Cref{lem: bounded Stein kernel}, then $\KSD(Q_\nu, P) \leq 4 \tau_\infty^{1/2} ( F_\nu(a_1) - F_\infty(a_1) + F_\infty(a_2) - F_\nu(a_2) )$. 
\end{proposition}
This result suggests that, to ensure quantitative robustness against the scaled t-distribution with $\nu$ degrees-of-freedom, we can choose the uncertainty radius in the robust-KSD test to be $\theta = 4 \tau_\infty^{1/2} ( F_\nu(a_1) - F_\infty(a_1) + F_\infty(a_2) - F_\nu(a_2) )$. The intersection points $a_1, a_2$ do not have a closed form; instead, we approximate them by numerical solvers, which is trivial for this one-dimensional problem, and the numerical error is negligible as evidenced empirically in \Cref{fig: heavy tail}.

\begin{proof}[Proof of \Cref{prop: KSD ball fat tails}]
    For $\nu > 2$, the change-of-variable formula shows that the scaled t-distribution $Q_\nu$ has density function $q(x) = Z_\nu q_\nu^\ast(x)$, where $Z_\nu \coloneqq \Gamma(\frac{\nu + 1}{2})/\big(\sqrt{\pi (\nu-2)} \Gamma(\frac{\nu}{2}) \big)$, $q_\nu^\ast(x) \coloneqq (1 + \frac{x^2}{\nu-2})^{- (\nu + 1)/2}$, and $\Gamma(\cdot)$ is the Gamma function \citep[see, e.g.,][Chapter 8.1]{forbes2011statistical}. 
    To show $q_\nu(x)$ intersects with $p(x)$ at exactly two points on $(0, \infty)$, we first show that $q_\nu(x)$ is \emph{strong super-Gaussian} \citep{palmer2010strong}, i.e., $x \mapsto \log q_\nu(\sqrt{x})$ is convex on $[0, \infty)$. This is immediate by writing
    \begin{align*}
        \log q_\nu(\sqrt{x}) \;=\; \log Z_\nu -\frac{\nu + 1}{2} \log\left( 1 + \frac{x}{\nu - 2} \right) \;,
    \end{align*}
    and noting that $u \mapsto \log(1 + u)$ is concave on $[0, \infty)$. \citet[Theorem 2]{palmer2010strong} shows that any symmetric and strongly super-Gaussian density belongs to the class of \emph{density cross inequalities} ($DC_+$), which are symmetric densities that cross a Gaussian density of equal variance exactly four times on $\R$, and take higher values than that normal density at $x=0$ and as $x \to \infty$.
    Since in this case both $Q_\nu$ and $P$ have unit variance, we conclude that $q_\nu$ and $p$ intersect at exactly four points on $\R$. By symmetry of these densities about $x=0$, this shows that $q_\nu$ and $p$ must intersect at exactly \emph{two} points on $(0, \infty)$.
    
    Call these two intersections $a_1, a_2$ and assume without loss of generality $0 < a_1 < a_2$. It then follows that $q_\nu(x) \geq p(x)$ on $[0, a_1) \cup [a_2, \infty)$ and $q_\nu(x) \leq p(x)$ on $[a_1, a_2)$.
    We therefore have
    \begin{align*}
        \delta_0
        \;&\coloneqq\;
        \int_\R | p(x) - q_\nu(x) | \diff x
        \\
        \;&=\;
        2\int_{0}^{a_1} ( q_\nu(x) - p(x) ) \diff x
        + 2\int_{a_1}^{a_2} ( p(x) - q_\nu(x) ) \diff x
        + 2\int_{a_2}^{\infty} ( q_\nu(x) - p(x) ) \diff x
        \\
        \;&=\;
        2 (F_\nu(a_1) - F_\infty(a_1)) 
        + 2\big(F_\infty(a_2) - F_\infty(a_1) - (F_\nu(a_2) - F_\nu(a_1) ) \big)
        \\ &\;\quad
        + 2 \big((1-F_\nu(a_2)) - (1-F_\infty(a_2))\big)
        \\
        \;&=\;
        2 (F_\nu(a_1) - F_\infty(a_1) + F_\infty(a_2) - F_\infty(a_1) - F_\nu(a_2) + F_\nu(a_1)
        + F_\infty(a_2) - F_\nu(a_2)
        ) 
        \\
        \;&=\;
        4 ( F_\nu(a_1) - F_\infty(a_1) + F_\infty(a_2) - F_\nu(a_2) )
        \;,
    \end{align*}
    where in the second line we have again used the symmetry of normal and t densities. Applying \Cref{prop: KSD bound density band} completes the proof.
\end{proof}

\section{Complementary Experiments}
\label{app: supp experiments}
This section includes additional experimental results. \Cref{app: rate contam ratio} discusses the rate requirement in \Cref{thm: robust tilted} and demonstrates that the KSD test is no longer qualitatively robust when this condition is not met. \Cref{app: mmd tests} reviews the robust MMD tests included in \Cref{sec: contam gaussian} and presents implementation details. \Cref{app: bw robust} includes an ablation study on the choice of kernel bandwidths in the robust-KSD test. \Cref{app: ms experiment} studies its scalability with dimension. \Cref{app: efron vs wild} compares robust-KSD using two different bootstrap methods: the weighted bootstrap used throughout this work, and a wild bootstrap that is more prevalent in the kernel testing literature.

\subsection{Decay Rate of Contamination Ratio}
\label{app: rate contam ratio}
We show empirically that the rate requirement in the qualitative robustness result \Cref{thm: robust tilted} is not an artifact of the proof, i.e., the tilted-KSD test is no longer qualitatively robust to Huber's contamination models with contamination ratio $\epsilon_n$ if $\epsilon_n = n^{-r}$ for any $r \leq 1/2$, where $n$ is the sample size. 
We run the standard KSD test with a tilted kernel under the contaminated Gaussian model as described in \Cref{sec: contam gaussian} with dimension $d=1$ and outlier $z = 10$. The contamination ratio is chosen to be $\epsilon_n = n^{-r}$ for different choices of $r$, and the probability of rejection as $n$ grows is shown in \Cref{fig: ol contam rate}. The results are averaged over $400$ repetitions instead of 100 repetitions as in \Cref{sec: contam gaussian} to reduce numerical uncertainty. When $r > 0.5$, the rejection probability converges to the prescribed test level $\alpha$, which is the limit of the rejection probability \emph{without} contamination. This aligns with \Cref{thm: robust tilted}. However, this rejection probability no longer converges to $\alpha$ when $r \leq 0.5$, thus showing that the tilted-KSD test is no longer qualitatively robust.

\begin{figure}[t]
    \begin{minipage}[t]{0.47\textwidth}
        \centering
        \includegraphics[width=1\textwidth]{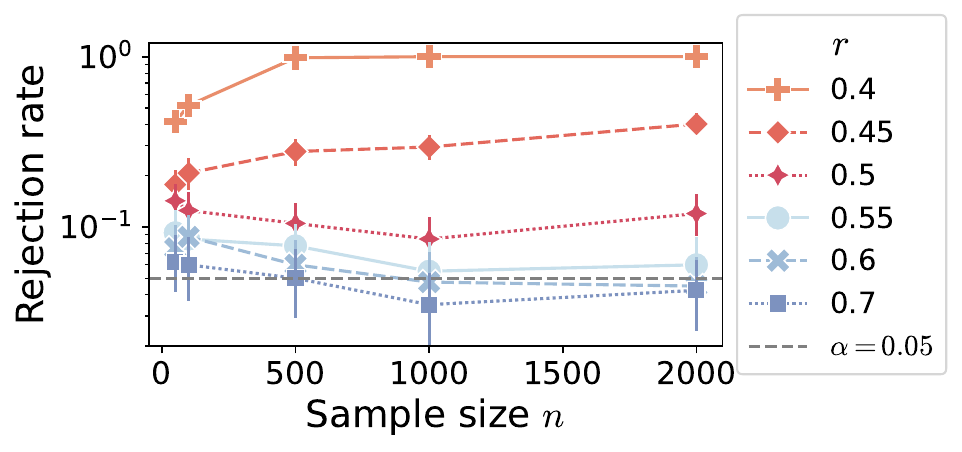}
        \caption{Rejection probability (in log scale) against sample size under the univariate contaminated Gaussian model with outlier $z = 10$. The contamination ratio scales as $\epsilon_n = n^{-r}$ for different values of $r$.}
        \label{fig: ol contam rate}
    \end{minipage}
    \hfill
    \begin{minipage}[t]{0.50\textwidth}
        \centering
        \includegraphics[width=.86\textwidth]{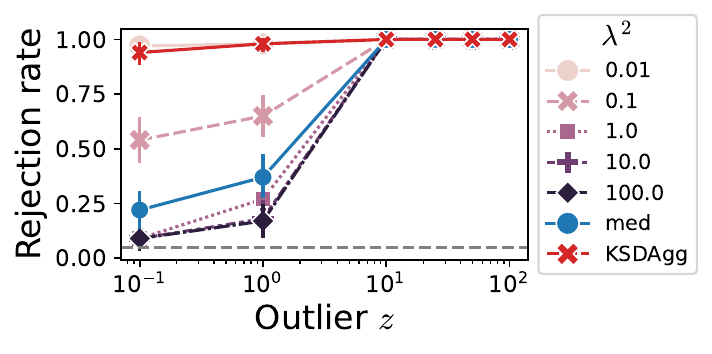}
        \caption{Rejection probability for an IMQ kernel with bandwidth $\lambda$. ``med'' is the median heuristic. ``KSDAgg'' is the test of \citet{schrab2022ksd}. The dashed line shows $\alpha = 0.05$. All tests all reject $H_0$ for large $z$.}
        \label{fig: bw standard}
    \end{minipage}
\end{figure}

\subsection{Implementation Details for the Robust MMD Tests}
\label{app: mmd tests}
This section provides implementation details for the two robust MMD-based tests included in \Cref{sec: contam gaussian}. Both tests use the Maximum Mean Discrepancy \citep[MMD;][Definition 2]{gretton2012kernel} between the data-generating distribution $Q$ and the model $P$ as a measure of their disparity. Given a reproducing kernel $k: \R^d \times \R^d \to \R$, the squared MMD between $Q$ and $P$ are defined as
\begin{align*}
    D_\MMD^2(Q, P)
    \;=\;
    \E_{\bX, \bX' \sim Q}[ k(\bX, \bX') ] + \E_{\bY, \bY' \sim P}[ k(\bY, \bY') ] - 2\E_{\bX \sim Q, \bY \sim P}[k(\bX, \bY)]
    \;.
\end{align*}
Given independent samples $\X_n = \{\bX_i\}_{i=1}^n$ from $Q$ and $\Y_m = \{\bY_j\}_{j=1}^m$ from $P$, the MMD can be estimated directly by $D_\MMD^2(Q_n, P_m)$, where $Q_n$ and $P_m$ are the empirical measures formed by $\X_n$ and $\Y_m$, respectively \citep{gretton2012kernel}. The computational cost of such estimate scales with sample size $n$ as $\cO((n+m)^2 + m C_{\text{sim}})$, where $C_{\text{sim}}$ is the cost of simulating one datum from $P$, compared with $\cO(n^2 + n C_{\text{score}})$ for KSD, where $C_{\text{score}}$ denotes the cost of a single score evaluation. In our experiment in \Cref{sec: contam gaussian}, $P$ is a Gaussian model, so both simulation and score evaluation are trivial and $C_{\text{sim}}, C_{\text{score}}$ are both small. Hence, we set $m = n$ so that the costs of MMD and of KSD are comparable. Moreover, we will use the squared-exponential kernel $k(\bx, \bx') = \exp(- \| \bx - \bx' \|_2^2 / (2\gamma^2))$ and set $\gamma^2$ using the median heuristic. Squared-exponential kernels are popular for MMDs and are also used in the original papers that proposed the two robust MMD tests \citep{sun2023kernel,schrab2024robust}. 

\subsubsection{The MMD-Dev Test}
The robust MMD test of \citet[Eq.~38]{sun2023kernel}, which we refer to as \emph{MMD-Dev}, targets the hypotheses
\begin{align}
    H_0^\MMD: \; Q \in \cB^\MMD(P; \theta_\MMD)
    \;,\qquad
    H_1^\MMD: \; Q \not\in \cB^\MMD(P; \theta_\MMD)
    \;,
    \label{eq: MMD-ball hypotheses}
\end{align}
where $\theta_\MMD \geq 0$ and $\cB^\MMD(P; \theta_\MMD) = \{Q \in \cP:  D_\MMD(Q, P) \leq \theta_\MMD \}$ is the MMD ball centered at $P$ and with radius $\theta_\MMD$. Given a test level $\alpha \in (0, 1)$, MMD-Dev rejects $H_0^\MMD$ if $D_\MMD(Q_n, P_n) > \theta_\MMD + \gamma_n$, where $\gamma_n = \sqrt{2K / n} (1 + \sqrt{- \log \alpha})$. \citet[Theorem 4]{sun2023kernel} shows that this test controls the Type-I error for any finite $n$, and is asymptotically optimal against certain alternatives. The decision threshold $\gamma_n$ is derived using a McDiarmid-type deviation inequality for MMDs due to \citet[Theorem 8]{gretton2012kernel}.

In our experiments, we chose $\theta_\mathrm{MMD} = \epsilon_0 \sqrt{2}$. This ensures that $\cB^\MMD(P; \theta_\MMD)$ contains Huber's contamination models with contamination ratios up to $\epsilon_0$, which we show in the next result. Notably, although we can compare the MMD-Dev test with our robust-KSD test on Huber's contamination models, they are in general \emph{not} directly comparable, since they target different null hypotheses.

\begin{lemma}[MMD balls and Huber's model]
\label{lem: MMD ball Huber model}
    Let $k$ be a reproducing kernel with $0 < k(\bx, \bx') \leq K < \infty$, for all $\bx, \bx' \in \R^d$. For any $\epsilon_0 \in [0, 1]$, if $\theta_\MMD = \epsilon_0 \sqrt{2K}$, then $\cP(P; \epsilon_0) \subset \cB^\MMD(P; \theta_\MMD)$, where $\cP(P; \epsilon_0)$ is the Huber's contamination model defined in \eqref{eq: Huber model}.
\end{lemma}

\begin{proof}[Proof of \Cref{lem: MMD ball Huber model}]
    For any $Q \in \cP(\R^d)$, we define its kernel mean embedding \citep[Chapter 4]{berlinet2004reproducing} as $\mu_Q(\cdot) \coloneqq \E_{\bX \sim Q}[k(\bX, \cdot)]$, which is well-defined since $k$ is bounded. By \citet[Theorem 1]{sriperumbudur2010hilbert}, the MMD between any $Q, P \in \cP(\R^d)$ can be equivalently written as $D_\MMD(Q, P) = \| \mu_Q - \mu_P \|_{\cH_k}$, where $\cH_k$ is the RKHS associated with $k$. Pick $\epsilon_0 \in [0, 1]$ and $R \in \cP$. For any $\epsilon \in [0, \epsilon_0]$, 
    \begin{align*}
        D_\MMD((1-\epsilon)P + \epsilon R, P)
        \;=\;
        \| \mu_{(1-\epsilon)P + \epsilon R} - \mu_P \|_{\cH_k}
        \;&=\;
        \| (1-\epsilon)\mu_P + \epsilon \mu_R - \mu_P \|_{\cH_k}
        \\
        \;&=\;
        \epsilon \| \mu_R - \mu_P \|_{\cH_k}
        \\
        \;&=\;
        \epsilon D_\MMD(R, P)
        \;,
    \end{align*}
    where the second equality holds due to the linearity of the expectation operator. Moreover, 
    \begin{align*}
        D_\MMD^2(R, P)
        \;&=\;
        \E_{\bX, \bX' \sim R}[ k(\bX, \bX') ] + \E_{\bY, \bY' \sim P}[ k(\bY, \bY') ] - 2\E_{\bX \sim R, \bY \sim P}[k(\bX, \bY)]
        \\
        \;&\leq\;
        \E_{\bX, \bX' \sim R}[ k(\bX, \bX') ] + \E_{\bY, \bY' \sim P}[ k(\bY, \bY') ] 
        \\
        \;&\leq\;
        2K
        \;,
    \end{align*}
    where the last line holds since $\sup_{\bx, \bx' \in \R^d} k(\bx, \bx') \leq K$ by assumption. This shows that $D_\MMD((1-\epsilon)P + \epsilon R, P) \leq \epsilon \sqrt{2K} \leq \epsilon_0 \sqrt{2K}$, so the claim holds.
\end{proof}

\subsubsection{The dcMMD Test}
The dcMMD test \citep{schrab2024robust} targets different hypotheses. Given $\epsilon_0 \in [0, 1]$, the null assumes that the observed sample $\X_n$ is generated by firstly drawing i.i.d.\ random variables from some probability measure, then corrupting them by replacing at most a proportion of $\epsilon_0$ of the sample by arbitrary values. The hypotheses can be formalized as follows:
\begin{align*}
    & H_0^\mathrm{dcMMD}: \textrm{ At least $(1-\epsilon_0) \times n$ random variables in $\X_n$ are i.i.d.\ from $P$.} 
    \\
    & H_1^\mathrm{dcMMD}: \textrm{ Otherwise}
    \;.
\end{align*}
The dcMMD test rejects $H_0^\mathrm{dcMMD}$ if $D_\MMD(Q_n, P_n) > q_{\alpha}^\MMD + 2\epsilon_0\sqrt{2K}$, where $q_{\alpha}^\MMD$ is the empirical quantile of $B$ permutation samples $\{ T_{\MMD}^b \}_{b=1}^B$, and each $T_{\MMD}^b$ is computed by \emph{(i)} randomly permuting $\X_n \cup \Y_n$, \emph{(ii)} partitioning the permuted set into two subsets $\X_n^b$ and $\Y_n^b$ of size $n$, and \emph{(iii)} computing $T_{\MMD}^b \coloneqq D_\MMD(Q_n^b, P_n^b)$, where $Q_n^b$ and $P_n^b$ are the empirical measures based on $\X_n^b$ and $\Y_n^b$, respectively. We use $B=500$, which is the default setup in \citet{schrab2024robust}. 

The dcMMD test gives stronger calibration guarantee than Huber's contamination model because it controls contamination proportion no larger than $\epsilon_0$ for \emph{any realization} of the sample $\X_n$, rather than in expectation as required by Huber's model. However, it is not necessarily stronger than KSD-balls or MMD-balls, since it does not account for the case where \emph{all} samples are under mild perturbation. 

\begin{figure}
    \centering
    \includegraphics[width=0.8\linewidth]{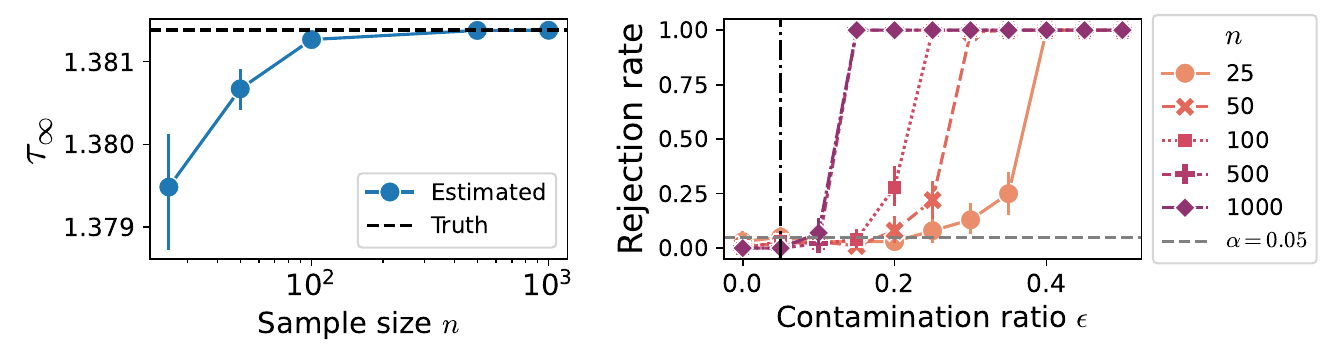}
    \caption{\emph{Left.} Estimated $\tau_\infty$ using the trick described in \Cref{sec: choosing uncertainty radius} compared with the ground-truth; the estimate becomes more accurate with larger samples. \emph{Right.} Rejection probability of robust-KSD under an outlier-contaminated Gaussian model; the Type-I error remains calibrated even for small sample sizes.}
    \label{fig: tau_infty}
\end{figure}

\subsection{Approximating the Supremum of the Stein Kernel}
\label{app: tau_infty}
As described in \Cref{sec: choosing uncertainty radius}, in our experiments we approximated the intractable supremum $\tau_\infty = \sup_{\bx \in \R^d} u_p(\bx, \bx)$ by $\max_{i = 1, \ldots, n} u_p(\bx_i, \bx_i)$. We now study how this approximation affects the performance of the robust-KSD test. We run the same contaminated Gaussian experiment in \Cref{sec: contam gaussian} with $d=1$ and outlier location $z = 10$. For this model, $\bx \mapsto u_p(\bx, \bx)$ is a simple, univariate function, so we can compute the ground-truth $\tau_\infty$ accurately by numerical optimization. Clearly, the ground truth $\tau_\infty$ is equivalent to taking $n = \infty$. This is evidenced by the left plot in \Cref{fig: tau_infty}, which shows that the finite-sample approximation $\max_{i = 1, \ldots, n} u_p(\bx_i, \bx_i)$ becomes more accurate as $n$ increases. The right plot of \Cref{fig: tau_infty} shows the rejection probability of robust-KSD with $\theta$ set to control at most $\epsilon_0 = 0.05$ proportion of contamination. Remarkably, even when this approximation under-estimates $\tau_\infty$ (which is the case for small $n$), robust-KSD still remains well-calibrated. This is not surprising: the supremum $\tau_\infty$ represents the maximal contribution of a single datum $\bx$ taking arbitrary values in the \emph{entire} sample space $\R^d$, regardless of whether it is present in the \emph{observed} data set. However, in practice, it suffices to control the contribution of any single datum in the observed data set.

\subsection{Ablation Study for Kernel Bandwidth in the Standard KSD Test}
\label{app: bw robust}
We evaluate how the choice of the kernel bandwidth affects the performance of the \emph{standard} KSD test. We run the same contaminated Gaussian experiment as in the previous subsection. The standard KSD test uses an IMQ kernel with bandwidth $\lambda^2 \in \Lambda \cup \{ \lambda_\mathrm{med}^2 \}$, where $\Lambda = \{ 0.01, 0.1, 1, 10, 100\}$ and $\lambda_\mathrm{med}^2$ is the bandwidth chosen by median heuristic. We also include KSDAgg, which also uses an IMQ kernel for a fair comparison. As shown in \Cref{fig: bw standard}, all standard IMQ-KSD tests reject the point null with high probability for large outlier values, \emph{regardless} of the bandwidth value. This suggests that no fixed bandwidth can ensure robustness, which is consistent with \Cref{thm: non robust stationary}, and is because the Stein kernel is unbounded regardless of the choice of $\lambda$, thus the non-robustness issue persists. The KSDAgg test is even more sensitive to contamination than the IMQ-KSD tests. As noted in \Cref{sec: contam gaussian}, this is because KSDAgg is designed to optimally combine multiple bandwidths to boost the test power against all alternatives, including contaminated models.

\subsection{Gaussian Mean-Shift Experiment}
\label{app: ms fixed dim}
We run a Gaussian mean-shift experiment to demonstrate the performance of our robust test against model deviations other than contamination. We use a Gaussian model $\cN(0, I_d)$ in dimension $d=50$, and draw data from $Q_{\mu_0} = \cN(\mu_0 e_1, I_d)$, where $e_1 = (1, 0, \ldots, 0)^\top \in \R^d$ and $\mu_0 \in \R$ is a mean-shift. The uncertainty radius $\theta$ is chosen to be the Tilted-KSD value corresponding to $\mu_0 = 0, 0.2$ and $0.6$, respectively. The purpose of this experiment is to demonstrate that the robust-KSD test is well-calibrated when $\ksd(Q_{\mu_0}, P) \leq \theta$ and consistent when $\ksd(Q_{\mu_0}, P) > \theta$. As shown in \Cref{fig: gauss mean shift 50d}, the standard tests reject with probability higher than robust-KSD. This is again expected since the standard tests are not robust to contamination. For mean-shift values not greater than the black vertical line, $Q_{\mu_0} \in \cB^\KSD(P; \theta)$, so we are under the null hypothesis and robust-KSD rejects no more frequently than the level, showing that it is well-calibrated. For larger $\mu_0$, robust-KSD rejects with probability approaching one, thus showing its power. Moreover, when $\mu_0=0$ so that $\theta = 0$, robust-KSD becomes identical to the standard Tilted-KSD test, showing that robust-KSD is indeed a generalization of the standard test.

\subsection{Scalability with Dimension}
\label{app: ms experiment}
We run the Gaussian mean-shift experiment in \Cref{app: ms fixed dim} in different dimensions; other experimental setups remain unchanged. The uncertainty radius for R-KSD is chosen to be the (non-squared) KSD value corresponding to mean-shift $\mu = 0.3$. Results are reported in \Cref{fig: gauss mean shift dim}. As dimension increases, both the two standard tests and R-KSD have declining power in correctly rejecting for large values of $\mu$. This shows that both the standard and the robust KSD tests suffer from the \emph{curse-of-dimensionality}, a known issue for kernel-based tests \citep{huang2023highdimensional,reddi2015high,ramdas2015decreasing}. Unsurprisingly, this issue is more prominent for the R-KSD test. This is because the KSD-ball $\cB^\KSD(P; \theta)$ could potentially include more distributions in higher dimensions.

\begin{figure}[t]
    \centering
    \includegraphics[width=.95\textwidth]{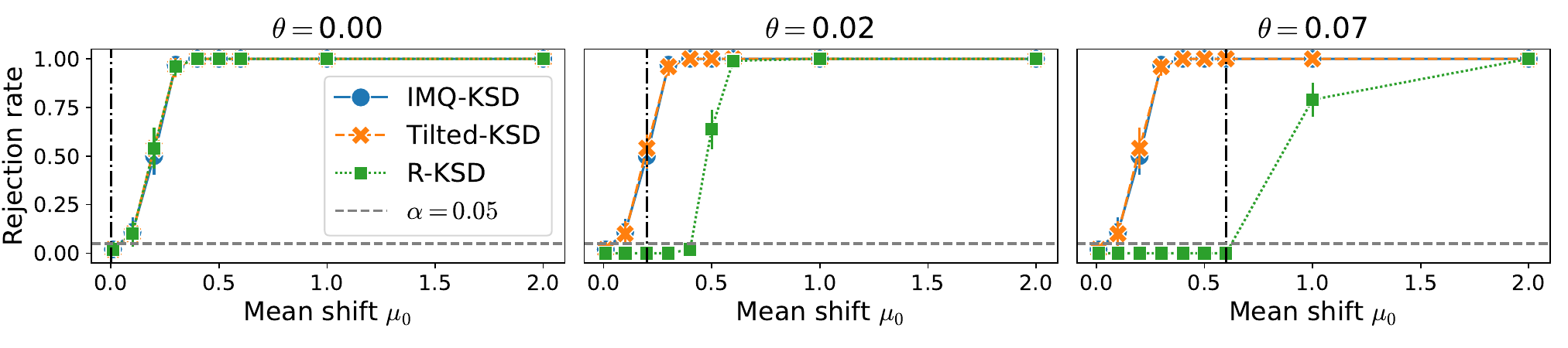}
    \caption{Probability of rejection against the mean-shift $\mu_0$ under a Gaussian model in $d=50$ dimensions. \emph{Black dash-dot line.} Uncertainty radius $\theta$, which is set to be the KSD value $\ksd(Q_{\mu_0}, P)$ corresponding to different values of $\mu_0$.}
    \label{fig: gauss mean shift 50d}
\end{figure}

\begin{figure}[t]
    \centering
    \includegraphics[width=1\textwidth]{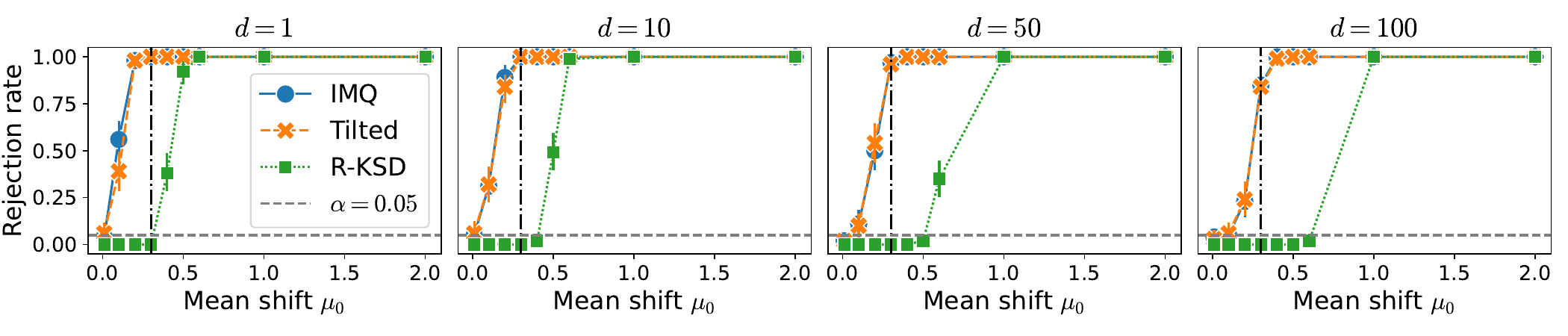}
    \caption{Probability of rejection against the mean shift under a Gaussian model in various dimensions. \emph{Grey dotted line.} Test level $\alpha = 0.05$. \emph{Black dash-dot line.} Uncertainty radius $\theta$, which is set to be the KSD value corresponding to $\mu = 0.3$.}
    \label{fig: gauss mean shift dim}
\end{figure}

\subsection{Different Contamination Distributions}
In most of our experiments in \Cref{sec: Experiments}, the data-generating distribution takes the form $Q = (1-\epsilon_0) P + \epsilon_0 R$ with the contamination distribution $R = \delta_\bz$ being a Dirac delta taking a single value at $\bz$. We now investigate how the choice of $R$ affects the empirical performance of R-KSD. Importantly, our results in \Cref{sec: robust KSD test} make no assumption on the form of $R$.

We set $P = \cN(0, 1)$ and generate data from $Q = (1-\epsilon) P + \epsilon \cN(10, \sigma^2)$ with varying $\sigma > 0$. Following \Cref{prop: KSD bound huber model}, we choose $\theta = \epsilon_0 \tau_\infty^{1/2}$, and fix the contamination tolerance to be $\epsilon_0 = 0.05$. Since \Cref{prop: KSD bound huber model} holds for all $R$, we expect R-KSD to remain (asymptotically) valid in this setting.

Results are shown in \Cref{fig: different noises}. For larger $\sigma$, all tests, including our R-KSD, saw a exhibit reduced test power. This is because, when $\sigma$ is large, the noise component $R$ becomes more dispersed, making it harder for KSD to detect discrepancies. This can be confirmed by the rightmost plot, which shows that the KSD between $Q$ and $P$ decreases with $\sigma$. Crucially, our R-KSD test is able to control the Type-I error regardless of the value of $\sigma$.

\begin{figure}[t]
    \centering
    \includegraphics[width=1\linewidth]{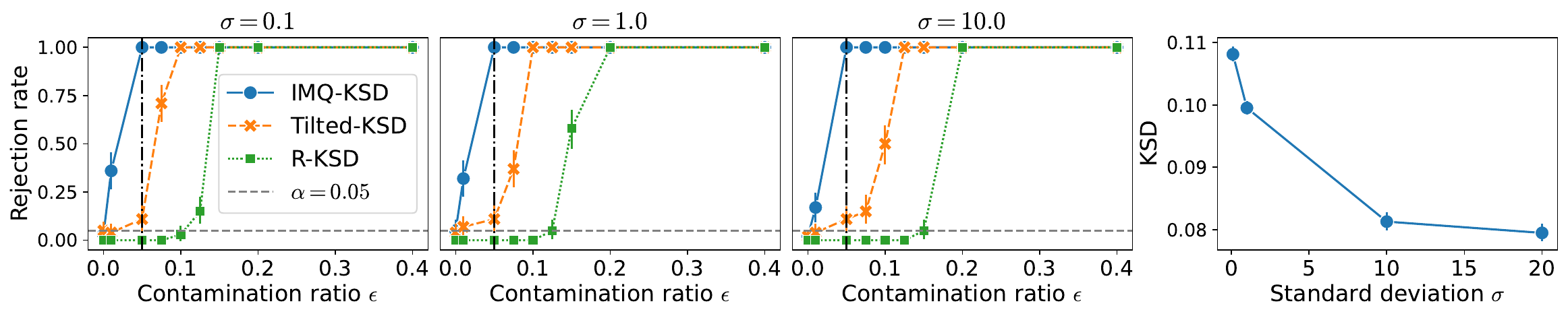}
    \caption{Mixture-of-Gaussian experiment. \emph{Left.} Probability of rejection under different standard deviations for the contamination component. \emph{Right.} KSD estimates. \emph{Grey dotted line.} Test level $\alpha = 0.05$.  \emph{Black dash-dot line.} Uncertainty radius $\theta = \epsilon_0 \tau_\infty^{1/2}$, with $\tau_\infty$ estimated following \Cref{sec: choosing uncertainty radius}.}
    \label{fig: different noises}
\end{figure}

\subsection{Weighted Bootstrap and Wild Bootstrap}
\label{app: efron vs wild}
We compare the robust-KSD test using the weighted bootstrap described in \Cref{sec: KSD-based GOF Testing} (equivalent to the Efron's bootstrap) against the wild bootstrap due to \citet{leucht2013dependent}. For the standard KSD test, weighted bootstrap was used in \citet{liu2016kernelized}, while the wild bootstrap is more popular in the literature \citep{chwialkowski2016kernel,schrab2022ksd,liu2023using}. Compared with weighted bootstrap, the wild approach is more flexible as it can work for dependent samples \citep{chwialkowski2014wild}, and theoretical guarantees on the power of the standard test were only proved with the wild bootstrap \citep{schrab2022ksd}. In this work, we have used the weighted approach since we can then leverage existing theoretical results to show the validity of the test; see \Cref{sec: bootstrap the P-KSD estimate}. This is for simplicity rather than necessity.

We numerically compare these two bootstrap methods using the contaminated Gaussian model in \Cref{sec: contam gaussian} and show that the use of weighted bootstrap does not negatively impact the test. \Cref{fig: gauss ol wild} shows the rejection probability of the robust-KSD test using the weighted bootstrap (R-KSD) and wild bootstrap (R-KSD-Wild). The uncertainty radius is chosen to be $\theta = \epsilon_0 \tau_\infty^{1/2}$ for different values of $\epsilon_0$, following \Cref{sec: contam gaussian}. In particular, $\epsilon_0 = 0$ corresponds to setting $\theta = 0$, in which case robust-KSD reduces to the standard KSD. As evident from the plot, these two bootstrap methods produce almost identical results regardless of the value of $\theta$. 

\begin{figure}[t]
    \centering
    \includegraphics[width=1.\linewidth]{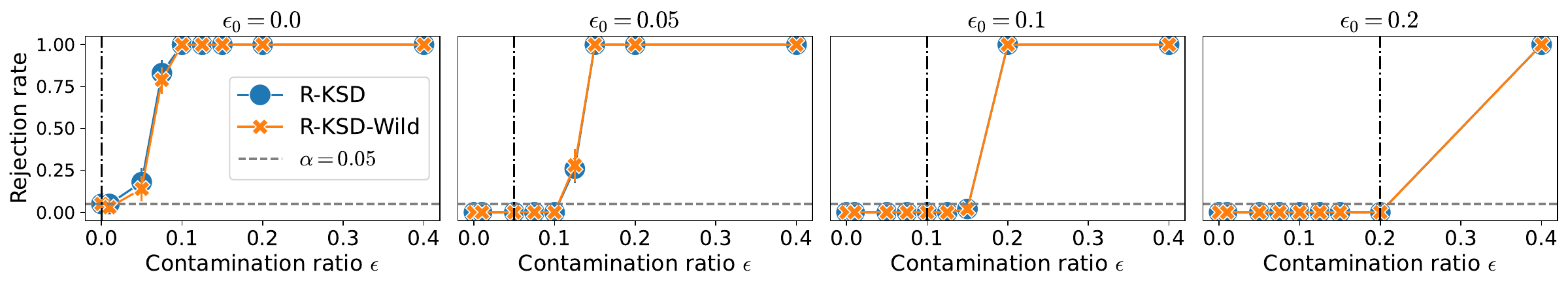}
    \caption{Probability of rejection under an outlier-contaminated Gaussian model using the weighted bootstrap and the wild bootstrap. \emph{Black dash-dot vertical line.} Uncertainty radius $\theta = \epsilon_0 \tau_\infty^{1/2}$, with $\tau_\infty$ estimated following \Cref{sec: choosing uncertainty radius}.}
    \label{fig: gauss ol wild}
\end{figure}

\section{Extension to U-statistics}
\label{app: extension to u-stats}
In this work, we have focused primarily on using the V-statistic \eqref{eq: KSD V stat} for estimating the KSD. An alternative estimator commonly used in the literature \citep{liu2016kernelized,jitkrittum2017linear,schrab2022ksd,liu2023using} is the following U-statistic $ \frac{1}{n(n-1)}\sum_{1 \leq i \neq j \leq n} u_p(\bX_i, \bX_j)$.
In this section, we discuss briefly how to extend our key results to this U-statistic estimate.

\begin{itemize}
\item \textbf{Extending the qualitative (non-)robustness results:}
The main ingredients for proving \Cref{thm: non robust stationary} are \emph{(i)} a deviation bound for the V-statistic $D^2(\X_n)$ as shown in \Cref{lem: bound vstat}, and \emph{(ii)} a concentration bound for the bootstrapped quantile $q^2_{\infty, 1-\alpha}(\X_n)$ as shown in \Cref{lem: bound boot quantile}. A U-statistic counterpart of both results can be shown by following the same proof technique and replacing $D^2(\X_n)$ with the U-statistic. For example, with the U-statistic, a decomposition similar to \eqref{eq: KSD U-stat decomposition} can be derived by summing over only the non-diagonal terms $i \neq j$ and replacing normalizing factors of the form $1/n^2$ to $1/(n(n-1))$. Similarly, the proof of \Cref{thm: robust tilted} can be extended to U-statistics by following the same steps in \Cref{app: robustness tilted}.

\item \textbf{Extending the robust-KSD test:}
To adapt the robust-KSD test to the U-statistic, we first note that rejecting the null if $\Delta_\theta(\X_n) > q_{\infty, 1-\alpha}$ is equivalent to rejecting the null if $D^2(\X_n) > (q_{\infty, 1-\alpha}+\theta)^2$, as argued in \eqref{eq: equiv test expression} and the paragraph thereafter. A natural approach is therefore to replace the V-statistic $D^2(\X_n)$ by the U-statistic, and the bootstrap quantile $q_{\infty, 1-\alpha}$ by the square-root of the bootstrap quantile formed with the U-statistic. However, since the bootstrap samples \eqref{eq: bootstrap sample} based on U-statistics can take \emph{negative} values, the bootstrap quantile can also take negative values, rendering its square-root $q_{\infty, 1-\alpha}$ undefined. One solution is to never reject the null when this happens. This has little impact when $\ksd^2(Q, P)$ is large, since the bootstrap samples are then likely to take positive values. However, when $\ksd^2(Q, P) \approx 0$, the bootstrap samples are more prone to taking negative values, thus making the test conservative.
\end{itemize}

\section{A Non-Asymptotically Valid Robust KSD Test}
\label{app: KSD dev}
The robust-KSD test introduced in \Cref{sec: robust KSD test} is only well-calibrated when $n \rightarrow \infty$. In this appendix, we now derive a robust test that is well-calibrated with finite samples. An immediate consequence is that a stronger, uniform Type-I error control can be achieved.

The test rejects the null $\Hc_0$ in \eqref{eq: KSD-ball hypotheses} if 
\begin{align*}
    \Delta_\theta(\X_n) \;>\; \gamma_n \;,
\end{align*}
where $\gamma_n = \sqrt{\tau_\infty / n} + \sqrt{- 2 \tau_\infty (\log\alpha) / n}$, the constant $\tau_\infty$ is defined in \Cref{lem: bounded Stein kernel}, and $\Delta_\theta(\X_n)$ is defined in \eqref{eq: robust test statistic}. We call this test \emph{robust-KSD-Dev} (R-KSD-Dev). 

The decision threshold $\gamma_n$ is based on the concentration bound \Cref{lem: P-KSD deviation bound}, which is a deviation bound of the McDiarmid's type \citep{mcdiarmid1989method}. This test can be viewed as a counterpart of the robust MMD test of \citet{sun2023kernel}, which is also constructed using McDiarmid's inequality. The next result, proved at the end of this section, shows its finite-sample validity under $\Hc_0$ as well as its consistency under $\Hc_1$.
\begin{theorem}
    \label{thm: robust KSD test validity KSD-ball}
    Let $\X_\infty = \{\bX_i\}_{i=1}^\infty$ be a sequence of independent random variables following $Q$. Suppose $k$ is a tilted kernel satisfying the conditions in \Cref{lem: bounded Stein kernel}, and let $\theta \geq 0$. Then
    \begin{enumerate}
        \item Under $\Hc_0$, for any $n$, it holds that $\sup\nolimits_{Q \in \cB^\KSD(P; \theta)} \Pr\nolimits_{\X_n \sim Q}\big( \Delta_\theta(\X_n) > \gamma_n \big) \;\leq\; \alpha$.
        \item Under $\Hc_1$, it holds that $\Pr\nolimits_{\X_n \sim Q}\big( \Delta_\theta(\X_n) > \gamma_n \big) \to 1$, as $n \to \infty$.
    \end{enumerate}
\end{theorem}
\begin{remark}[Alternative deviation bounds]
\label{rem: other deviation bounds}
    Alternative deviation bounds to \Cref{lem: P-KSD deviation bound} can also be used to construct similar tests. More precisely, \Cref{thm: robust KSD test validity KSD-ball} holds for any threshold $\gamma_n$ that satisfies
    \begin{align*}
        \Pr\nolimits_{\X_n \sim Q}( \S_P(\X_n, Q) > \gamma_n ) \;\leq\; \alpha \;,
    \end{align*}
    where $\S_P(\X_n, Q) = \big\| \E_{\bX \sim Q_n}[u_p(\cdot, \bX)] - \E_{\bX \sim Q}[u_p(\cdot, \bY)] \big\|_{\cH_u}$ is a Hilbert-space norm, as shown in \Cref{lem: P-KSD closed form}. Thus, any deviation bound for Hilbert-space norms may be applied. Examples include another McDiarmid's bound of \citet[Theorem 7]{gretton2012kernel}, the empirical Bernstein bounds of \citet[Theorem A.1]{wolfer2022variance} and \citet[Corollary 1]{martinez2024empirical}, as well as the Hilbert-space valued Hoeffding bound of \citet[Theorem 3.5]{pinelis1994optimum}, particularly its i.i.d.\ variant \citep[Lemma E.1]{chatalic2022nystrom}.  
\end{remark}
\begin{remark}[Justification for McDiarmid-type bound]
    We opted for the McDiarmid-type bound in \Cref{lem: P-KSD deviation bound} because, in our setting, it is tighter than the bounds from \citet{gretton2012kernel}, \citet{wolfer2022variance} and \citet{chatalic2022nystrom}. In particular, while \citet{wolfer2022variance} claim their empirical Bernstein bound outperforms McDiarmid’s inequality, we find that this only holds for their bound tailored to translation-invariant kernels \citep[Theorem 3.1]{wolfer2022variance}, but not for the non-translation-invariant version \citep[Theorem A.1]{wolfer2022variance}, which our setting requires since we have assumed \Cref{lem: bounded Stein kernel}. Moreover, although the bound in \Cref{lem: P-KSD deviation bound} is worse than the empirical Bernstein bound of \citet{martinez2024empirical}, we find that the difference is only marginal, and the former is considerably easier to implement.
\end{remark}
In \Cref{fig: gauss ol dev}, we compare R-KSD-Dev with the bootstrap-based robust-KSD test, the robust MMD test and the standard tests under the outlier-contaminated Gaussian model in \Cref{sec: contam gaussian}. We set the uncertainty radius to be $\theta = \epsilon_0 \tau_\infty^{1/2}$ for various values of $\epsilon_0$, where $\tau_\infty$ is estimated following \Cref{sec: choosing uncertainty radius}. As expected, R-KSD-Dev controls Type-I error against Huber's contamination with maximal proportion $\epsilon_0$. However, when $\epsilon > \epsilon_0$, it has lower power compared with robust-KSD and the MMD-based tests. This is because R-KSD-Dev uses a deviation bound to construct the decision threshold, which is a uniform bound over all data-generating distributions and is therefore conservative. For the same reason, we expect tests constructed using the alternative deviation bounds discussed in \Cref{rem: other deviation bounds} to suffer from similar conservativeness.

Since the robust-KSD-Dev test has finite-sample guarantee on the Type-I error, it might be preferable to the bootstrap-based robust-KSD test when the sample size is small. On the other hand, the conservative nature of the robust-KSD-Dev test suggests it might only be effective for identifying the most aberrant behaviors.

\begin{figure}
    \centering
    \includegraphics[width=.9\linewidth]{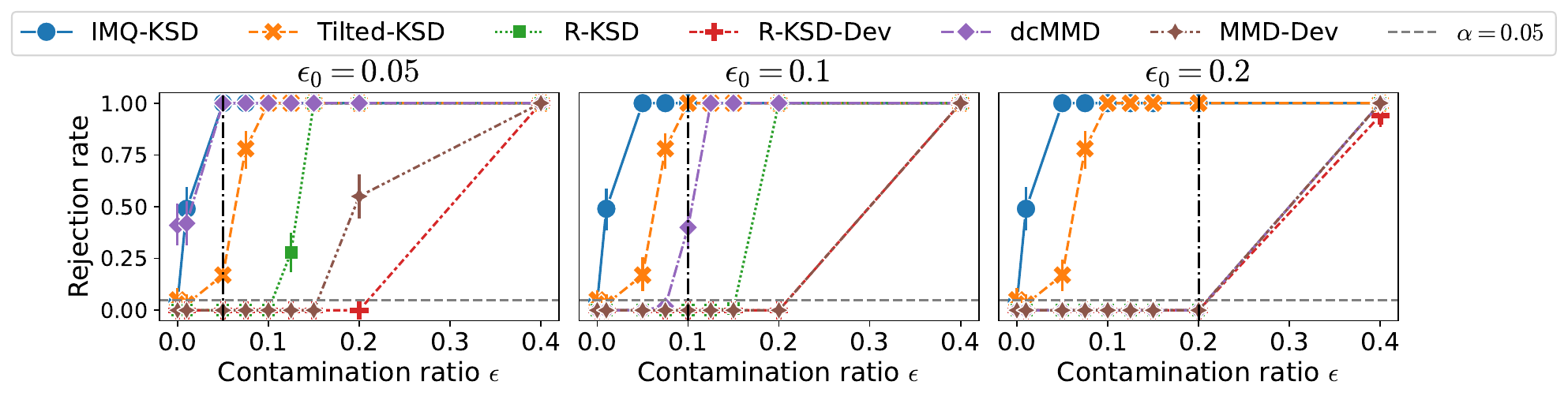}
    \caption{Rejection probability under an outlier-contaminated Gaussian model. \emph{Black dash-dot vertical line.} Uncertainty radius $\theta = \epsilon_0 \tau_\infty^{1/2}$, with $\tau_\infty$ estimated following \Cref{sec: choosing uncertainty radius}.}
    \label{fig: gauss ol dev}
\end{figure}

\begin{proof}[Proof of \Cref{thm: robust KSD test validity KSD-ball}]
    Proceeding as in the proof of \Cref{thm: bootstrap validity}, for any $Q \in \cB^\KSD(P; \theta)$, we have the inequality
    \begin{align*}
        \Pr\nolimits_{\X_n \sim Q}\big( \Delta_\theta(\X_n) \;>\; \gamma_n \big)
        \;&\leq\;
        \Pr\nolimits_{\X_n \sim Q}\big( \S_P(\X_n, Q) \;>\; \gamma_n \big)
        \;,
    \end{align*}
    By the deviation bound in \Cref{lem: P-KSD deviation bound}, the RHS is bounded by $\alpha$, thus proving the first claim. 
    
    We now fix $Q \not\in \cB^\KSD(P; \theta)$ so that we are under $\Hc_1$. Since $\Delta_\theta(\X_n) = \max(D(\X_n) - \theta, 0) \geq D(\X_n) - \theta$, we have
    \begin{align*}
        \Pr\nolimits_{\X_n \sim Q}\big( \Delta_\theta(\X_n) > \gamma_n \big)
        \;&\geq\;
        \Pr\nolimits_{\X_n \sim Q}\big( \ksd(\X_n) - \theta > \gamma_n \big)
        \\
        \;&=\;
        \Pr\nolimits_{\X_n \sim Q}\Big( \sqrt{n} \big(\ksd(\X_n) - \ksd(Q, P) \big) > \sqrt{n}\big(\gamma_n + \theta - \ksd(Q, P)\big) \Big)
        \;.
    \end{align*}
    The same argument in \Cref{sec: proof of thm: bootstrap validity} shows that $\sqrt{n}\big( \ksd(\X_n) - \ksd(Q, P) \big)$ converges weakly to a Gaussian distribution. On the other hand, since $\ksd(Q, P) > \theta$ under $\Hc_1$ and $\gamma_n \to 0$ as $n \to \infty$, the term $\sqrt{n}\big(\gamma_n + \theta - \ksd(Q, P) \big)\to -\infty$. Therefore, the probability in the last line converges to one, thus proving the second claim.
\end{proof}

\clearpage
\bibliography{ref}


\end{document}

%% file: preamble.tex
\usepackage{mathtools, nccmath}

\usepackage[utf8]{inputenc} 
\usepackage[T1]{fontenc}    
\usepackage{url}            
\usepackage{booktabs}       
\usepackage{amsfonts}       
\usepackage{nicefrac}       
\usepackage{microtype}      
\usepackage{xcolor}         
\usepackage{lipsum}		    
\usepackage{amsfonts}
\usepackage{bm}
\usepackage{graphicx}
\usepackage{algorithm}
\usepackage{algpseudocode}
\usepackage{dsfont}
\usepackage{caption}
\usepackage{subcaption}
\usepackage{centernot}
\usepackage{wrapfig}
\usepackage{adjustbox}
\usepackage{bbm}

\usepackage{multirow}
\usepackage{etoc}

\newtheorem{assumption}{Assumption}
\crefname{theorem}{Theorem}{Theorems}
\crefname{assumption}{Assumption}{Assumptions}

\newtheorem{remark}{Remark}

\newtheorem{definition}{Definition}

\ifdefined\proof
    \renewenvironment{proof}[1][]{\par\noindent{\bf \ifthenelse{\equal{#1}{}}{Proof}{#1}\ }}{\hfill\BlackBox\\[2mm]}
\else
    \newenvironment{proof}[1][]{\par\noindent{\bf Proof\ifthenelse{\equal{#1}{}}{Proof}{#1}\ }}{\hfill\BlackBox\\[2mm]}
\fi

\definecolor{ao(english)}{rgb}{0.0, 0.5, 0.0}
\definecolor{alizarin}{rgb}{0.82, 0.1, 0.26}
\usepackage{pifont}

\newcommand{\tagaligneq}{\refstepcounter{equation}\tag{\theequation}}

\renewcommand{\complement}{\mathsf{c}}

\def\bV{\mathbf{V}}
\def\bW{\mathbf{W}}
\def\bX{\mathbf{X}}
\def\bY{\mathbf{Y}}
\def\bZ{\mathbf{Z}}

\def\ba{\mathbf{a}}
\def\bb{\mathbf{b}}
\def\bc{\mathbf{c}}

\def\bh{\mathbf{h}}

\def\bs{\mathbf{s}}

\def\bu{\mathbf{u}}
\def\bv{\mathbf{v}}

\def\bx{\mathbf{x}}

\def\bz{\mathbf{z}}

\def\E{\mathbb{E}}

\def\R{\mathbb{R}}
\def\S{\mathbb{S}}

\def\X{\mathbb{X}}
\def\Y{\mathbb{Y}}
\def\Z{\mathbb{Z}}

\def\cA{\mathcal{A}}
\def\cB{\mathcal{B}}
\def\cC{\mathcal{C}}

\def\cH{\mathcal{H}}

\def\cN{\mathcal{N}}
\def\cO{\mathcal{O}}
\def\cP{\mathcal{P}}
\def\cQ{\mathcal{Q}}

\newcommand{\Var}{\text{\rm Var}}

\newcommand{\indicator}{\mathbbm{1}}

\newcommand*\diff{\mathop{}\!\mathrm{d}}
\newcommand{\ksd}{D}

\newcommand{\MMD}{\mathrm{MMD}}
\newcommand{\KSD}{\mathrm{KSD}}

\newcommand{\cmark}{\ding{51}}
\newcommand{\xmark}{\ding{55}}

\newcommand{\Hc}{H^\mathrm{C}}

\newcommand{\paranoheading}[1]{\textit{#1.}}